%% file: main-arxiv.tex
\documentclass{article}

\pdfoutput=1
\usepackage{microtype}
\usepackage{graphicx}
\usepackage{subfig}
\usepackage{booktabs} 
\usepackage{mathtools, nccmath}
\usepackage{amsfonts}
\usepackage[font=small,labelfont=bf]{caption} 

\usepackage[mathscr]{eucal}
\usepackage[inline]{enumitem}

\RequirePackage{fancyhdr}
\RequirePackage{color}
\RequirePackage{algorithm}
\RequirePackage{algorithmic}
\RequirePackage{natbib}
\RequirePackage{eso-pic} 
\RequirePackage{forloop}
\RequirePackage{hyperref}

\hypersetup{
    colorlinks=true,
    linkcolor=blue,
    citecolor=blue,
    filecolor=blue,
    urlcolor=magenta,
}

\usepackage[left=1in, top=1in, bottom=1in, right=1in]{geometry}
\usepackage{hyperref}

\newif\ifincludenotationtable
\newif\ifarxiv
\newif\ifpaper

\title{Noise-tolerant fair classification}
\date{\vspace{-5ex}}
\author{Alexandre Louis Lamy\footnote{Equal contribution} \footnote{Department of Computer Science, Columbia University, New York, US, \texttt{\url{alexandre.l.lamy@columbia.edu}, \url{ziyuan.zhong@columbia.edu}, \url{verma@cs.columbia.edu}}} \and Ziyuan Zhong\addtocounter{footnote}{-1}\footnotemark[\value{footnote}] \addtocounter{footnote}{1}\footnotemark[\value{footnote}] \and Aditya Krishna Menon\footnote{Google, New York, US, \texttt{\url{adityakmenon@google.com}}} \and \addtocounter{footnote}{-1} Nakul Verma\footnotemark[\value{footnote}]}

\newif\ifarxiv
\newif\ifpaper

\RequirePackage{amsmath,amsthm,amsfonts,amssymb, mathtools} 

\theoremstyle{plain}
\newtheorem{theorem}{Theorem}
\newtheorem{corollary}[theorem]{Corollary}
\newtheorem{lemma}[theorem]{Lemma}

\newtheorem*{claim*}{Claim}

\theoremstyle{definition}

\theoremstyle{remark}

\def\ddefloop#1{\ifx\ddefloop#1\else\ddef{#1}\expandafter\ddefloop\fi}

\def\ddef#1{\expandafter\def\csname bb#1\endcsname{\ensuremath{\mathbb{#1}}}}
\ddefloop ABCDEFGHIJKLMNOPQRSTUVWXYZ\ddefloop

\def\ddef#1{\expandafter\def\csname c#1\endcsname{\ensuremath{\mathcal{#1}}}}
\ddefloop ABCDEFGHIJKLMNOPQRSTUVWXYZ\ddefloop

\def\ddef#1{\expandafter\def\csname v#1\endcsname{\ensuremath{\boldsymbol{#1}}}}
\ddefloop ABCDEFGHIJKLMNOPQRSTUVWXYZabcdefghijklmnopqrstuvwxyz\ddefloop

\def\ddef#1{\expandafter\def\csname v#1\endcsname{\ensuremath{\boldsymbol{\csname #1\endcsname}}}}
\ddefloop {alpha}{beta}{gamma}{delta}{epsilon}{varepsilon}{zeta}{eta}{theta}{vartheta}{iota}{kappa}{lambda}{mu}{nu}{xi}{pi}{varpi}{rho}{varrho}{sigma}{varsigma}{tau}{upsilon}{phi}{varphi}{chi}{psi}{omega}{Gamma}{Delta}{Theta}{Lambda}{Xi}{Pi}{Sigma}{varSigma}{Upsilon}{Phi}{Psi}{Omega}{ell}\ddefloop

\DeclareMathOperator*{\argmin}{arg\,min}

\DeclareMathOperator{\Ex}{\mathbb{E}}

\DeclareMathOperator{\R}{\mathbb{R}}

\DeclareMathOperator{\sign}{sign}

\newcommand{\set}[1]{\ensuremath{\left\{#1\right\}}}

\newcommand{\abs}[1]{\ensuremath{\left|#1\right|}}

\DeclareRobustCommand{\1}{\text{\usefont{U}{bbold}{m}{n}1}}

\newcommand{\ellf}{\bar \ell}
\newcommand{\Lf}{\bar L}
\newcommand{\ellc}{\ell}
\newcommand{\Lc}{L}

\newcommand{\signf}{c_f}

\usepackage[inline]{enumitem}
\usepackage[mathscr]{eucal}
\usepackage{graphicx}
\usepackage{subfig}
\usepackage{mathtools, nccmath}
\usepackage[font=small,labelfont=bf]{caption}
\RequirePackage{algorithm}
\RequirePackage{algorithmic}
\RequirePackage{natbib}

\def\independenT#1#2{\mathrel{\rlap{$#1#2$}\mkern2mu{#1#2}}}
\usepackage[toc,page]{appendix}

\newcommand{\FINALCHANGE}[2]{#2}

\newcommand{\AKMEDIT}[1]{{#1}}

\begin{document}

\maketitle

\begin{abstract}
\input{abstract.tex}
\end{abstract}

\section{Introduction}
\label{sec:intro}
\input{intro.tex}

\section{Related work}
\label{sec:related}
\input{related.tex}
\section{Background and notation}

\label{sec:background}
\input{background.tex}

%

\section{Fairness under sensitive attribute noise}
\label{section:sensnoise}
\arxivtrue
\paperfalse
\input{noise_on_a.tex}

%

%
\section{Experiments}
\label{sec:experiments}
\input{7experiment.tex}

\section{Conclusion and future work}
\input{conclusion.tex}

\newpage
\bibliography{references}
\bibliographystyle{bibstyle}

\clearpage
\onecolumn

\input{9appendix.tex}

\end{document}

%% file: abstract.tex
Fairness-aware learning involves designing algorithms that do not discriminate with respect to some sensitive feature (e.g., race or gender). 
Existing work on the problem operates under the assumption that the sensitive feature available in one's training sample is perfectly reliable.
This assumption may be violated in many real-world cases:
for example, respondents to a survey may choose to conceal or obfuscate their group identity out of fear of potential discrimination. 
This poses the question of whether one can still learn fair classifiers given \emph{noisy} sensitive features.
In this paper, we answer the question in the affirmative:
we show that if one measures fairness using the \emph{mean-difference score},
and sensitive features are
subject to noise from
the \emph{mutually contaminated learning} model,
then owing to a simple identity
we only need to change the desired fairness-tolerance.
The requisite tolerance can be estimated by leveraging existing noise-rate estimators from the label noise literature.
We finally show that our procedure is empirically effective on two case-studies involving sensitive feature censoring.

%% file: intro.tex
Classification is concerned with maximally discriminating between a number of pre-defined groups.
\emph{Fairness-aware} classification concerns the analysis and design of classifiers that
do not discriminate with respect to some sensitive feature (e.g., race, gender, age, income).
Recently, much progress has been made
on devising {appropriate} measures of fairness~\citep{Calders,Dwork:2011,Feldman,Hardt2016,ZafarWWW,Zafar:2017b,Kusner,Kim,Speicher:2018,Heidari:2019},
and
means of achieving them~\citep{ZemelICML,ZafarWWW,Calmon:2017,Dwork:2018,reduction,Donini:2018,Cotter:2018,Williamson:2019,Mohri:2019}.

Typically, fairness is achieved by adding constraints which depend on the sensitive feature, and then correcting one's learning procedure to achieve these fairness constraints.
For example, suppose the data comprises of pairs of individuals and their loan repay status,
and the sensitive feature is gender.
Then, we may add a constraint that we should predict equal loan repayment for both men and women (see \S\ref{section:fairnessdef} for a more precise statement).
However, this and similar approaches assume that we are able to correctly measure or obtain the sensitive feature. 
In many real-world cases, one may only observe noisy versions of the sensitive feature.
For example, survey respondents may choose to conceal or obfuscate their group identity out of 
{concerns of potential mistreatment or outright discrimination}. 

One is then brought to ask whether fair classification in the presence of such \emph{noisy} sensitive features is still possible.
Indeed, if the noise is high enough and all original information about the sensitive features is lost, then it is as if the sensitive feature was not provided.
Standard learners can then be 
unfair on such data~\citep{Dwork:2011, pedreshi2008discrimination}.
Recently,~\citet{hashimoto2018fairness} showed that progress is  possible, albeit for specific fairness measures.
The question of what can be done under a smaller amount of noise is thus both interesting and non-trivial.

In this paper, we consider two {practical} scenarios where we may only observe noisy sensitive features:
\begin{enumerate}[itemsep=0pt,topsep=0pt,leftmargin=16pt,label={(\arabic*)}]
    \item suppose we are releasing data involving human participants.
    Even if noise-free sensitive features are available,
    we may wish to \emph{add} noise so as to obfuscate sensitive attributes,
    and thus protect participant data from potential misuse.
    Thus, being able to learn fair classifiers under sensitive feature noise is a way to achieve both privacy \emph{and} fairness.
    
    \item suppose we wish to analyse data where 
    the presence of the sensitive feature is only known for a subset of individuals,
    while for others the feature value is unknown.
    For example,
    patients filling out a form may feel comfortable disclosing that they do not have a pre-existing medical condition; however, some who do have this condition may 
    {wish to refrain from responding}.
    This can be seen as a variant of the \emph{positive and unlabelled} (PU) setting~\citep{Denis:1998a},
    where the sensitive feature is present (positive) for some individuals,
    but absent (unlabelled) for others.
\end{enumerate}

By considering {popular measures of fairness and a general model of noise},
we show that fair classification is possible under 
many settings, including the above.
Our precise contributions are:
\begin{enumerate}[itemsep=0pt,topsep=0pt,leftmargin=24pt,label=(\textbf{C\arabic*})]
    \item we show that if the sensitive features are subject to noise as per the \emph{mutually contaminated learning model}~\citep{Scott13},
and one measures fairness using the \emph{mean-difference score}~\citep{Calders2010},
then a simple identity (Theorem~\ref{thm: sensitive noise reduction}) yields that we only need to change the desired fairness-tolerance.
The requisite tolerance can be estimated by leveraging existing noise-rate estimators
from the label noise literature,
yielding a reduction (Algorithm~\ref{alg:reduction}) to regular noiseless fair classification.

    \item we show that our procedure is empirically effective on both case-studies mentioned above.
\end{enumerate}

In what follows,
we review the existing literature on learning fair and noise-tolerant classifiers in~\S\ref{sec:related},
and introduce the 
novel
problem formulation of noise-tolerant fair learning 
in~\S\ref{sec:background}.
We then detail how to address this problem in~\S\ref{section:sensnoise},
and empirically confirm the efficacy of our approach in~\S\ref{sec:experiments}.

%% file: related.tex
We review relevant literature on fair and noise-tolerant 
machine learning.

\subsection{Fair machine learning}
\label{sec:bg-fairness}

Algorithmic fairness has gained significant attention recently because of the undesirable social impact caused by bias in machine learning algorithms~\citep{COMPAS,gendershades,Lahoti}.
There are two central 
{objectives}:
designing appropriate application-specific fairness criterion, and
developing predictors that respect the chosen fairness conditions.

{Broadly, fairness objectives can be categorised} into
individual- and group-level fairness.
Individual-level fairness~\citep{Dwork:2011,Kusner,Kim} requires the treatment of ``similar'' individuals to be similar.
Group-level fairness 
asks the treatment of the groups divided based on some sensitive attributes (e.g., gender, race) to be similar.
Popular notions of group-level fairness include
demographic parity~\citep{Calders}
and
equality of opportunity~\citep{Hardt2016}%
;
see \S\ref{section:fairnessdef} for formal definitions. 

{Group-level fairness criteria have been the subject of significant algorithmic design and analysis,}
and are achieved in three possible ways:
%
\begin{itemize}[itemsep=0pt,topsep=0pt,leftmargin=16pt]
    \item[--] pre-processing methods~\citep{ZemelICML,ZemelICLR,Lum,Johndrow:2017,Calmon:2017,delBarrio:2018,Adler:2018},
    which
    {usually find a new representation of the data where the bias with respect to the sensitive feature is explicitly removed.}
    
    \item[--] methods enforcing fairness during training~\citep{Calders,Woodworth,ZafarWWW,reduction},
    which
    usually add a constraint that is a proxy of the fairness criteria or add a regularization term to penalise fairness violation.

    \item[--] post-processing methods~\citep{Feldman,Hardt2016},
    which usually apply a thresholding function to make the prediction satisfying the chosen fairness notion across groups.
\end{itemize}

\subsection{Noise-tolerant classification}

Designing noise-tolerant classifiers is a classic topic of study, concerned with the setting where one's training labels are corrupted in some manner.
Typically, works in this area postulate a particular model of label noise, and study the viability of learning under this model.
Class-conditional noise (CCN)~\citep{Angluin88} is one such 
{effective}
noise model.
Here, samples from each class have their labels flipped with some constant (but class-specific) probability.
Algorithms that deal with CCN corruption have been well studied~\citep{natarajan2013learningw,Liu2016ClassificationWN,northcutt2017rankpruning}.
These methods typically first estimate the noise rates, which are then used for prediction.
A special case of CCN learning is learning from positive and unlabelled data (PU learning)~\citep{Elkan},
where in lieu of explicit negative samples, one has a pool of unlabelled data.

Our interest in this paper will be the \emph{mutually contaminated} (MC) \emph{learning} noise model~\citep{Scott13}.
This model (described in detail in \S\ref{sec:mc-learning}) captures both CCN and 
PU learning as special cases~\citep{pmlr-v30-Scott13,corruption}, {as well as other interesting noise models}.

%% file: background.tex
We recall the settings of standard and fairness-aware binary classification%
\footnote{For simplicity, we consider the setting of binary target and sensitive features.
However, our derivation and method can be easily extended to the multi-class setting.},
and establish notation.

Our notation is summarized in Table \ref{table:notation}.

\subsection{Standard binary classification}

Binary classification concerns predicting the label or \emph{target feature} $Y \in \set{0,1}$ that best corresponds to a given instance $X \in \cX$.
Formally, suppose $D$ is a distribution over (instance, target feature) pairs from $\cX \times \set{0,1}$.
Let $f \colon \cX \to \R$ be a score function,
and $\cF \subset \R^{\cX}$ be a user-defined class of such score functions.
Finally, let $\ellc: \R \times \set{0,1} \to \R_+$ be a loss function measuring the disagreement between a given score and binary label.
The goal of binary classification is to minimise 
\begin{align}
    \label{eqn:standard-risk}
    \Lc_D(f) &:= \Ex_{(X, Y) \sim D}[\ellc(f(X), Y)].
\end{align}

%
\subsection{Fairness-aware classification}
\label{section:fairnessdef}

In fairness-aware classification, the goal of accurately predicting the target feature $Y$ remains.
However, there is an additional \emph{sensitive feature} $A \in \set{0,1}$ upon which we do not wish to discriminate.
Intuitively, some user-defined fairness loss should be roughly the same regardless of $A$.

Formally, suppose $D$ is a distribution over 
(instance, sensitive feature, target feature) triplets from
$\cX \times \{0,1\} \times \{0,1\}$.
The goal of \emph{fairness-aware} binary classification is to find%
\footnote{Here, $f$ is assumed to not be allowed to use $A$ at test time, which is a common legal restriction~\citep{lipton2018does}.
Of course, $A$ can be used at training time to find an $f$ which satisfies the constraint.}
\begin{equation}
    \label{eqn:fairness-objective}
    \begin{aligned}
        f^*    &:= \argmin_{f\in \mathcal{F}} \Lc_D(f), \textrm{ such that } \Lambda_D( f ) \leq \tau \\
        \Lc_D(f) &:= \Ex_{(X, A, Y) \sim D}[\ellc(f(X), Y)],
    \end{aligned}
\end{equation}
for user-specified 
\emph{fairness tolerance} $\tau \geq 0$, and
\emph{fairness constraint} $\Lambda_D \colon \cF \to \R_+$. 
Such constrained {optimisation} problems can  be solved in various ways,
e.g., convex relaxations~\citep{Donini:2018},
alternating {minimisation}~\citep{ZafarWWW,Cotter:2018},
or {linearisation}~\citep{Hardt2016}.


%

A number of fairness constraints $\Lambda_D( \cdot )$ have been proposed in the literature.
We focus on two {important and} specific choices in this paper, inspired by~\citet{Donini:2018}:
\begin{align}
    \label{eqn:ddp}
    \Lambda^{\mathrm{DP}}_D( f ) &:= \abs{ \Lf_{D_{0, \cdot}}(f) - \Lf_{D_{1, \cdot}}(f) } \\ 
    \label{eqn:deo}
    \Lambda^{\mathrm{EO}}_D( f ) &:= \abs{ \Lf_{D_{0, 1}}(f) - \Lf_{D_{1, 1}}(f) }, 
\end{align}
where we denote by $D_{a,\cdot},D_{\cdot,y},$ and $D_{a,y}$ the distributions over $\cX \times \{0,1\} \times \{0,1\}$ given by $D_{\mid A=a}, D_{\mid Y=y}, $ and $D_{\mid A=a, Y=y}$ and $\ellf: \R \times \set{0,1} \to \R_+$ is the user-defined fairness loss with corresponding $\Lf_D(f) := \Ex_{(X, A, Y) \sim D}[\ellf(f(X), Y)]$.
Intuitively, these measure the difference in the average of the fairness loss incurred among the instances with and without the sensitive feature.

Concretely, if $\ellf$ is taken to be $\ellf(s, y) = \1[\sign(s) \neq 1]$ and the 0-1 loss $\ellf(s, y) = \1[\sign(s) \neq y]$ respectively, then 
for $\tau = 0$,
\eqref{eqn:ddp} and \eqref{eqn:deo}
correspond to the \textit{demographic parity}~\citep{Dwork:2011} and \textit{equality of opportunity}~\citep{Hardt2016} constraints. 
Thus, we denote these two relaxed fairness measures \textit{disparity of demographic parity} (DDP) and \textit{disparity of equality of opportunity} (DEO).
These quantities are also known as the \emph{mean difference score} in~\citet{Calders2010}.%
\footnote{Keeping $\ellf$ generic allows us to capture a range of group fairness definitions, not just demographic parity and equality of opportunity; e.g., disparate mistreatment~\citep{ZafarWWW} corresponds to using the 0-1 loss and $\Lambda^{\mathrm{DP}}_D$, and equalized odds can be captured by simply adding another constraint for $Y=0$ along with $\Lambda^{\mathrm{EO}}_D$. }

\begin{table}[t]
\caption {Glossary of commonly used symbols} \label{table:notation}
\centering

\resizebox{0.99\linewidth}{!}{
\begin{tabular}{llll}
\toprule
\toprule
\textbf{Symbol} &  \textbf{Meaning} & \textbf{Symbol} &  \textbf{Meaning} \\ 
\toprule
$X$ &  instance & $D_{\textrm{corr}}$ &  corrupted distribution $D$\\
$A$ &  sensitive feature & $f$ &  score function $f: \cX \to \R$\\
$Y$ &  target feature & $\ellc$ & accuracy loss $\ellc: \R \times \set{0, 1} \to \R_+$ \\
$D$ &  distribution $\mathbb{P}(X,A,Y)$ &  $\Lc_D$ &  expected accuracy loss on $D$\\
$D_{a,\cdot}$ &  distribution $\mathbb{P}(X,A,Y|A=a)$ & $\ellf$ & fairness loss $\ellf: \R \times \set{0, 1} \to \R_+$\\
$D_{\cdot,y}$ &  distribution $\mathbb{P}(X,A,Y|Y=y)$ & $\Lf_D$ & expected fairness loss on $D$ \\
$D_{a,y}$ &  distribution $\mathbb{P}(X,A,Y|A=a,Y=y)$  & $\Lambda_D$ &  fairness constraint  \\ 
\bottomrule
\end{tabular}
}
\end{table}

\subsection{Mutually contaminated learning}
\label{sec:mc-learning}

In the framework of learning from mutually contaminated distributions (MC learning)~\citep{pmlr-v30-Scott13},
instead of observing samples from the ``true'' (or ``clean'') joint distribution $D$,
one observes samples from a corrupted distribution $D_{\mathrm{corr}}$.
The corruption is such that the observed \emph{class-conditional} distributions are mixtures of their true counterparts.
More precisely, let $D_{y}$ denote the conditional distribution for label $y$.
Then, one assumes that
\FINALCHANGE{
\begin{equation}
    \label{equation:mc-model}
    \begin{aligned}
        D_{0, \mathrm{corr}} &= (1-\alpha) \cdot D_{1} + \alpha \cdot  D_{0}\\
        D_{1, \mathrm{corr}} &= \beta \cdot D_{1} + (1-\beta) \cdot  D_{0},
    \end{aligned}
\end{equation}
}{
\begin{equation}
    \label{equation:mc-model}
    \begin{aligned}
        D_{1, \mathrm{corr}} &= (1-\alpha) \cdot D_{1} + \alpha \cdot  D_{0}\\
        D_{0, \mathrm{corr}} &= \beta \cdot D_{1} + (1-\beta) \cdot  D_{0},
    \end{aligned}
\end{equation}
}
where $\alpha, \beta \in (0,1)$ are (typically unknown) noise parameters with $\alpha + \beta < 1$.%
\footnote{\AKMEDIT{The constraint imposes no loss of generality:
when $\alpha+\beta>1$, we can simply flip the two labels and apply our theorem.
When $\alpha+\beta=1$, all information about the sensitive attribute is lost.
This pathological case 
is equivalent to not measuring the sensitive attribute at all.}}
Further, the corrupted base rate $\pi_{\mathrm{corr}} := \mathbb{P}[ Y_{\mathrm{corr}} = 1 ]$ may be arbitrary.
The MC learning framework subsumes CCN and PU learning~\citep{pmlr-v30-Scott13,corruption}, \AKMEDIT{which are prominent noise models that have seen sustained study in recent years~\citep{Jain:2017,Kiryo:2017,vanRooyen:2018,Katz:2019,Charoenphakdee:2019}.}

%% file: noise_on_a.tex
The standard fairness-aware learning problem assumes we have access to the true sensitive attribute,
so that we can both measure and control our classifier's unfairness as measured by, e.g., Equation~\ref{eqn:ddp}.
Now suppose that rather than being given the sensitive attribute, we get a noisy version of it. 
We will show that the fairness constraint on the clean distribution is \emph{equivalent} to a \emph{scaled} constraint on the 
noisy distribution.
This gives a simple reduction from fair machine learning in the presence of noise to
the regular fair machine learning, which can be done in a variety of ways as discussed in~\S\ref{sec:bg-fairness}.

\subsection{Sensitive attribute noise model}
\label{section:sensnoisesetup}

As previously discussed, we use MC learning as our noise model,
as this captures both CCN and PU learning as special cases;
hence, we automatically obtain results for both these interesting settings.

Our specific formulation of MC learning noise on the sensitive feature is as follows.
Recall from \S\ref{section:fairnessdef} that $D$ is a distribution over $\mathscr{X} \times \set{0,1} \times \set{0,1}$.
Following (\ref{equation:mc-model}),
for unknown noise parameters
$\alpha,\beta \in (0,1)$ with $\alpha + \beta < 1$,
we assume that the corrupted class-conditional distributions are:
\begin{equation}\label{equation:mcnoise1}
    \begin{aligned}
        &D_{1,\cdot,\mathrm{corr}} = (1-\alpha) \cdot D_{1,\cdot} + \alpha \cdot  D_{0,\cdot}\\
        &D_{0,\cdot,\mathrm{corr}} = \beta \cdot D_{1,\cdot} + (1-\beta) \cdot  D_{0,\cdot},
    \end{aligned}
\end{equation}
and that the corrupted base rate is $\pi_{a,\mathrm{corr}}$ (we write the original base rate, $\mathbb{P}_{(X, A, Y) \sim D}[A=1]$ as $\pi_a$).
{That is,} the distribution over (instance, label) pairs for the group with $A = 1$, i.e. $\mathbb{P}( X, Y \mid A = 1 )$,
is assumed to be mixed with the distribution for the group with $A = 0$, and vice-versa.

Now, when interested in the EO constraint, it can be simpler to assume that the noise instead satisfies
\begin{equation}\label{equation:mcnoise2}
    \begin{aligned}
        &D_{1,1,\mathrm{corr}} = (1-\alpha') \cdot D_{1,1} + \alpha' \cdot D_{0,1}\\
        &D_{0,1,\mathrm{corr}} = \beta' \cdot D_{1,1} + (1-\beta') \cdot D_{0,1},
    \end{aligned}
\end{equation}
for noise parameters $\alpha', \beta' \in (0, 1)$.
As shown by the following, this is not a different assumption.

\begin{lemma}\label{lemma:mcrelation}
    Suppose there is noise in the sensitive attribute only, as given in Equation \eqref{equation:mcnoise1}.
    Then,
    there exists constants $\alpha', \beta'$ such that
    Equation \eqref{equation:mcnoise2} holds.
\end{lemma}

Although the lemma 
gives a way to calculate $\alpha', \beta'$ from $\alpha, \beta$, in practice it may be useful to consider \eqref{equation:mcnoise2} independently.
Indeed, when one is interested in the EO constraints we will show below that only knowledge of $\alpha', \beta'$ is required.
It is often much easier to estimate $\alpha', \beta'$ directly
(which can be done in the same way as estimating $\alpha, \beta$ simply by considering $D_{\cdot, 1, \mathrm{corr}}$ rather than $D_{\mathrm{corr}}$).

\subsection{Fairness constraints under MC learning}

We now show that the previously introduced fairness constraints for demographic parity and equality of opportunity are automatically robust to MC learning noise in $A$.

\begin{theorem}\label{thm: sensitive noise reduction}
    Assume that we have noise as per Equation \eqref{equation:mcnoise1}.
    Then, for any $\tau > 0$ and $f \colon X \to \mathbb{R}$,
    \begin{align*}
     \Lambda^{\mathrm{DP}}_{D}( f ) \leq \tau &\iff \Lambda^{\mathrm{DP}}_{D_{\mathrm{corr}}}(f) \leq \tau \cdot (1-\alpha - \beta) \\
         \Lambda^{\mathrm{EO}}_{D_{\cdot, 1}}( f ) \leq \tau &\iff \Lambda^{\mathrm{EO}}_{D_{\mathrm{corr}, \cdot, 1}}(f) \leq \tau \cdot (1-\alpha' - \beta'),
    \end{align*}
    where $\alpha'$ and $\beta'$ are as per Equation~\eqref{equation:mcnoise2} and Lemma~\ref{lemma:mcrelation}.
\end{theorem}

The above can be seen as a consequence of the immunity of the \emph{balanced error}~\citep{Chan:1998,Brodersen:2010,consistency} to corruption under the MC model.
Specifically, consider a distribution $D$ over an input space $\mathscr{Z}$ and label space $\mathscr{W} = \set{0, 1}$.
Define
$$ B_{D} := \mathbb{E}_{Z \mid W = 0}[ h_0( Z ) ] + \mathbb{E}_{Z \mid W = 1}[ h_1( Z ) ] $$
for functions $h_0, h_1 \colon \mathscr{Z} \to \R$.
Then, if for every $z \in \mathbb{R}$
$h_0( z ) + h_1( z ) = 0$, 
we have~\citep[Theorem 4.16]{vanRooyen:2015}, \citep{Blum:1998,Zhang:2008,corruption}
\begin{equation}
    \label{eqn:ber-identity}
     B_{D_{\mathrm{corr}}} = (1 - \alpha - \beta) \cdot B_{D},
\end{equation}
where $D_{\mathrm{corr}}$ refers to a corrupted version of $D$ under MC learning with noise parameters $\alpha, \beta$.
{That is,} the effect of MC noise on $B_D$ is simply to perform a scaling.
Observe that
$B_D = \Lf_D( f )$
if
we set
$Z$ to $X \times Y$,
$W$ to the sensitive feature $A$,
and
$h_0( (x, y) ) = +\ellf( y, f( x ) )$,
$h_1( (x, y) ) = -\ellf( y, f( x ) )$.
Thus, (\ref{eqn:ber-identity}) implies $\Lf_D( f ) = (1 - \alpha - \beta) \cdot \Lf_{D_{\mathrm{corr}}}( f )$,
and thus Theorem~\ref{thm: sensitive noise reduction}.

\subsection{Algorithmic implications}

Theorem~\ref{thm: sensitive noise reduction} has an important algorithmic implication.
Suppose we pick a fairness constraint $\Lambda_D$ and seek to solve
Equation~\ref{eqn:fairness-objective} for a given tolerance $\tau \geq 0$.
Then,
given samples from $D_{\mathrm{corr}}$, it suffices to simply change the tolerance to $\tau' = \tau \cdot (1 - \alpha - \beta)$.

Unsurprisingly, $\tau'$ depends on the noise parameters $\alpha, \beta$.
In practice, these will be unknown;
however, there have been several algorithms proposed to estimate these from noisy data alone~\citep{pmlr-v30-Scott13,corruption,Liu2016ClassificationWN,Ramaswamy:2016,northcutt2017rankpruning}.
Thus, we may use these to construct estimates of $\alpha, \beta$, and plug these in to construct an estimate of $\tau'$.

In sum, we may tackle fair classification in the presence of noisy $A$ by suitably combining \emph{any} existing fair classification method (that takes in a parameter $\tau$ that is proportional to mean-difference score of some fairness measures), and \emph{any} existing noise estimation procedure.
This is summarised in Algorithm~\ref{alg:reduction}.
Here, {\sf FairAlg} is any existing fairness-aware classification method that solves Equation~\ref{eqn:fairness-objective}, and {\sf NoiseEst} is any noise estimation method that estimates $\alpha, \beta$.

\renewcommand{\algorithmicrequire}{\textbf{Input:}}
\renewcommand{\algorithmicensure}{\textbf{Output:}}

\begin{algorithm}
	\caption{Reduction-based algorithm for fair classification given noisy $A$.}
	\label{alg:reduction}
	\begin{algorithmic}[1]
	    \REQUIRE Training set $S = \{ (x_i, y_i, a_i) \}_{i=1}^n$,
	    scorer class $\cF$,
	    fairness tolerance $\tau \geq 0$,
	    fairness constraint $\Lambda( \cdot )$,
	    fair classification algorithm {\sf FairAlg},
	    noise estimation algorithm {\sf NoiseEst}
	    \ENSURE  Fair classifier $f^* \in \cF$
	    \STATE   $\hat{\alpha}, \hat{\beta} \leftarrow {\sf NoiseEst}( S )$
	    \STATE   $\tau' \leftarrow ( 1 - \hat{\alpha} - \hat{\beta} ) \cdot \tau$
	    \STATE  \textbf{return} ${\sf FairAlg}( S, \cF, \Lambda, \tau' )$
	\end{algorithmic}
\end{algorithm}

\FINALCHANGE{}{\subsection{Noise rate and sample complexity}

So far, we have shown that at a distribution level, fairness with respect to the noisy sensitive attribute is equivalent to fairness with respect to the real sensitive attribute. 
However, from a sample complexity perspective, a higher noise rate will require a larger number of samples for the empirical fairness to generalize well, i.e., guarantee fairness at a distribution level. 
A concurrent work by~\citet{fairlearning_withprivatedata} derives precise sample complexity bounds and makes this relationship explicit.
}

\subsection{Connection to \FINALCHANGE{differential privacy}{privacy and fairness}}

While Algorithm \ref{alg:reduction} gives a way of achieving fair classification on an already noisy dataset such as the {use case} described in example (2) of \S\ref{sec:intro}, it can also be used to simultaneously achieve fairness and privacy.
As described in example (1) of \S\ref{sec:intro}, the very nature of the sensitive attribute makes it likely that even if noiseless sensitive attributes are available, one might want to add noise to guarantee some form of privacy.
Note that {simply} removing the feature does not suffice, because it would make difficult the task of
developing fairness-aware classifiers for the dataset
~\citep{Gupta:2018}.
Formally, we can give the following privacy guarantee by adding CCN noise to the sensitive attribute. 

\begin{lemma}\label{thm: randomized response for differential privacy}
    To achieve $(\epsilon, \delta=0)$ differential privacy
    on the sensitive attribute we can add CCN noise with $\rho^+ = \rho^- = \rho \geq \frac{1}{\exp{(\epsilon)}+1}$ to the sensitive attribute.
\end{lemma}

Thus, if a desired level of differential privacy is required before releasing a dataset, one could simply add the required amount of CCN noise to the sensitive attributes, publish this modified dataset as well as the noise level, and researchers could use Algorithm \ref{alg:reduction} (without even needing to estimate the noise rate) to do fair classification as usual. 

\FINALCHANGE{Recently,~\citet{Kearns18} explored preserving differential privacy~\citep{Dwork06} while 
{maintaining} fairness constraints.}{There is previous work that tries to preserve privacy of individuals' sensitive attributes while learning a fair classifier.
\cite{pmlr-v80-kilbertus18a} employs the cryptographic tool of secure multiparty computation (MPC) to try to achieve this goal. 
However, as noted by~\citet{Kearns18}, the individual information that MPC tries to protect can still be inferred from the learned model.
Further, the method of \citet{pmlr-v80-kilbertus18a} is limited to using demographic parity as the fairness criteria. 
}

A more recent work of~\citet{Kearns18} explored preserving differential privacy~\citep{Dwork06} while 
{maintaining} 
fairness constraints. 
The authors proposed two methods:
one adds Laplace noise to training data and apply the post-processing method in~\citet{Hardt2016},
while \FINALCHANGE{another}{the other} modifies the method in~\citet{reduction} using the exponential mechanism as well as Laplace noise.
Our work differs from them in three major ways:
\begin{enumerate}[itemsep=0pt,topsep=0pt,label=(\emph{\arabic*})]
    
    \item \FINALCHANGE{our work allows for fair classification to be done using a 
{\emph{any}} fairness-aware classifier, whereas the method of \citet{Kearns18} 
{requires the use of} a particular classifier.}{our work allows for fair classification to be done using 
{\emph{any}} \emph{in-process} fairness-aware classifier that allows user to specify desired fairness level. On the other hand, the first method of \citet{Kearns18} 
requires the use of a post-processing algorithm 
(which generally have worse trade-offs than in-processing algorithms~\citet{reduction}), 
while the second method requires the use of a single \emph{particular} classifier.}
    
    
    \item our focus is on designing fair-classifiers with noise-corrupted sensitive attributes; {by contrast, the main concern in~\citet{Kearns18} is achieving differential privacy}\FINALCHANGE{}{and thus they do not discuss how to deal with noise that is already present in the dataset}.
    
    \item \FINALCHANGE{we deal with not only equalized odds, but also demographic parity.}{our method is shown to work for a large class of different fairness definitions.}
\end{enumerate}

\FINALCHANGE{}{Finally, a concurrent work of~\citet{fairlearning_withprivatedata} 
builds upon our method
for the problem of preserving privacy of the sensitive attribute.
The authors use a randomized response procedure on the sensitive attribute values, followed by a two-step procedure to train a fair classifier using the processed data. 
Theoretically, their method improves upon the sample complexity of our method and extends our privacy result to the case of non-binary groups. 
However, their method solely focuses on preserving privacy rather than the general problem of sensitive attribute noise.}



%% file: 7experiment.tex
We demonstrate that it is viable to learn fair classifiers given noisy sensitive features.\footnote{\FINALCHANGE{}{Source code is available at \url{https://github.com/AIasd/noise\_fairlearn}.}}
As our underlying fairness-aware classifier,
we use a modified version of the classifier implemented in \citet{reduction} with the DDP and DEO constraints which, as discussed in \S\ref{section:fairnessdef}, are special cases 
of our more general constraints \eqref{eqn:ddp} and \eqref{eqn:deo}. The classifier's original constraints can also be shown to be noise-invariant but in a slightly different way (see Appendix \ref{appendix:agarwal_constraint} for a discussion). An advantage of this classifier is that it is shown to reach levels of fairness violation that are very close to the desired level ($\tau$), i.e., for small enough values of $\tau$ it will reach the constraint boundary.

While we had to choose a particular classifier, our method can be used before using any downstream fair classifier as long as it can take in a parameter $\tau$ that controls the strictness of the fairness constraint and that its constraints are special cases of our very general constraints \eqref{eqn:ddp} and \eqref{eqn:deo}.

\subsection{Noise setting}
\label{ssec:CCN_and_MC}
Our case studies focus on two common special cases of MC learning: CCN and PU learning.
Under CCN noise the sensitive feature's value is randomly flipped with probability $\rho^+$ if its value was 1, or with probability $\rho^-$ if its value was 0. As shown in \citet[Appendix C]{corruption}, CCN noise is a special case of MC learning. 
For PU learning we consider the censoring setting~\citep{Elkan} which is a special case of CCN learning where one of $\rho^+$ and $\rho^-$ is 0.
While our results also apply to the case-controlled setting of PU learning~\citep{Ward:2009}, the former setting is more natural in our context. Note that from $\rho^+$ and $\rho^-$ one can obtain $\alpha$ and $\beta$ as described in~\citet{corruption}.

\subsection{Benchmarks}

For each case study, we evaluate our method
(termed {\sf cor scale});
recall this scales the input parameter $\tau$ using Theorem \ref{thm: sensitive noise reduction} and the values of $\rho^+$ and $\rho^-$, and then uses the fair classifier to perform classification.
We compare our method with three  different baselines.
The first two trivial baselines are applying the fair classifier directly on the non-corrupted data (termed {\sf nocor}) and on the corrupted data (termed {\sf cor}). While the first baseline is clearly the ideal, it won't be possible when only the corrupted data is available. The second baseline should show that there is indeed an empirical need to deal with the noise in some way and that it cannot simply be ignored.

The third, non-trivial, baseline (termed {\sf denoise}) is to first denoise $A$ and then apply the fair classifier on the denoised distribution. This denoising is done by applying the {\sf RankPrune} method in \cite{northcutt2017rankpruning}. Note that we provide {\sf RankPrune} with the same known values of $\rho^+$ and $\rho^-$ that we use to apply our scaling so this is a fair comparison to our method.
{Compared to {\sf denoise}, we do \emph{not} explicitly infer individual sensitive feature values;
thus, our method does not compromise privacy.}

For both case studies, we study the relationship between the input parameter $\tau$ and the testing error and fairness violation. 
For simplicity, we only consider the DP constraint.

\subsection{Case study: privacy preservation}\label{casestudy: privacy}
\input{case_study_privacy.tex}

%
\subsection{Case study: PU learning}\label{casestudy: pu}
\input{case_study_pu.tex}

%% file: case_study_privacy.tex

\begin{figure*}[t]
    \centering
    
    \subfloat[{\tt COMPAS} dataset (privacy case study).]{%
    \includegraphics[width=0.4\textwidth]{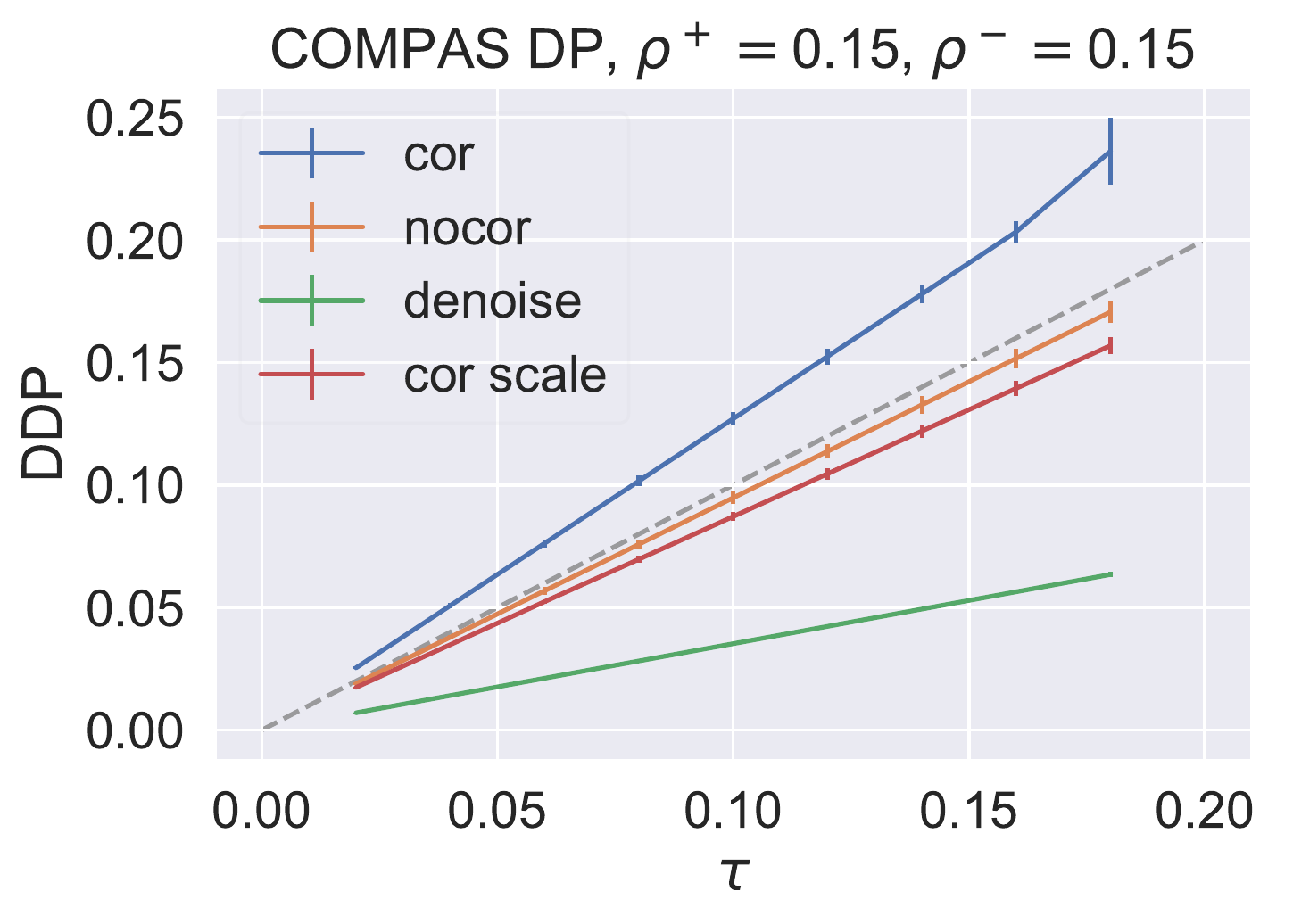}
    \includegraphics[width=0.4\textwidth]{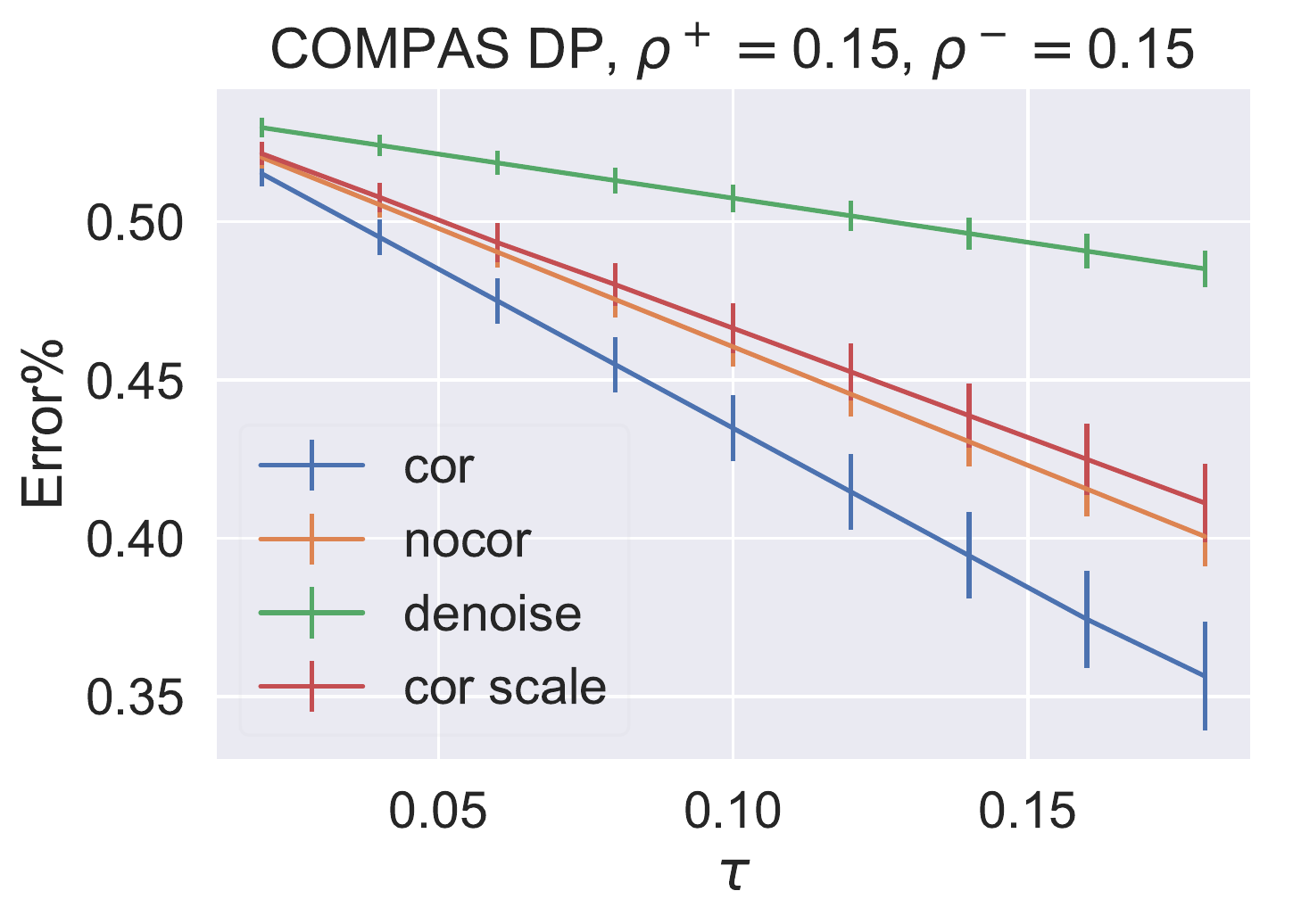}
    \label{fig:DP_compas}
    }
    
    \subfloat[{\tt law school} dataset (PU learning case study).]{%
    \includegraphics[width=0.4\textwidth]{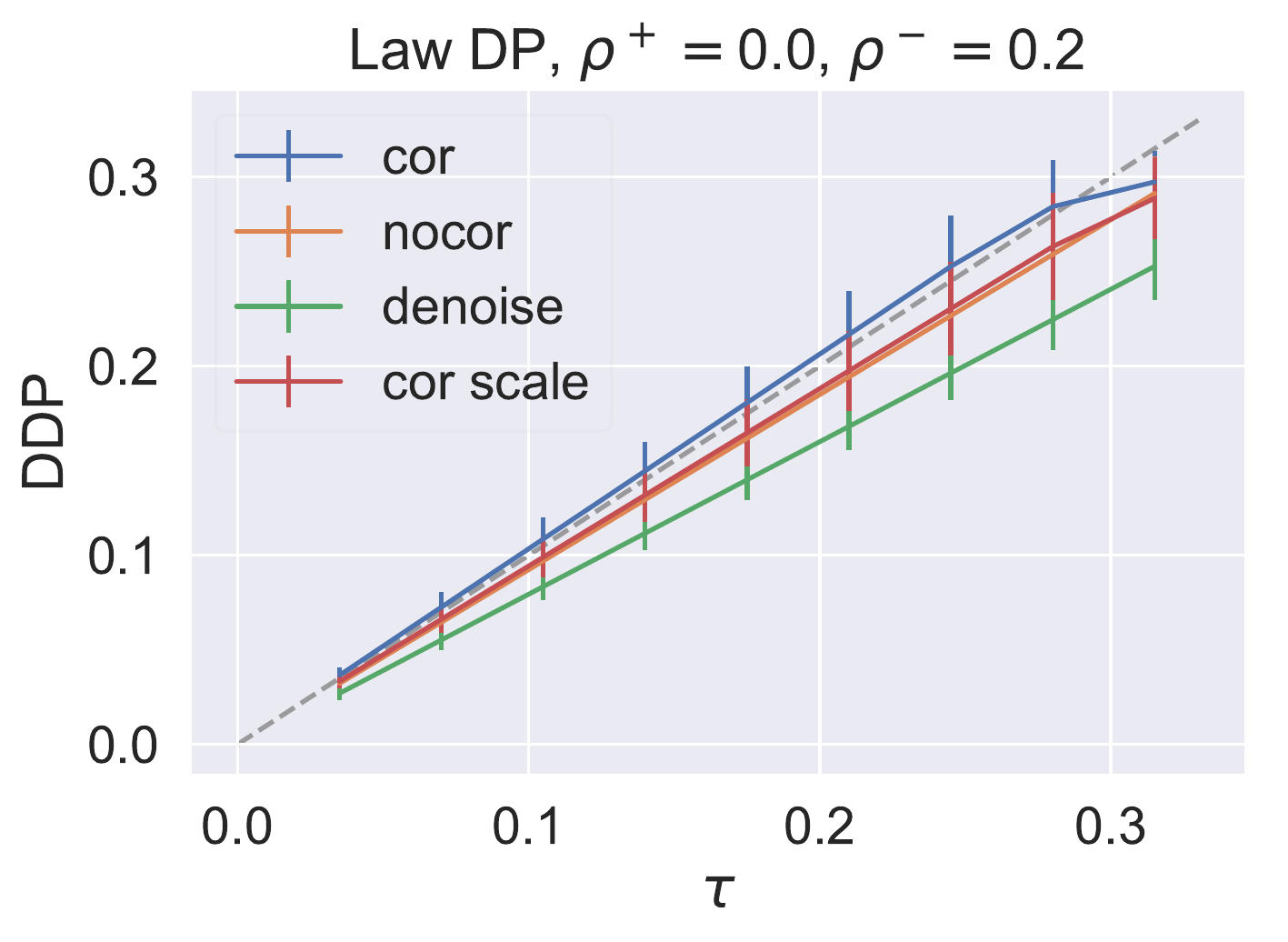} 
    \includegraphics[width=0.4\textwidth]{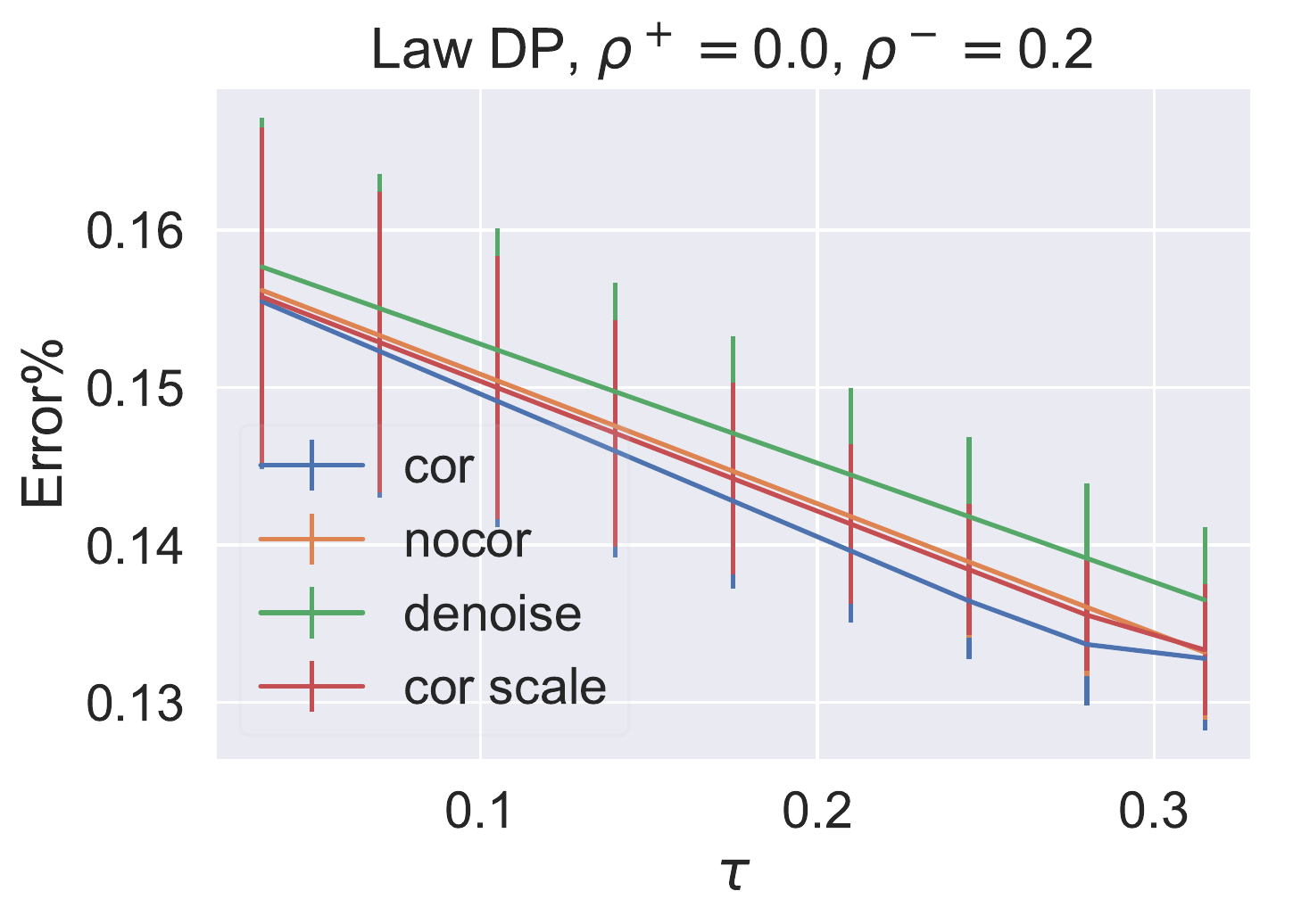}
    \label{fig:law}
    }
    
    \caption{Relationship between input fairness tolerance $\tau$ versus DP fairness violation (left panels), and versus error (right panels).
    {Our method ({\sf cor scale}) achieves approximately the ideal fairness violation (indicated by the gray dashed line in the left panels),
    with only a mild degradation in accuracy compared to training on the uncorrupted data (indicated by the {\sf nocor} method).
    Baselines that perform no noise-correction ({\sf cor})
    and explicitly denoise the data ({\sf denoise})
    offer suboptimal tradeoffs by comparison;
    for example, the former achieves slightly lower error rates, but does so at the expense of greater fairness violation.}}
\end{figure*}

In this case study, we look at {\tt COMPAS}, a dataset from Propublica~\citep{COMPAS} that is widely used in the study of fair algorithms.
Given various features about convicted individuals, the task is to predict recidivism and the sensitive attribute is race.
The data comprises 7918 examples and 10 features.
In our experiment, we assume that to preserve differential privacy, 
CCN noise with 
{$\rho^+=\rho^- = {0.15}$}
is added to the sensitive attribute. 
As per {Lemma} \ref{thm: randomized response for differential privacy}, this guarantees $(\epsilon, \delta=0)$ differential privacy with 
{$\epsilon = {1.73}$.}
We assume that the noise level $\rho$ is released with the dataset (and is thus known). 
We performed fair classification on this noisy data using our method and compare the results to the three benchmarks described above.

Figure \ref{fig:DP_compas} shows the average result over three runs each with a random 80-20 training-testing split.
(Note that fairness violations and errors are calculated with respect to the true uncorrupted features.)
{We draw two key insights from this graph:}
\begin{enumerate}[label=(\roman*),itemsep=-2pt,topsep=0pt,leftmargin=24pt]
    \item {in terms of fairness violation,
    our method ({\sf cor scale}) approximately achieves the desired fairness tolerance (shown by the {gray dashed} line).}
    This is both expected and ideal, and it matches what happens when there is no noise ({\sf nocor}).
    By contrast, the na\"{i}ve method {\sf cor} strongly violates the fairness constraint.

    \item {in terms of accuracy, 
    our method only suffers mildly
    compared with the ideal noiseless method ({\sf nocor});
some degradation} is expected as noise will lead to some loss of information.
By contrast, {\sf denoise} sacrifices much more {predictive} accuracy than our method.


\end{enumerate}
In light of both the above, our method is seen to achieve the best overall tradeoff between fairness and accuracy.
{Experimental results with EO constraints,
and other commonly studied datasets in the fairness literature ({\tt adult}, {\tt german}), 
show similar trends as}
in Figure \ref{fig:DP_compas}, and are included in Appendix \ref{appendix:privacy} for completeness.


%% file: case_study_pu.tex



\begin{figure*}[t]
    \centering
    \includegraphics[width=0.4\textwidth]{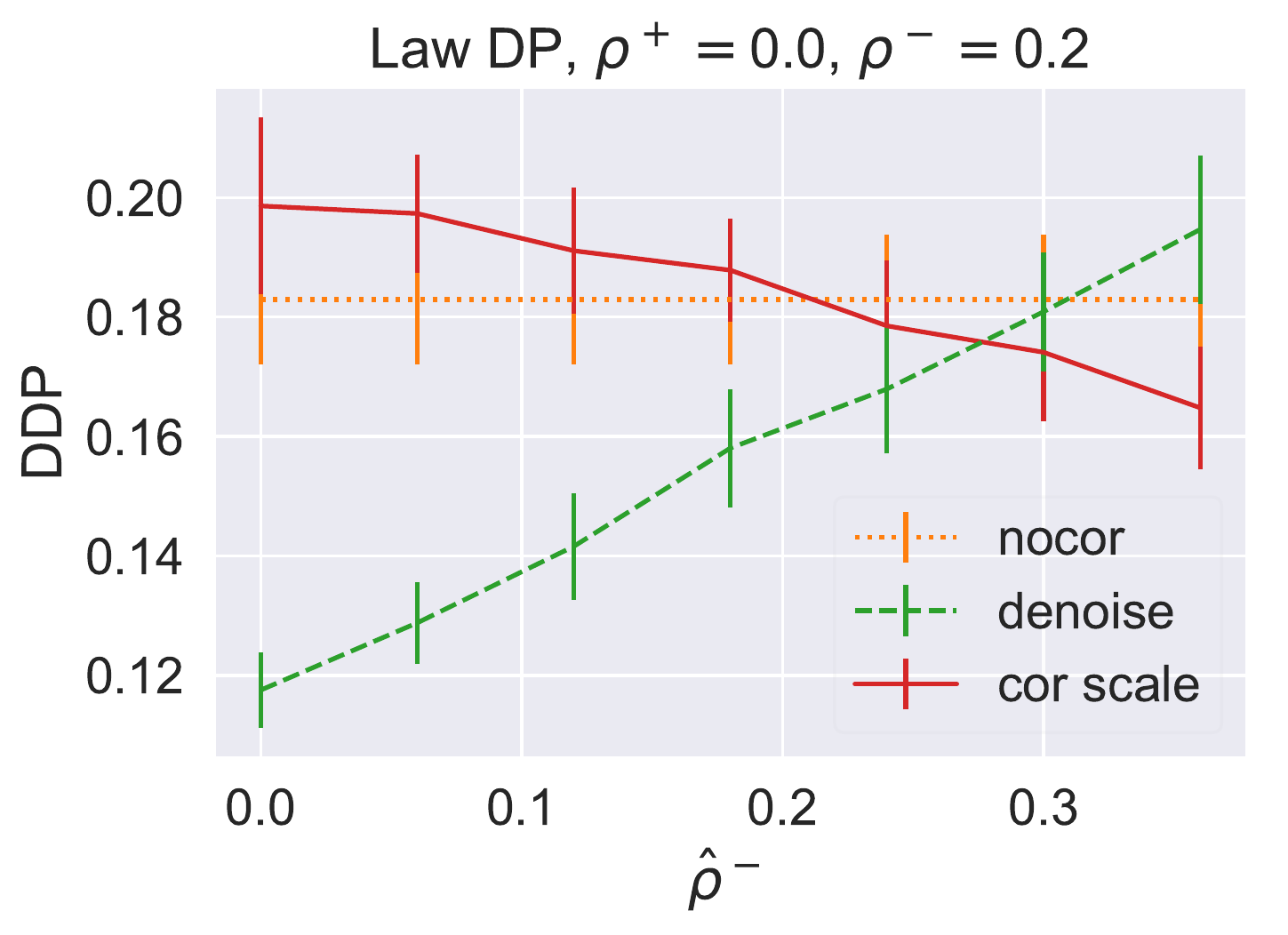}
    \includegraphics[width=0.4\textwidth]{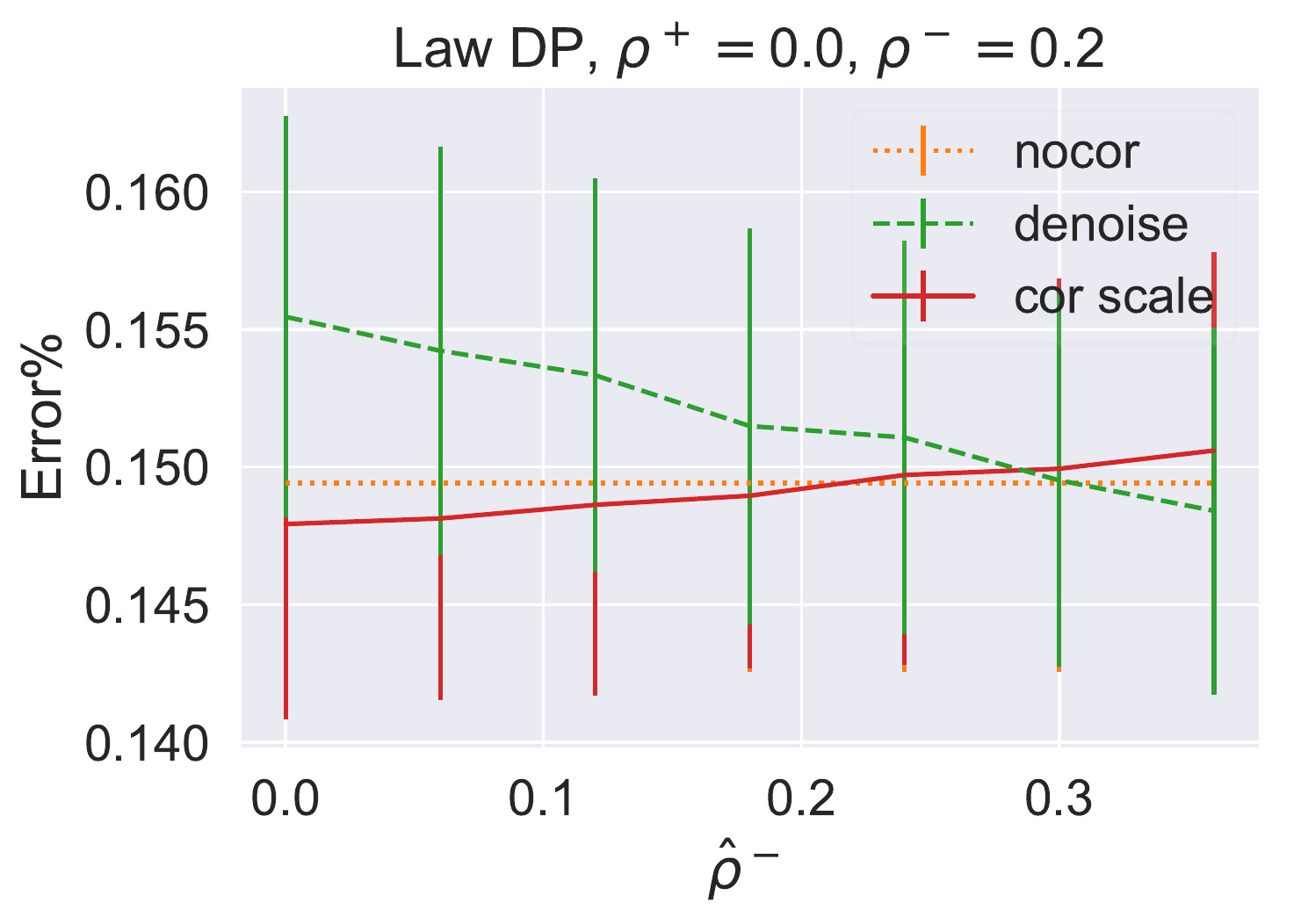}
    \caption{Relationship between the estimated noise level $\hat{\rho}^-$ and fairness violation/error on the {\tt law school} dataset using DP constraint (testing curves), with $\hat{\rho}^+ = 0$ and $\tau=0.2$. Our method ({\sf cor scale}) is not overly sensitive to imperfect estimates of the noise rate, evidenced by its fairness violation and accuracy closely tracking that of training on the uncorrupted data ({\sf nocor})
    as $\hat{\rho}^-$ is varied. That is, red curve in the left plot closely tracks the yellow reference curve.
    By contrast, the baseline that explicitly denoises the data ({\sf denoise}) deviates strongly from {\sf nocor},
    and is sensitive to small changes in $\hat{\rho}^-$.
    This illustrates that our method performs well even when noise rates must be estimated.}
    \label{fig:law_est_err}
\end{figure*}

%

In this case study, we consider the dataset {\tt law school}, which is a subset of the original dataset from LSAC~\citep{Wightman}. In this dataset, one is provided with information about various individuals (grades, part time/full time status, age, etc.) and must determine whether or not the individual passed the bar exam.
The sensitive feature is race; we only consider black and white.
After prepossessing the data by removing instances that had missing values and those belonging to other ethnicity groups (neither black nor white) we were left with 3738 examples each with 11 features. 

While the data ostensibly provides the true values of the sensitive attribute, one may imagine having access to only PU information.
Indeed, when the data is collected one could imagine that individuals from the minority group would have a much greater incentive to 
{conceal} their group membership due to fear of discrimination.
Thus, any individual identified as belonging to the majority group could be assumed to have been correctly identified (and would be part of the positive instances).
On the other hand, no definitive conclusions could be drawn about individuals identified as belonging to the minority group (these would therefore be part of the unlabelled instances). 

To model a PU learning scenario, we added CCN noise to the dataset with $\rho^+ = 0$ and 
{$\rho^- = {0.2}$.}
We initially assume that the noise rate is known. 
Figure \ref{fig:law} shows the average result over three runs under this setting each with a random 80-20 training-testing split. 
{We draw the same conclusion as before:
our method
achieves the highest accuracy while
respecting the specified fairness constraint.
}

Unlike in the privacy case, the noise rate in the PU learning scenario is usually unknown in practice,
and must be estimated.
{Such estimates will inevitably be approximate.}
We thus evaluate the impact of the error of the noise rate estimate on all methods.
{In Figure~\ref{fig:law_est_err}, we consider a PU scenario where 
we only have access to an estimate $\hat{\rho}^-$ of the negative noise rate,
whose true value is $\rho^- = 0.2$.}
Figure \ref{fig:law_est_err} shows the impact of different values of $\hat{\rho}^-$ on the fairness violation and error.
We see that that as long as this estimate is reasonably accurate, 
our method performs the best in terms of being closest to the case of running the fair algorithm on uncorrupted data.

In sum,
these results are consistent with our derivation and show that our method {\sf cor scale} can achieve the desired degree of fairness while minimising loss of accuracy.
Appendix \ref{appendix:pu} includes results for different settings of $\tau$, noise level, and on other datasets showing similar trends.

%% file: conclusion.tex
In this paper, we showed both theoretically and empirically that even under the very general MC learning noise model~\citep{Scott13} on the sensitive feature, fairness can still be preserved by scaling the input unfairness tolerance parameter $\tau$.
In future work, it would be interesting to consider the case of categorical sensitive attributes (as applicable, e.g., for race),
and the more challenging case of instance-dependent noise~\citep{Awasthi:2015}.
\AKMEDIT{We remark also that in independent work,~\citet{Awasthi:2019} studied the effect of sensitive attribute noise on the post-processing method of~\citet{Hardt2016}.
In particular, they 
identified conditions when such post-processing can still yield an approximately fair classifier.
Our approach has an advantage of being applicable to a generic in-processing fair classifier;
however, their approach also handles the case where the sensitive feature is used as an input to the classifier.
Exploration of the synthesis of the two approaches is another promising direction for future work.
}

%% file: 9appendix.tex
\begin{appendices}

\section{Proofs of results in the main body}
\label{sec:proof of main theorem}

\subsection{Proof of Lemma~\ref{lemma:mcrelation}}
\begin{proof} 
    Suppose that we have noise as given by Equation \eqref{equation:mcnoise1}. We denote by $A$ the random variable
    denoting the value of the true sensitive attribute and by $A_{\mathrm{corr}}$ the random variable denoting the value 
    of the corrupted sensitive attribute.

    Then, for any measurable subset of instances $U$,

    \ifpaper
    \begin{align*}
        &\mathbb{P}[X \in U \mid Y = 1 , A_{\mathrm{corr}} = 1]\\
        &= \frac{ \mathbb{P}[X \in U, Y=1 \mid A_{\mathrm{corr}} = 1] }{\mathbb{P}[Y=1 \mid A_{\mathrm{corr}} = 1]} \\
        &= \frac{ \mathbb{P}[X \in U, Y=1 \mid A_{\mathrm{corr}} = 1] }{(1-\alpha)\mathbb{P}[Y=1 \mid A = 1] + \alpha \mathbb{P}[Y=1 \mid A=0]} \\
        &= \frac{ (1-\alpha)\mathbb{P}[X \in U, Y=1 \mid A = 1] }{(1-\alpha)\mathbb{P}[Y=1 \mid A = 1] + \alpha \mathbb{P}[Y=1 \mid A=0]} \\
        &\phantom{=} + \frac{\alpha\mathbb{P}[X \in U, Y=1 \mid A = 0] }{(1-\alpha)\mathbb{P}[Y=1 \mid A = 1] + \alpha \mathbb{P}[Y=1 \mid A=0]} \\
        &= \frac{ (1-\alpha)\mathbb{P}[Y=1 \mid A=1]\mathbb{P}[X \in U \mid Y=1, A = 1]}{(1-\alpha)\mathbb{P}[Y=1 \mid A = 1] + \alpha \mathbb{P}[Y=1 \mid A=0]} \\
        &\phantom{=} + \frac{\alpha\mathbb{P}[Y=1 \mid A = 0]\mathbb{P}[X \in U \mid Y=1, A = 0] }{(1-\alpha)\mathbb{P}[Y=1 \mid A = 1] + \alpha \mathbb{P}[Y=1 \mid A=0]} \\
        &= (1-\alpha')\mathbb{P}[X \in U \mid Y=1, A = 1] \\
        &\phantom{=} + \alpha' \mathbb{P}[X \in U \mid Y=1, A = 0],
   \end{align*}
   \fi
   
    \ifarxiv
    \begin{align*}
        &\mathbb{P}[X \in U \mid Y = 1 , A_{\mathrm{corr}} = 1]\\
        &= \frac{ \mathbb{P}[X \in U, Y=1 \mid A_{\mathrm{corr}} = 1] }{\mathbb{P}[Y=1 \mid A_{\mathrm{corr}} = 1]} \\
        &= \frac{ \mathbb{P}[X \in U, Y=1 \mid A_{\mathrm{corr}} = 1] }{(1-\alpha)\mathbb{P}[Y=1 \mid A = 1] + \alpha \mathbb{P}[Y=1 \mid A=0]} \\
        &= \frac{ (1-\alpha)\mathbb{P}[X \in U, Y=1 \mid A = 1] }{(1-\alpha)\mathbb{P}[Y=1 \mid A = 1] + \alpha \mathbb{P}[Y=1 \mid A=0]} + \frac{\alpha\mathbb{P}[X \in U, Y=1 \mid A = 0] }{(1-\alpha)\mathbb{P}[Y=1 \mid A = 1] + \alpha \mathbb{P}[Y=1 \mid A=0]} \\
        &= \frac{ (1-\alpha)\mathbb{P}[Y=1 \mid A=1]\mathbb{P}[X \in U \mid Y=1, A = 1]}{(1-\alpha)\mathbb{P}[Y=1 \mid A = 1] + \alpha \mathbb{P}[Y=1 \mid A=0]} + \frac{\alpha\mathbb{P}[Y=1 \mid A = 0]\mathbb{P}[X \in U \mid Y=1, A = 0] }{(1-\alpha)\mathbb{P}[Y=1 \mid A = 1] + \alpha \mathbb{P}[Y=1 \mid A=0]} \\
        &= (1-\alpha')\mathbb{P}[X \in U \mid Y=1, A = 1] + \alpha' \mathbb{P}[X \in U \mid Y=1, A = 0],
   \end{align*}
   \fi

   where in the last equality we set 
   \begin{equation*}
    \resizebox{
      \ifdim\width>\linewidth
        \linewidth
      \else
        \width
      \fi
    }{!}{$\displaystyle
    \alpha' := \frac{\alpha\mathbb{P}[Y=1 \mid A = 0] }{(1-\alpha)\mathbb{P}[Y=1 \mid A = 1] + \alpha \mathbb{P}[Y=1 \mid A=0]}.$%
    }
   \end{equation*}
   Note that the last equality is equivalent to the first equality of Equation \eqref{equation:mcnoise2}
   with $\alpha'$ as in the lemma.
   
   The proof for $\beta'$ is exactly the same and simply expands $\mathbb{P}[X \in U \mid Y = 1 , A_{\mathrm{corr}} = 0]$
   instead of $\mathbb{P}[X \in U \mid Y = 1 , A_{\mathrm{corr}} = 1]$.
\end{proof}

\subsection{Proof of Theorem~\ref{thm: sensitive noise reduction}}
\begin{proof}
    For the DP-like constraints simply note that by definition of $D_{\mathrm{corr}}$ we have that
    $$\Lf_{D_{0,\cdot, \mathrm{corr}}}(f) = (1-\beta) \cdot \Lf_{D_{0,\cdot}}(f) + \beta \cdot \Lf_{D_{1,\cdot}}(f)$$
    and similarly,
    $$\Lf_{D_{1,\cdot, \mathrm{corr}}}(f) = (1-\alpha) \cdot \Lf_{D_{1,\cdot}}(f) + \alpha \cdot \Lf_{D_{0,\cdot}}(f)$$
    Thus we have that
    \ifpaper
    \begin{align*}
        \Lf_{D_{0,\cdot, \mathrm{corr}}}(f)-\Lf_{D_{1,\cdot, \mathrm{corr}}}(f) = &\, (1-\alpha-\beta) \cdot \\
        &(\Lf_{D_{0,\cdot}}(f) - \Lf_{D_{1,\cdot}}(f)),
    \end{align*}
    \fi
    
    \ifarxiv
    \begin{align*}
        \Lf_{D_{0,\cdot, \mathrm{corr}}}(f)-\Lf_{D_{1,\cdot, \mathrm{corr}}}(f) = &\, (1-\alpha-\beta) \cdot (\Lf_{D_{0,\cdot}}(f) - \Lf_{D_{1,\cdot}}(f)),
    \end{align*}
    \fi
    which immediately implies the desired result.
    
    The result for the EO constraint is obtained in the exact same way by simply replacing 
    $D_{a, \cdot}$ with 
    $D_{a, 1}$, $D_{a, \cdot, \text{corr}}$ with $D_{a, 1, \text{corr}}$, and $\alpha$ and $\beta$ with
    $\alpha'$ and $\beta'$.
\end{proof}

\subsection{Proof of Lemma~\ref{thm: randomized response for differential privacy}}
\begin{proof}
Basic definitions in Differential Privacy are provided in Appendix~\ref{sec:diff-privacy}. Consider an instance $\{x_i,y_i,a_i\}$ with only $x_i$ disclosed by an attacker. Assume the sensitive attribute $a_i$ is queried. Denote $\hat{a}_i$ to be the sensitive attribute of the instance after adding noise i.e. being flipped with probability $\rho$. The attacker is interested in knowing if $a_i=0$ or $a_i=1$ by querying $\hat{a}_i$.

Since
\begin{align*}
    \frac{\mathbb{P} [\hat{a}_i=1 | a_i=1]}{\mathbb{P} [\hat{a}_i=1 | a_i=0]}=\frac{\mathbb{P} [\hat{a}_i=0 | a_i=0]}{\mathbb{P} [\hat{a}_i=0 | a_i=1]}
\end{align*}
we can reason in a similar way for $\hat{a}_i=0$. Thus, let us focus on the case where $\hat{a}_i=1$.  Let us consider two neighbor instances $\{x_i,y_i,0\}$ and $\{x_i,y_i,1\}$.
Essentially, we want to upper-bound the ratio 
\[
\frac{\mathbb{P} [\hat{a}=1|a_i=1]}{\mathbb{P} [\hat{a}=1|a_i=0]} := \frac{1-\rho}{\rho}
\]
by $\exp (\epsilon)$, and lower-bound the ratio by $\exp (-\epsilon)$.
The lower bound is always true since $\rho<0.5$. For the upper-bound, We have:
\begin{align*}
    \frac{1-\rho}{\rho} \leq \exp{(\epsilon)} \iff \rho \geq \frac{1}{\exp (\epsilon)+1}.
\end{align*}
\end{proof}

\clearpage

\section{Background on Differential Privacy}
\label{sec:diff-privacy}

The following definitions are from Appendix~\citet{Dwork06}. They are used for the proof of Lemma~\ref{thm: randomized response for differential privacy} in~\ref{sec:proof of main theorem}.

\textbf{Probability simplex}: Given a discrete set $B$, the probability simplex over $B$, denoted $\Delta (B)$ is:
\[
\Delta (B) = \{ x \in \mathbb{R}^{\vert B \vert }: \forall i, x_i \geq 0, \text{ and } \sum_{i=1}^{\vert B \vert} x_i = 1\}.
\]
\textbf{Randomized Algorithms}: A randomized algorithm $\mathcal{M}$ with domain $A$ and range $B$ is an algorithm associated with a total map $M:A\rightarrow \Delta (B)$. On input $a \in A$, the algorithm $\mathcal{M}$ outputs $\mathcal{M}(a)=b$ with probability $(M(a))_b$ for each $b \in B$. The probability space is over the coin flips of the algorithm $\mathcal{M}$.

For simplicity we will avoid implementation details and we will consider databases as histograms. Given a universe $\mathcal{X}$ an histogram over $\mathcal{X}$ is an object in $\mathbb{N}^{\vert \mathcal{X} \vert}$. We can bake in the presence or absence of an individual notion in a definition of distance between databases.\\

\textbf{Distance Between Databases}: The $l_1$ norm $\| x \|_1$ of a database $x \in \mathbb{N}^{\vert \mathcal{X} \vert}$ is defined as:
\[
\| x \|_1 = \sum_{i=1}^{\vert \mathcal{X} \vert} x_i.
\]
The $l_1$ distance between two databases $x$ and $y$ is defined as  $\| x - y \|_1$.\\

\textbf{Differential Privacy}: A randomized algorithm $\mathcal{M}$ with domain $\mathbb{N}^{\vert \mathcal{X} \vert}$ is $(\epsilon, \delta)$-differentially private if for all $S \subseteq$ Range$(\mathcal{M})$ and for all $x,y\in \mathbb{N}^{\mathcal{X}}$ such that $\| x-y \|_1 \leq 1$:
\[
\mathbb{P} [\mathcal{M}(x) \in S] \leq \exp(\epsilon) \cdot \mathbb{P} [\mathcal{M}(y)\in S]+\delta,
\]
where the probability space is over the coin flips of the mechanism $\mathcal{M}$.

\clearpage

\section{Relationship between mean-difference score and the constraint used in \citet{reduction}}
\label{appendix:agarwal_constraint}

\citet{reduction} adopts slightly different fairness constraints than ours. Using our notation and letting $\signf(X) = \sign(f(X))$, instead of bounding  $\Lambda^{\mathrm{DP}}_D( f )$ by $\tau$, they bound 
$$\max_{a\in\{0,1\}} \abs{\mathbb{E}_{D_{a,\cdot}}[\signf(X)] - \mathbb{E}_D[\signf(X)]}$$ and $$\max_{a\in\{0,1\}} \abs{\mathbb{E}_{D_{a,1}}[\signf(X)] - \mathbb{E}_{D_{\cdot, 1}}[\signf(X)]}$$ for DP and EO respectively  by $\tau$.
The two have the following relationship.
\begin{theorem}
Under the setting of fair binary classification with a single binary sensitive attribute and using $\ellf(s, y) = \1[\sign(s)]$ we have that
\[
\max_{a\in\{0,1\}} \abs{\mathbb{E}_{D_{a, \cdot}}[\signf(X)] - \mathbb{E}_D[\signf(X)]} = \max_{a\in\{0,1\}} ( \mathbb{P}[A = 0], \mathbb{P}[A = 1] ) \Lambda^{\mathrm{DP}}_D( f )
\]
and
\[
\max_{a\in\{0,1\}} \abs{\mathbb{E}_{D_{a, 1}}[\signf(X)] - \mathbb{E}_{D_{\cdot, 1}}[\signf(X)]} = \max_{a\in\{0,1\}} ( \mathbb{P}[A = 0 \mid Y = 1], \mathbb{P}[A = 1 \mid Y = 1] ) \Lambda^{\mathrm{EO}}_D( f )
\]
\end{theorem}
\begin{proof}
For the DP case,
\begin{align*}
    &\abs{\mathbb{E}_{D_{1,\cdot}}[\signf(X)] - \mathbb{E}_D[\signf(X)]}
    \\&= \abs{\mathbb{E}_{D_{1,\cdot}}[\signf(X)] - (\mathbb{P}[A=1]\mathbb{E}_{D_{1,\cdot}}[\signf(X)]+\mathbb{P}[A=0]\mathbb{E}_{D_{0,\cdot}}[\signf(X)])}\\
    &= \abs{(1-\mathbb{P}[A=1])\mathbb{E}_{D_{1,\cdot}}[\signf(X)]-\mathbb{P}[A = 0]\mathbb{E}_{D_{0,\cdot}}[\signf(X)]}\\
    &= \abs{\mathbb{P}[A=0]\mathbb{E}_{D_{1,\cdot}}[\signf(X)]-\mathbb{P}[A=0]\mathbb{E}_{D_{0,\cdot}}[\signf(X)]}\\
    &= \mathbb{P}[A=0]\abs{(\mathbb{E}_{D_{1,\cdot}}[\signf(X)]-\mathbb{E}_{D_{0,\cdot}}[\signf(X)])}\\
    &= \mathbb{P}[A=0]\abs{ \Lf_{D_{0, \cdot}}(f) - \Lf_{D_{1, \cdot}}(f) }\\
    &= \mathbb{P}[A=0]\Lambda^{\mathrm{DP}}_D( f )
\end{align*}

and similarly
\[
\abs{\mathbb{E}_{D_{0,\cdot}}[\signf(X)]-\mathbb{E}_D[\signf(X)]} = \mathbb{P}[A=1]\Lambda^{\mathrm{DP}}_D( f )
\]
so the theorem holds.

The result for the EO case is proved in exactly the same way by simply replacing $\mathbb{P}[A=0], \mathbb{P}[A=1],$ $D_{a,\cdot}$ and $D$ with $\mathbb{P}[A=0 \mid Y=1], \mathbb{P}[A=1 \mid Y=1],$ $D_{a,1}$ and $D_{\cdot, 1}$ respectively.

\end{proof}

We then have the following as an immediate corollary.
\begin{corollary}
Assuming that we have noise as described above by Equation \eqref{equation:mcnoise1} and that we take $\ellf(s, y) = \1[\sign(s)]$ then we have that if $\max_{a\in\{0,1\}} ( \mathbb{P}_D[A = 0], \mathbb{P}_D[A = 1] ) = \max_{a\in\{0,1\}} ( \mathbb{P}_{D_\mathrm{corr}}[A = 0], \mathbb{P}_{D_\mathrm{corr}}[A = 1] )$ then:
\[
\max_{a\in\{0,\cdot\}} \abs{\mathbb{E}_{D_{a,\cdot}}[\signf(X)] - \mathbb{E}_D[\signf(X)]} < \tau \iff \max_{a\in\{0,1\}} \abs{\mathbb{E}_{D_{a,\cdot,\textrm{corr}}}[\signf(X)] - \mathbb{E}_{D_{\textrm{corr}}}[\signf(X)]} < \tau \cdot (1-\alpha-\beta).
\]
And if $\max_{a\in\{0,1\}} ( \mathbb{P}_{D_{\cdot, 1}}[A = 0], \mathbb{P}_{D_{\cdot, 1}}[A = 1] ) = \max_{a\in\{0,1\}} ( \mathbb{P}_{D_{\cdot, 1, \mathrm{corr}}}[A = 0], \mathbb{P}_{D_{\cdot, 1, \mathrm{corr}}}[A = 1] )$ then: 
\[
\max_{a\in\{0,1\}} \abs{\mathbb{E}_{D_{a,1}}[\signf(X)] - \mathbb{E}_{D_{\cdot,1}}[\signf(X)]} < \tau \iff \max_{a\in\{0,1\}} \abs{\mathbb{E}_{D_{a,1,\textrm{corr}}}[\signf(X)] - \mathbb{E}_{D_{\cdot,1,\textrm{corr}}}[\signf(X)]} < \tau \cdot (1-\alpha'-\beta').
\]
\end{corollary}
Even if the noise does not satisfy these new assumptions, we can still bound the constraint. Note that both $\max_{a\in\{0,1\}} ( \mathbb{P}[A = 0], \mathbb{P}[A = 1] )$ and $\max_{a\in\{0,1\}} ( \mathbb{P}[A = 0 \mid Y = 1], \mathbb{P}[A = 1 \mid Y = 1] )$ have values between $0.5$ and $1$. Thus,
\[
\frac{1}{2} \Lambda^{\mathrm{DP}}_D( f ) \leq \max_{a\in \{0,1\}}\abs{\mathbb{E}_{D_{a,\cdot}}[\signf(X)]-\mathbb{E}_D[\signf(X)]} \leq \Lambda^{\mathrm{DP}}_D( f )
\]
\[
\frac{1}{2} \Lambda^{\mathrm{EO}}_D( f ) \leq \max_{a\in \{0,1\}}\abs{\mathbb{E}_{D_{a,1}}[\signf(X)]-\mathbb{E}_{D_{\cdot,1}}[\signf(X)]} \leq \Lambda^{\mathrm{EO}}_D( f ),
\] and therefore the following corollary holds:
\begin{corollary}
Assuming that we have noise as described above by Equation \eqref{equation:mcnoise1} and that we take $\ellf(s, y) = \1[\sign(s)]$ then we have that:
\[
 \max_{a\in\{0,1\}} \abs{\mathbb{E}_{D_{a,\cdot,\textrm{corr}}}[\signf(X)] - \mathbb{E}_{D_{\textrm{corr}}}[\signf(X)]} < \frac{1}{2} \tau \cdot (1-\alpha-\beta)
 \Rightarrow
 \max_{a\in\{0,1\}} \abs{\mathbb{E}_{D_{a,\cdot}}[\signf(X)] - \mathbb{E}_D[\signf(X)]} < \tau
\]
and,
\[
 \max_{a\in\{0,1\}} \abs{\mathbb{E}_{D_{a,1,\textrm{corr}}}[\signf(X)] - \mathbb{E}_{D_{\cdot,1,\textrm{corr}}}[\signf(X)]} < \frac{1}{2} \tau \cdot (1-\alpha'-\beta')
 \Rightarrow
 \max_{a\in\{0,1\}} \abs{\mathbb{E}_{D_{a,1}}[\signf(X)] - \mathbb{E}_{D_{\cdot,1}}[\signf(X)]} < \tau.
\]
\end{corollary}

In addition to giving a simple way to use the classifier of \citet{reduction} without any modification, these results seem to indicate that with small modifications our scaling method can apply to an even wider range of fair classifiers than formally shown.

\clearpage

\section{More results for the privacy case study}
\label{appendix:privacy}
In this section we give some additional results for the privacy case study. 

{Figure~\ref{fig:DP_compas_full} shows additional results on {\tt COMPASS} for different noise levels $\rho^+ = \rho^- \in \{0.15, 0.3\}$.}

\begin{figure*}[!h]
    \centering
    \includegraphics[width=0.24\textwidth]{img_privacy/{disp_test_compas,0.15,0.15,1.0,DP,Agarwal,3,False}.pdf}
    \includegraphics[width=0.24\textwidth]{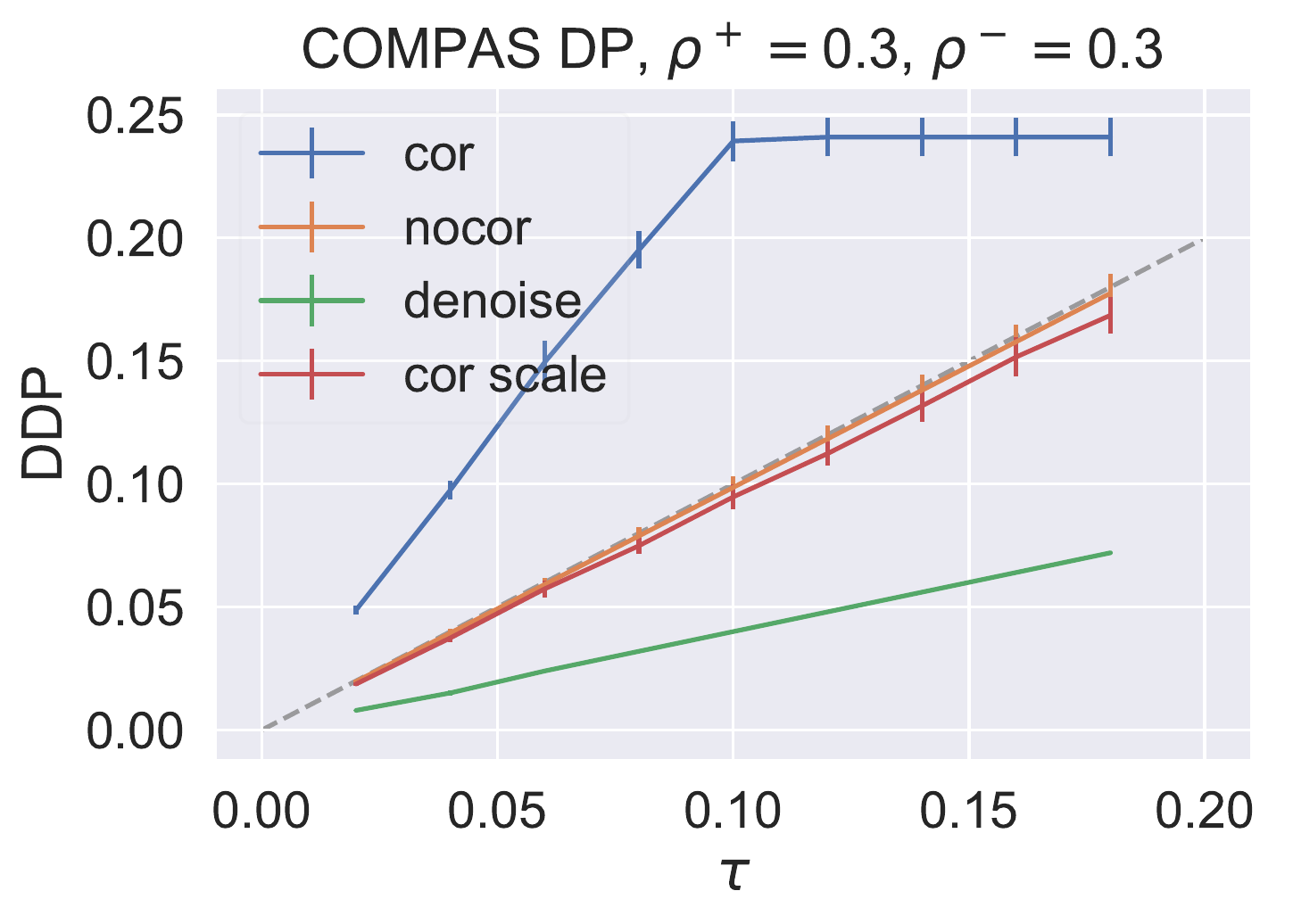}
    \includegraphics[width=0.24\textwidth]{img_privacy/{error_test_compas,0.15,0.15,1.0,DP,Agarwal,3,False}.pdf}
    \includegraphics[width=0.24\textwidth]{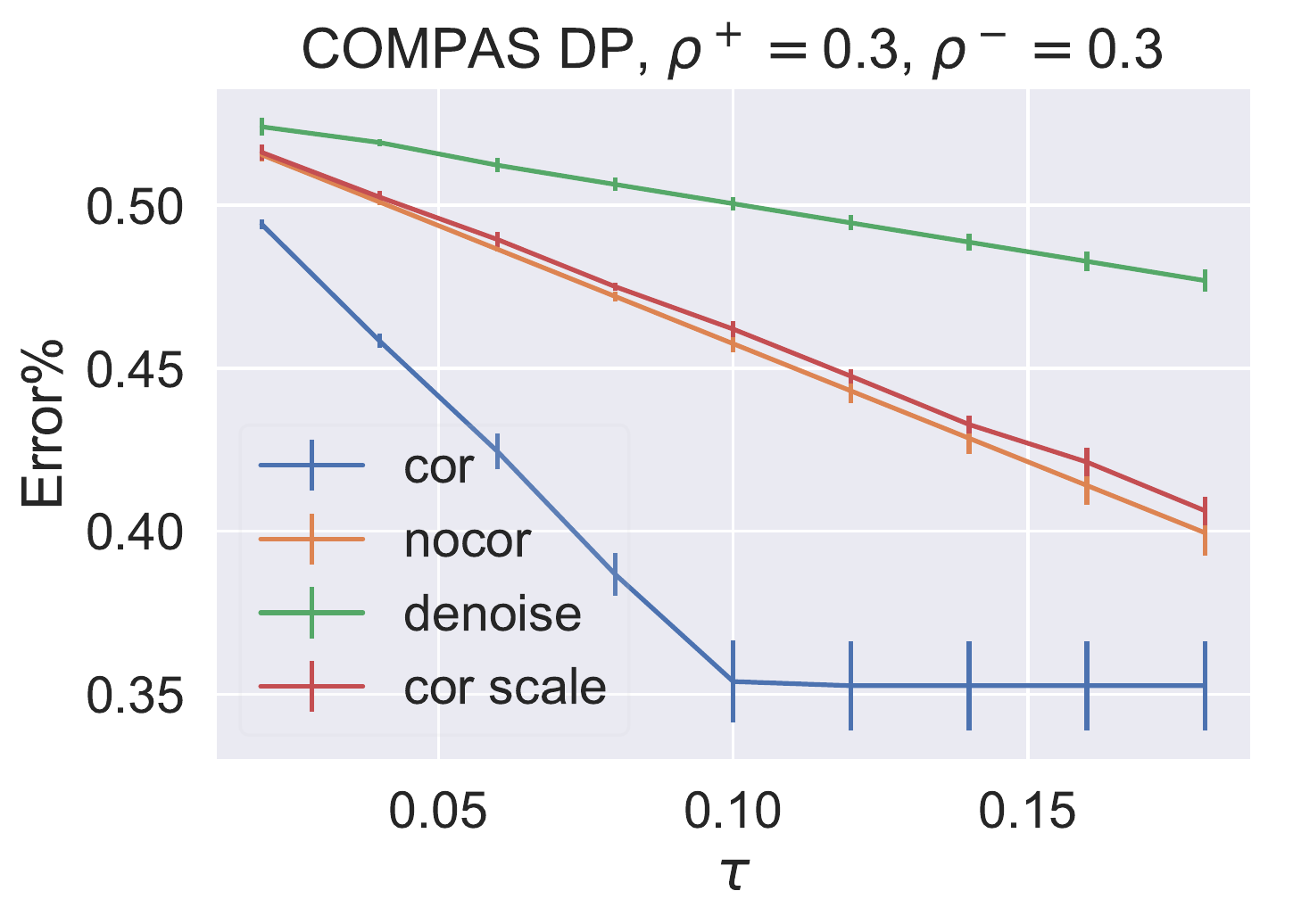}

    \caption{Relationship between input $\tau$ and fairness violation/error on the {\tt COMPAS} dataset using DP constraint (testing curves). The 
    {gray dashed line} represents the ideal fairness violation.}
    \label{fig:DP_compas_full}
\end{figure*}

Figure~\ref{fig:EO_compas} shows the results under the EO constraint for the {\tt COMPAS} dataset. That is, the dataset and setting is the same as described in section \ref{casestudy: privacy} but with the EO constraint instead of the DP constraint. We see that the trends are the same.

\begin{figure*}[h]
    \centering
    \includegraphics[width=0.24\textwidth]{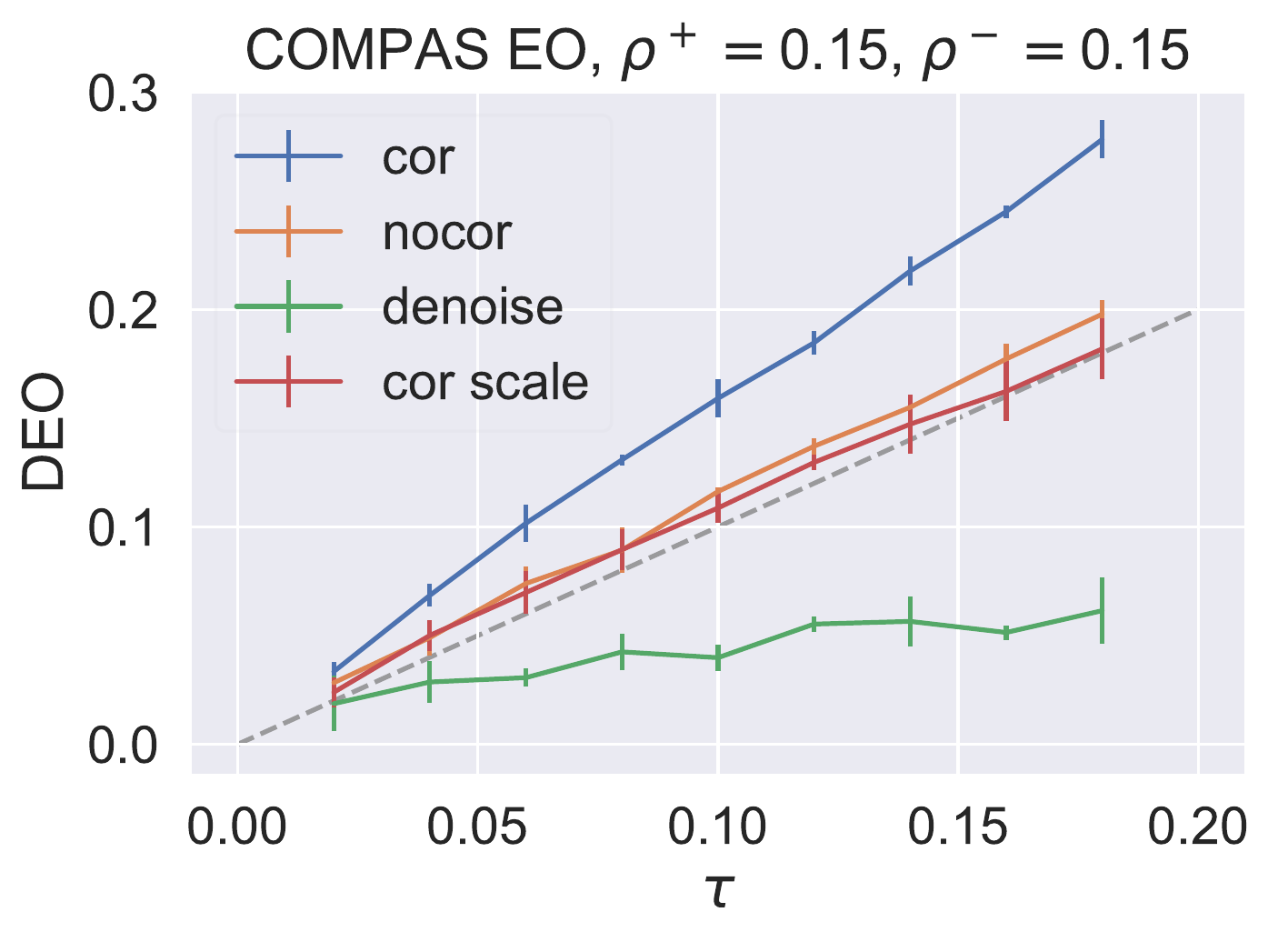}
    \includegraphics[width=0.24\textwidth]{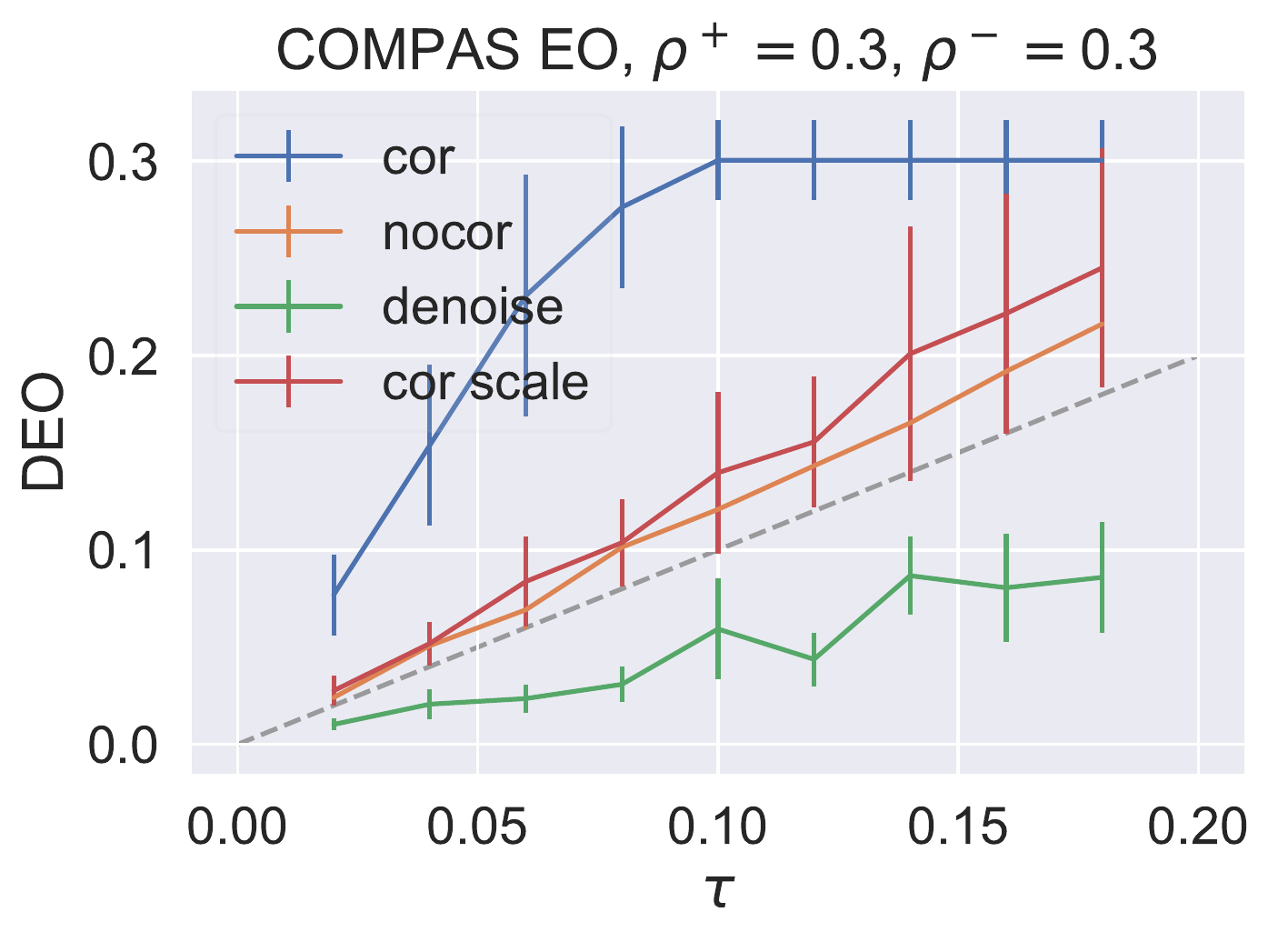}
    \includegraphics[width=0.24\textwidth]{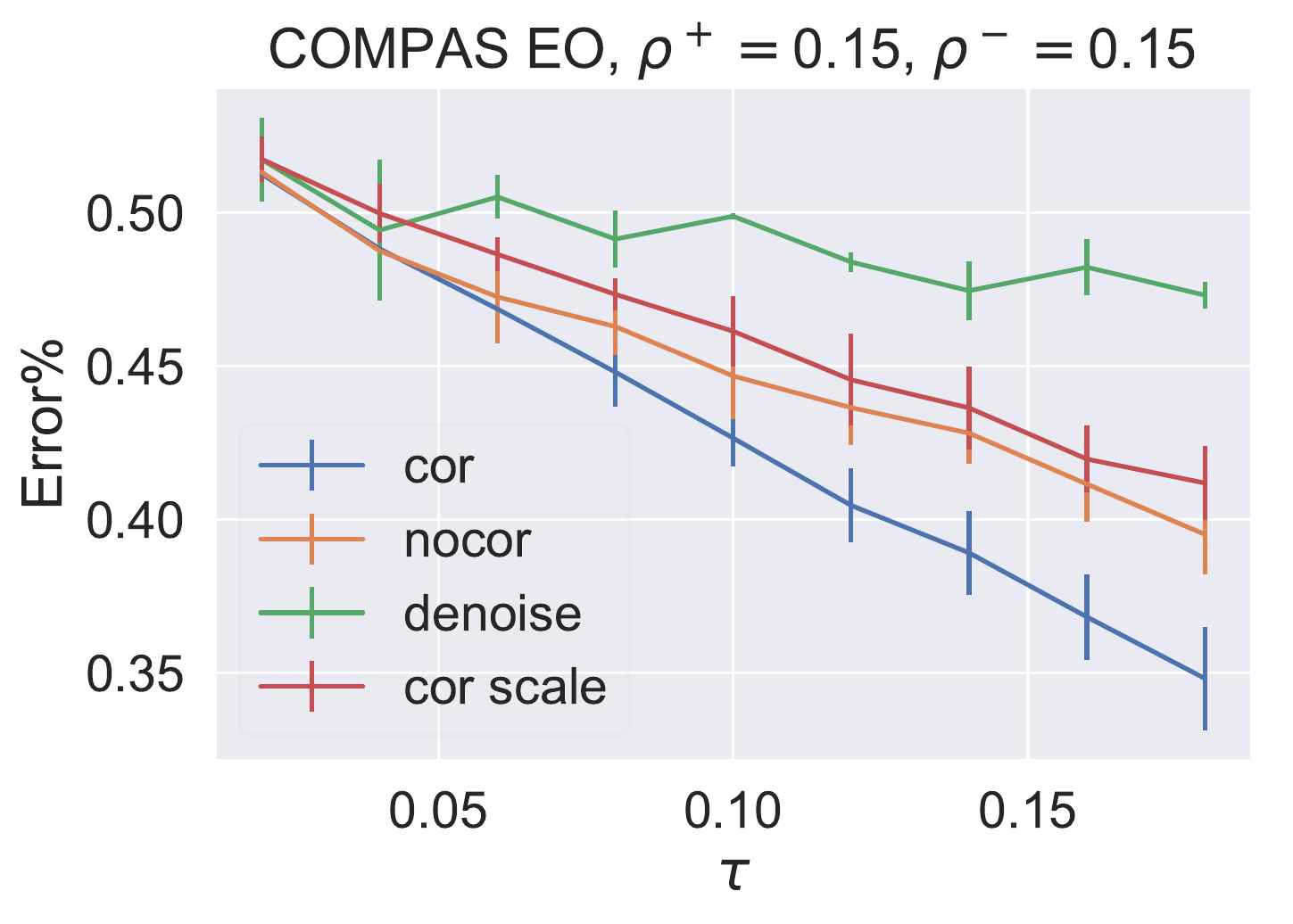}
    \includegraphics[width=0.24\textwidth]{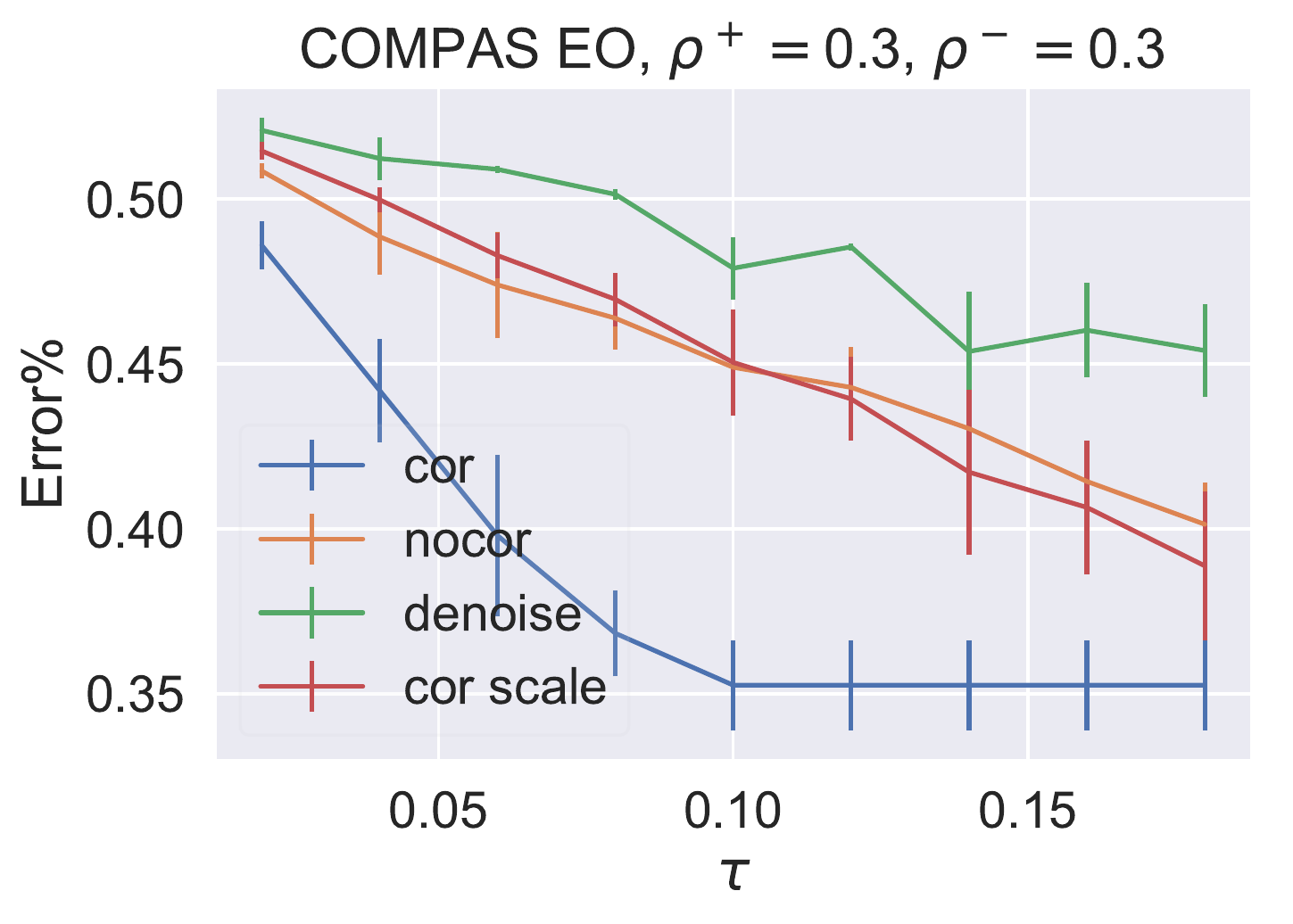}
    
    \includegraphics[width=0.24\textwidth]{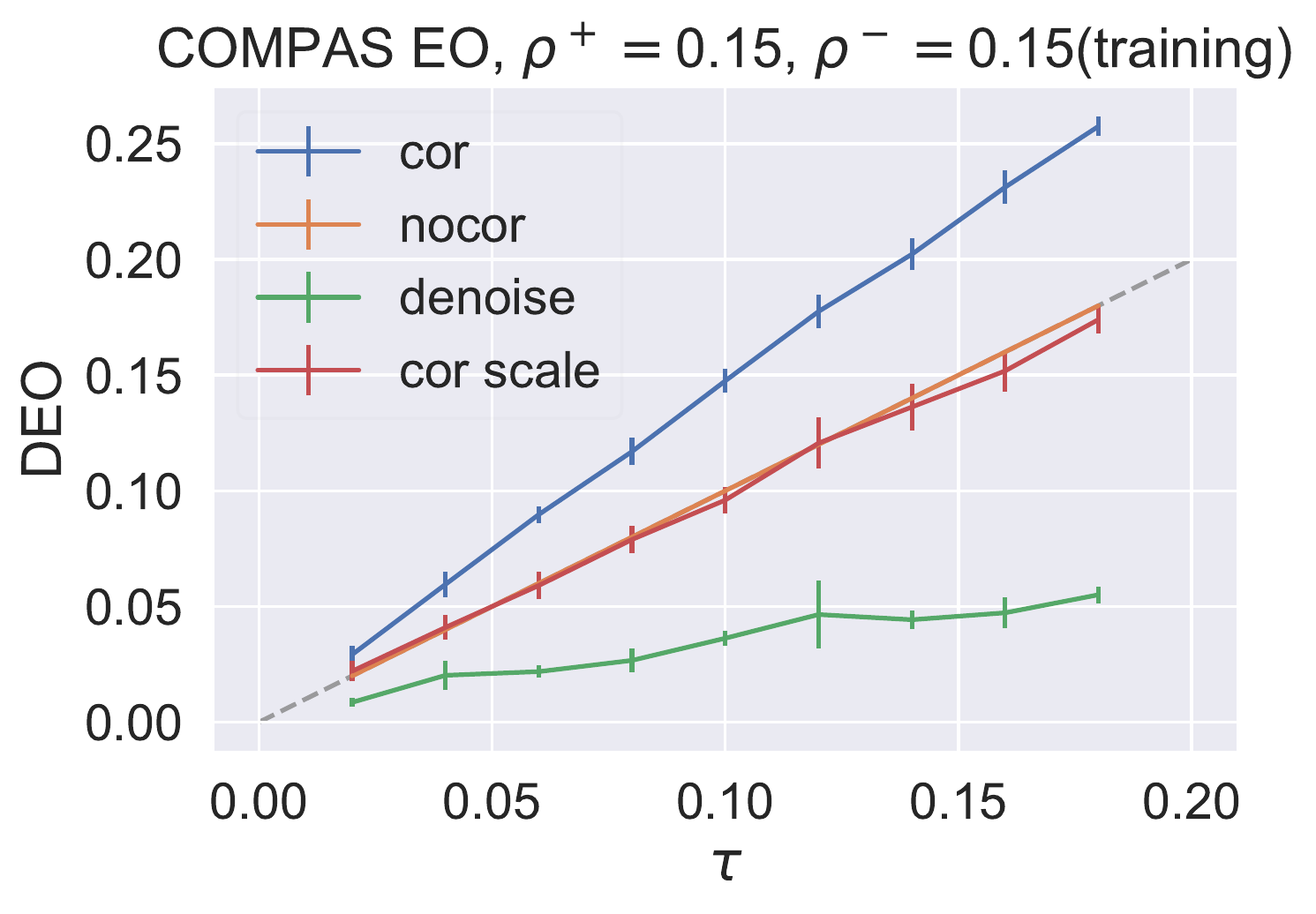}
    \includegraphics[width=0.24\textwidth]{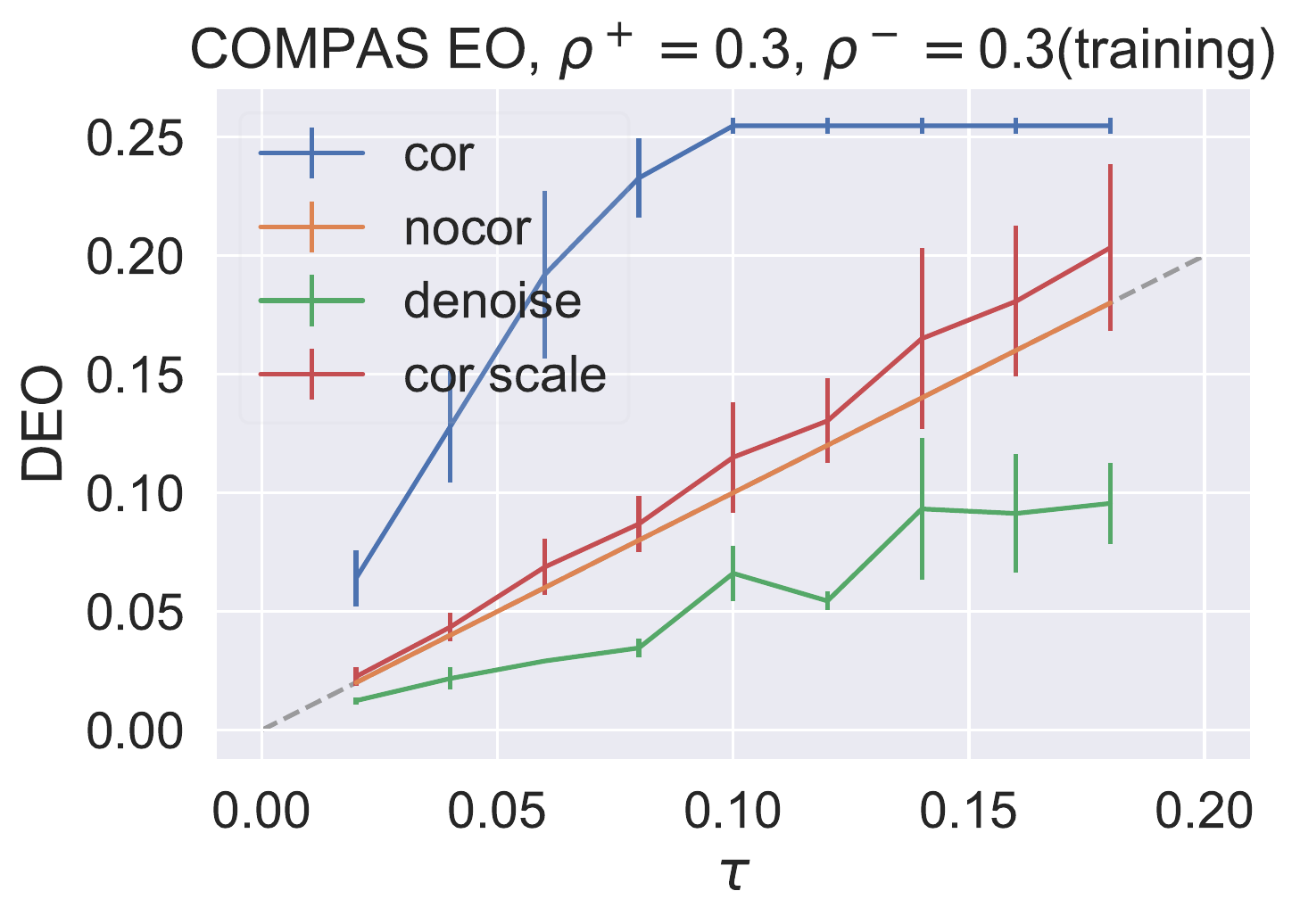}
    \includegraphics[width=0.24\textwidth]{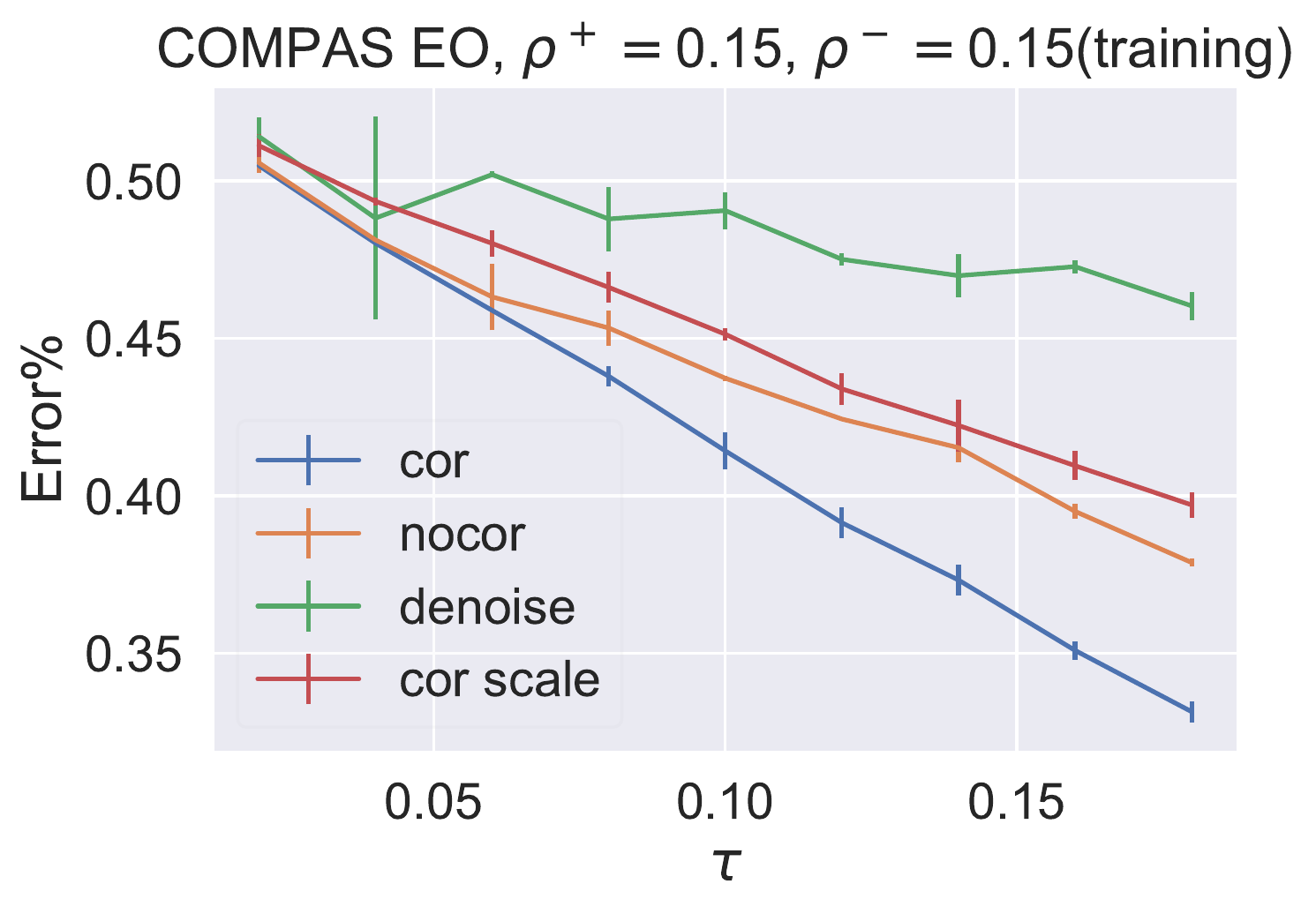}
    \includegraphics[width=0.24\textwidth]{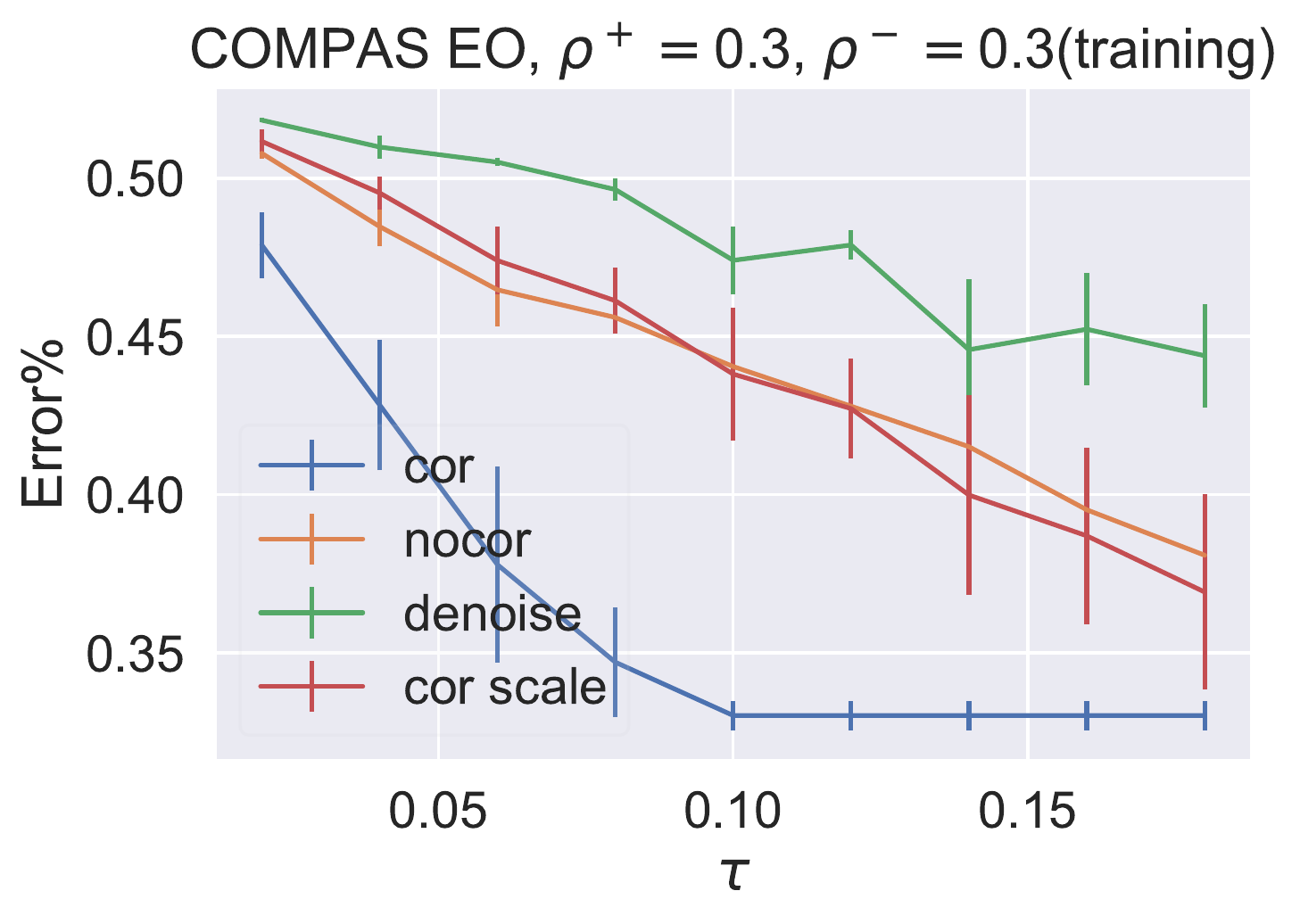}
    
    \caption{(EO)(testing and training) Relationship between input $\tau$ and fairness violation/error on the {\tt COMPAS} dataset.}
    \label{fig:EO_compas}
\end{figure*}

Figures \ref{fig:bank_training} and Figure \ref{fig:bank_testing} show results on the {\tt bank} dataset~\citep{Bank} with the DP and EO constraints respectively. This dataset is a subset of the original Bank Marketing dataset from the UCI repository~\citep{UCI}. The task is to predict if a client subscribes a term deposit. The sensitive attribute is if a person is middle aged(i.e. has an age between 25 and 60). The data comprises 11162 examples and 17 features. Again we note that the trends are the same.

\begin{figure*}[h]
    \centering
    \includegraphics[width=0.24\textwidth]{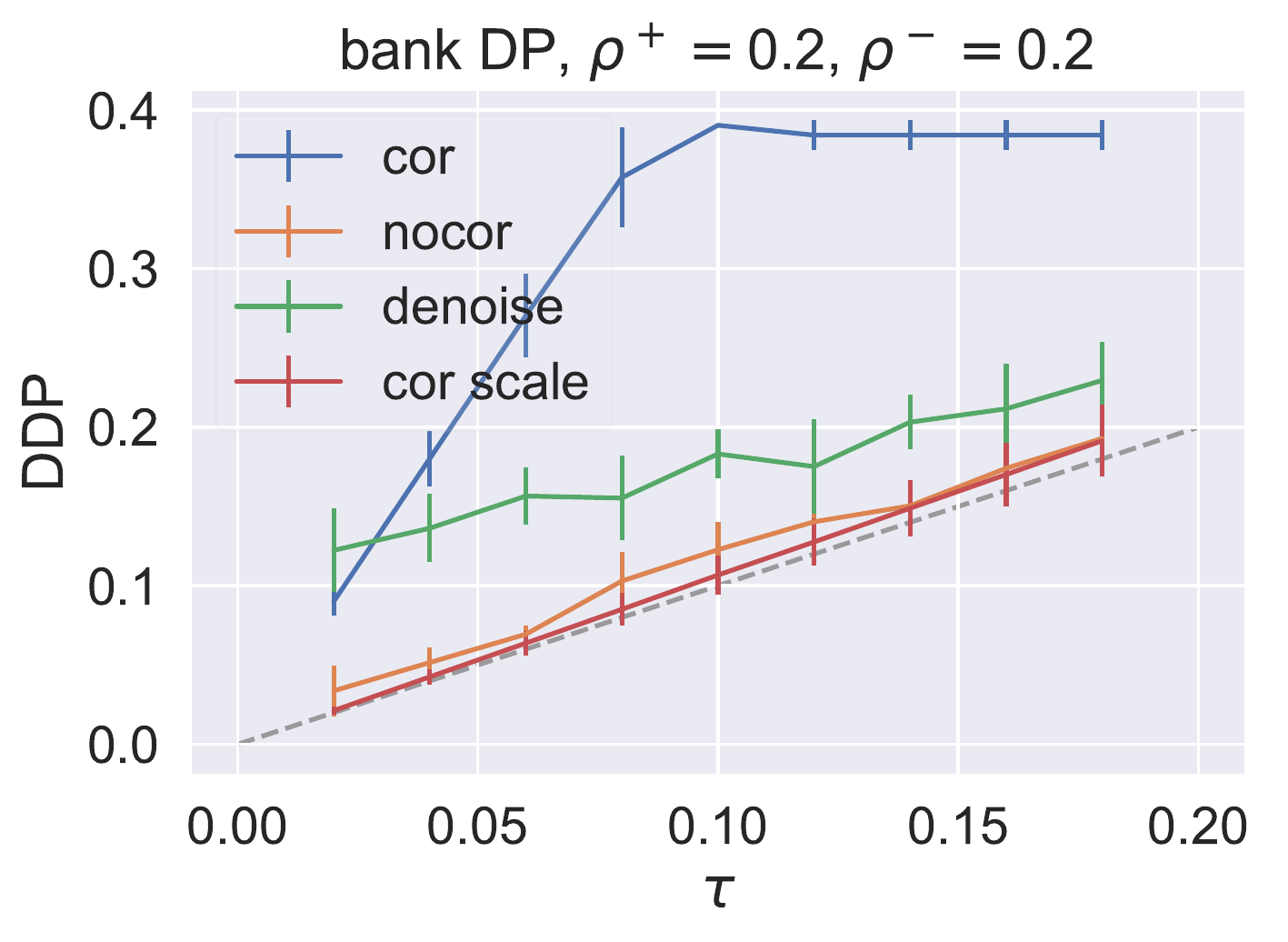}
    \includegraphics[width=0.24\textwidth]{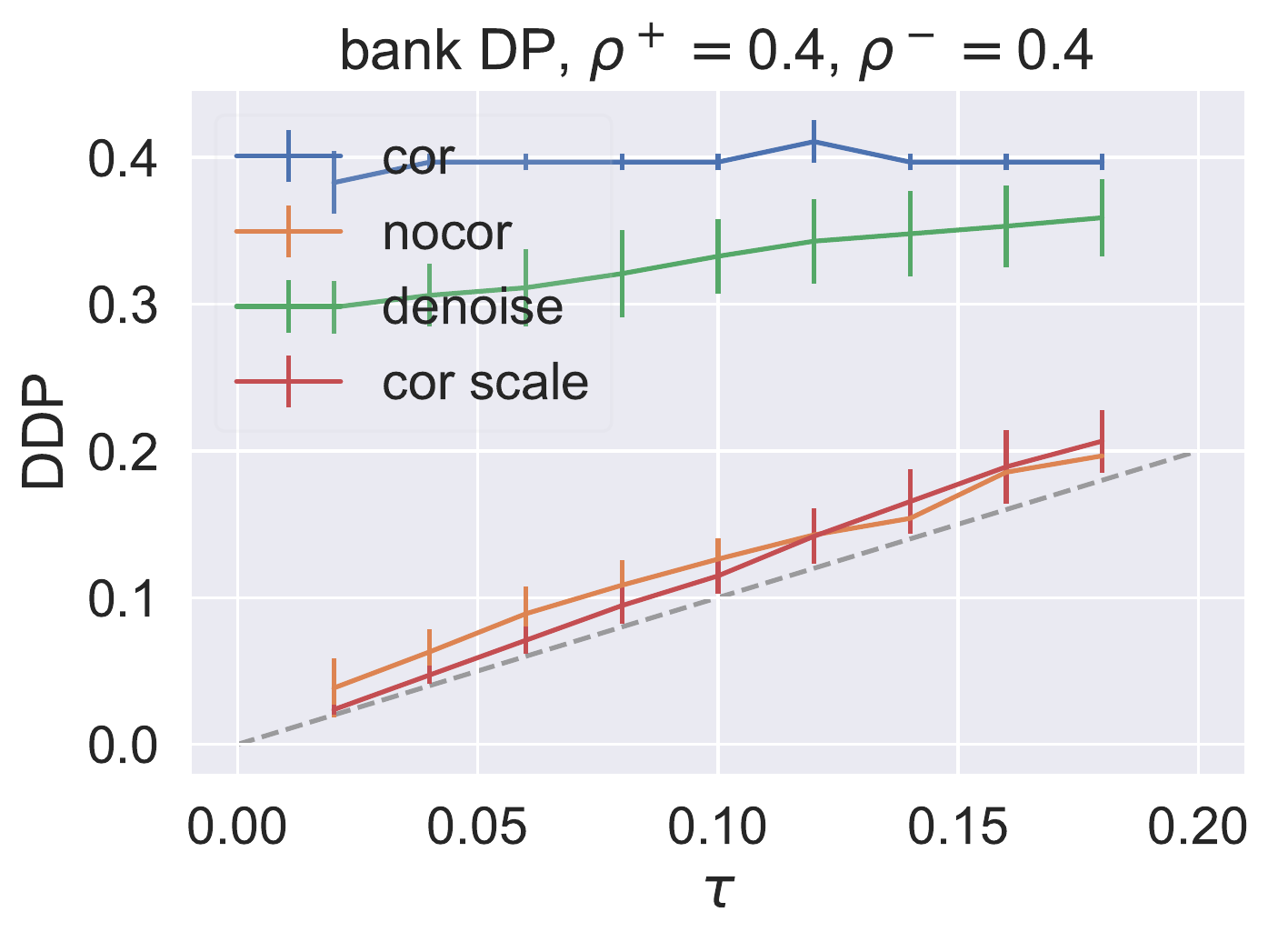}
    \includegraphics[width=0.24\textwidth]{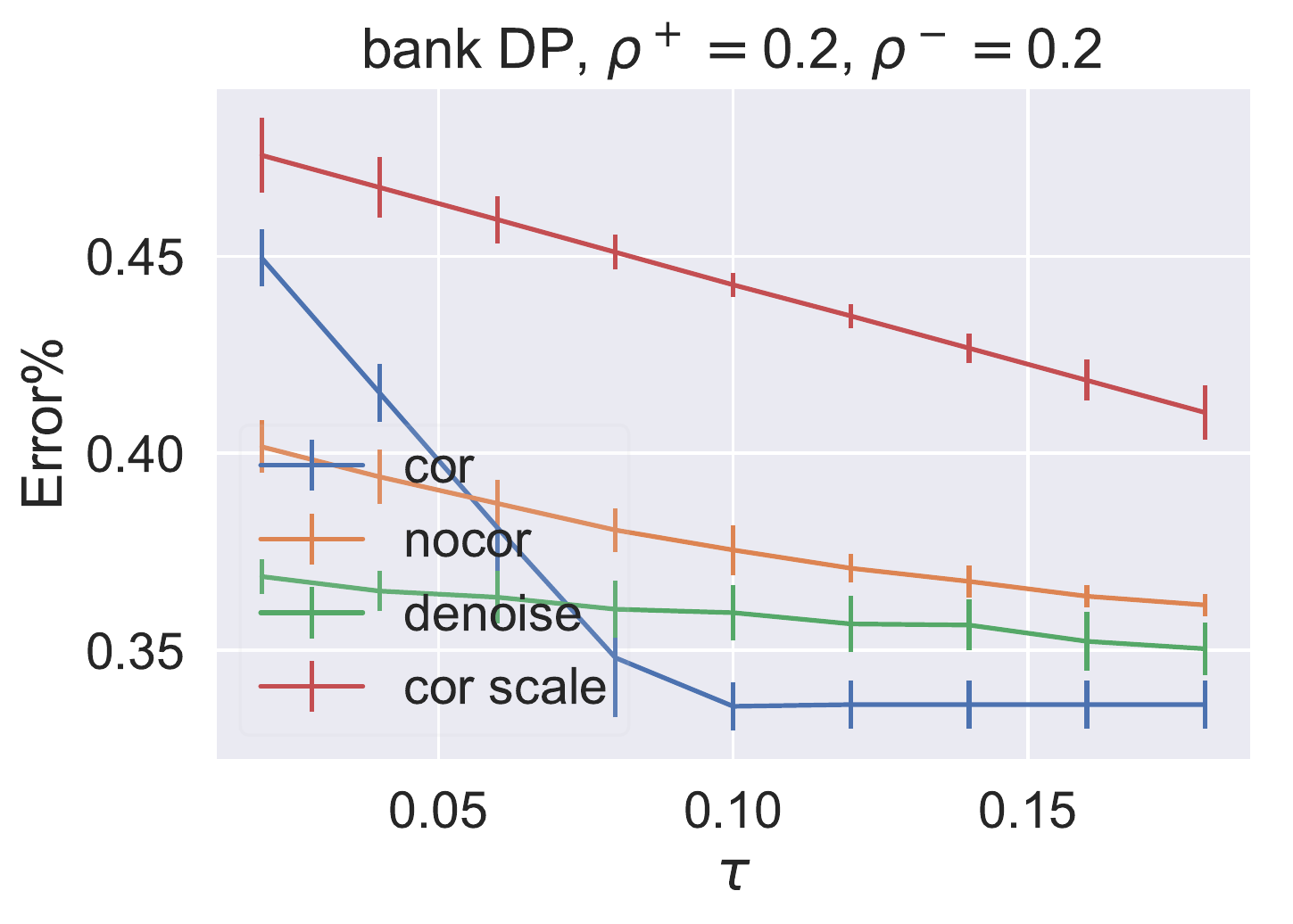}
    \includegraphics[width=0.24\textwidth]{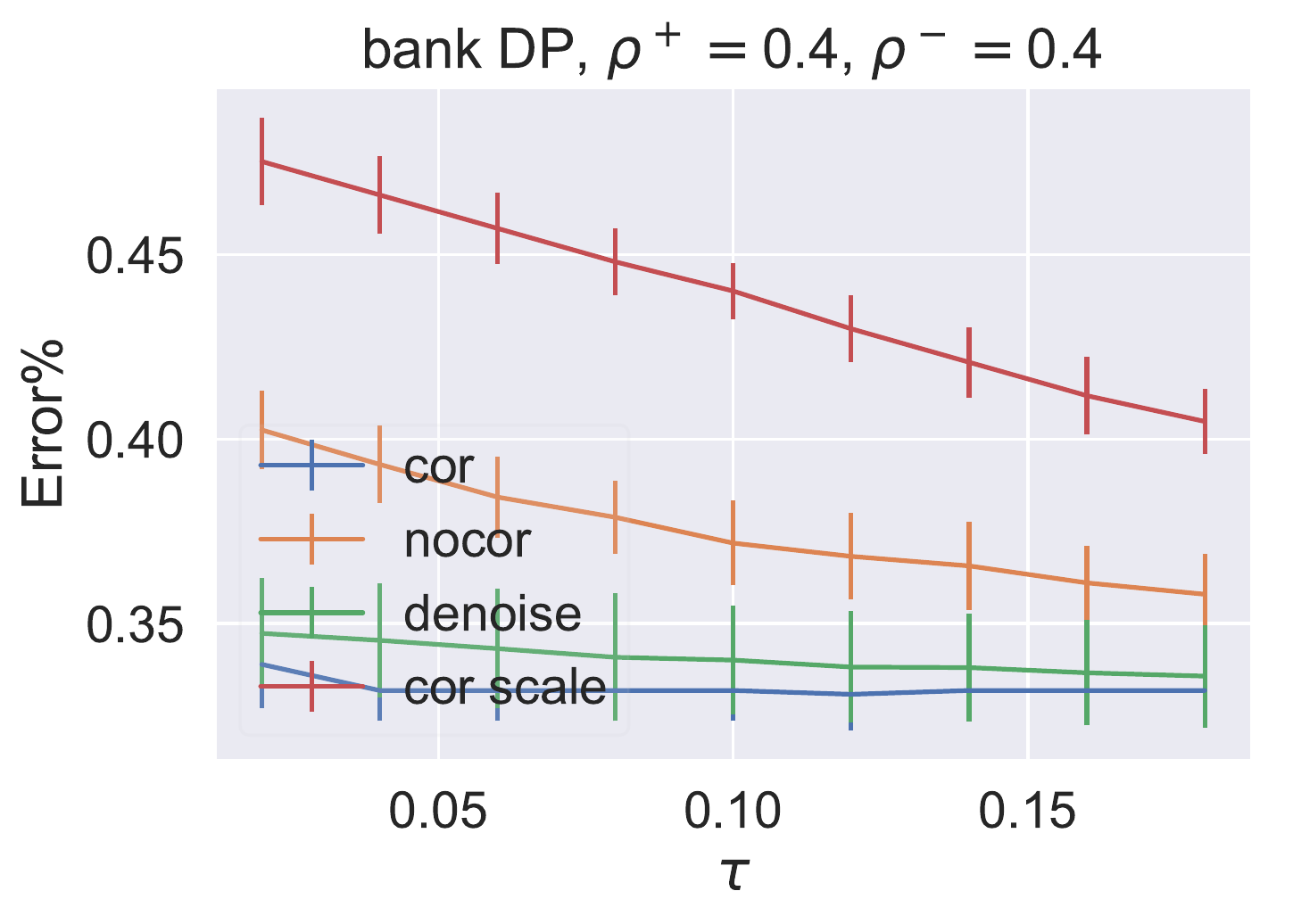}
    
    \includegraphics[width=0.24\textwidth]{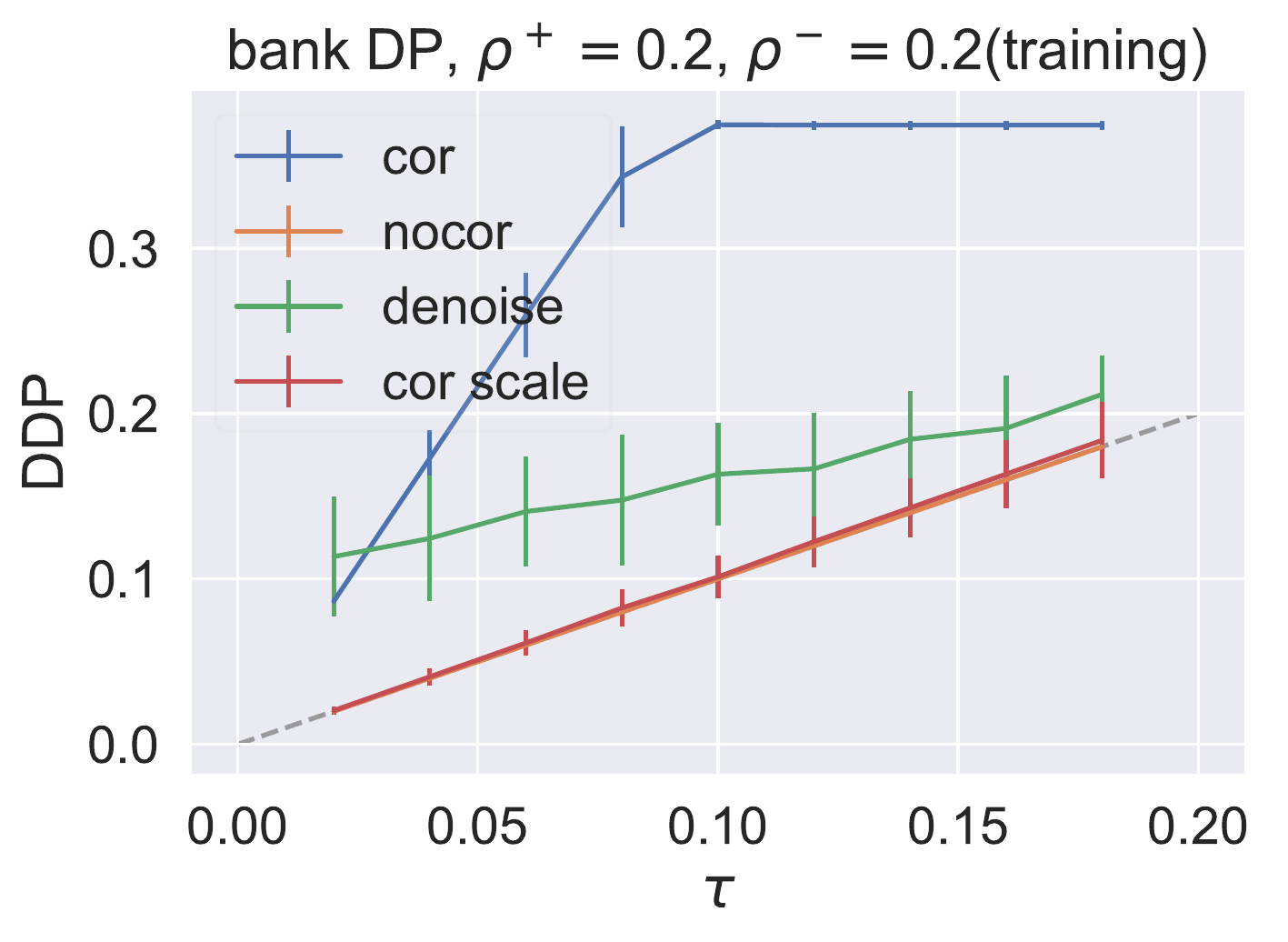}
    \includegraphics[width=0.24\textwidth]{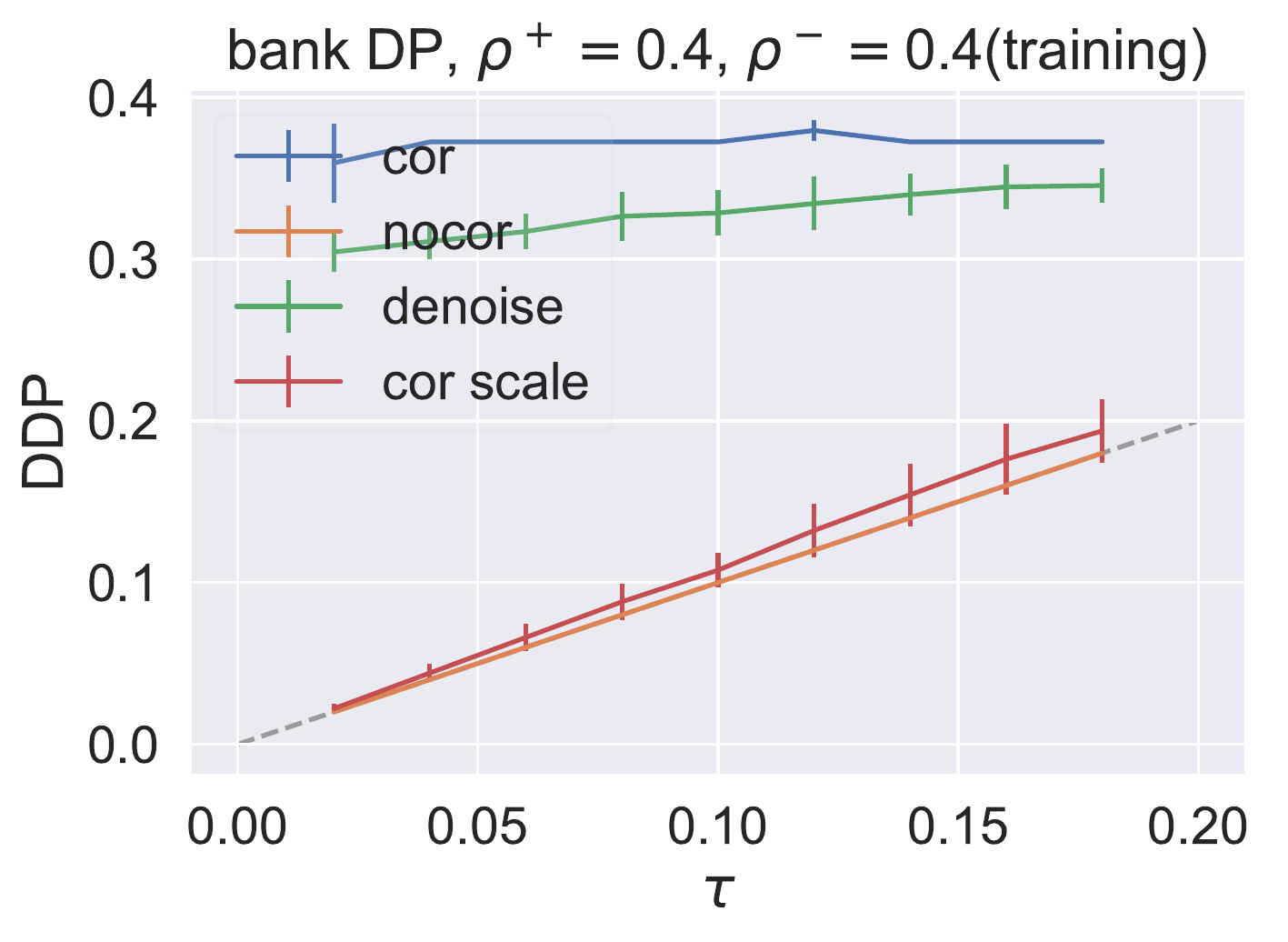}
    \includegraphics[width=0.24\textwidth]{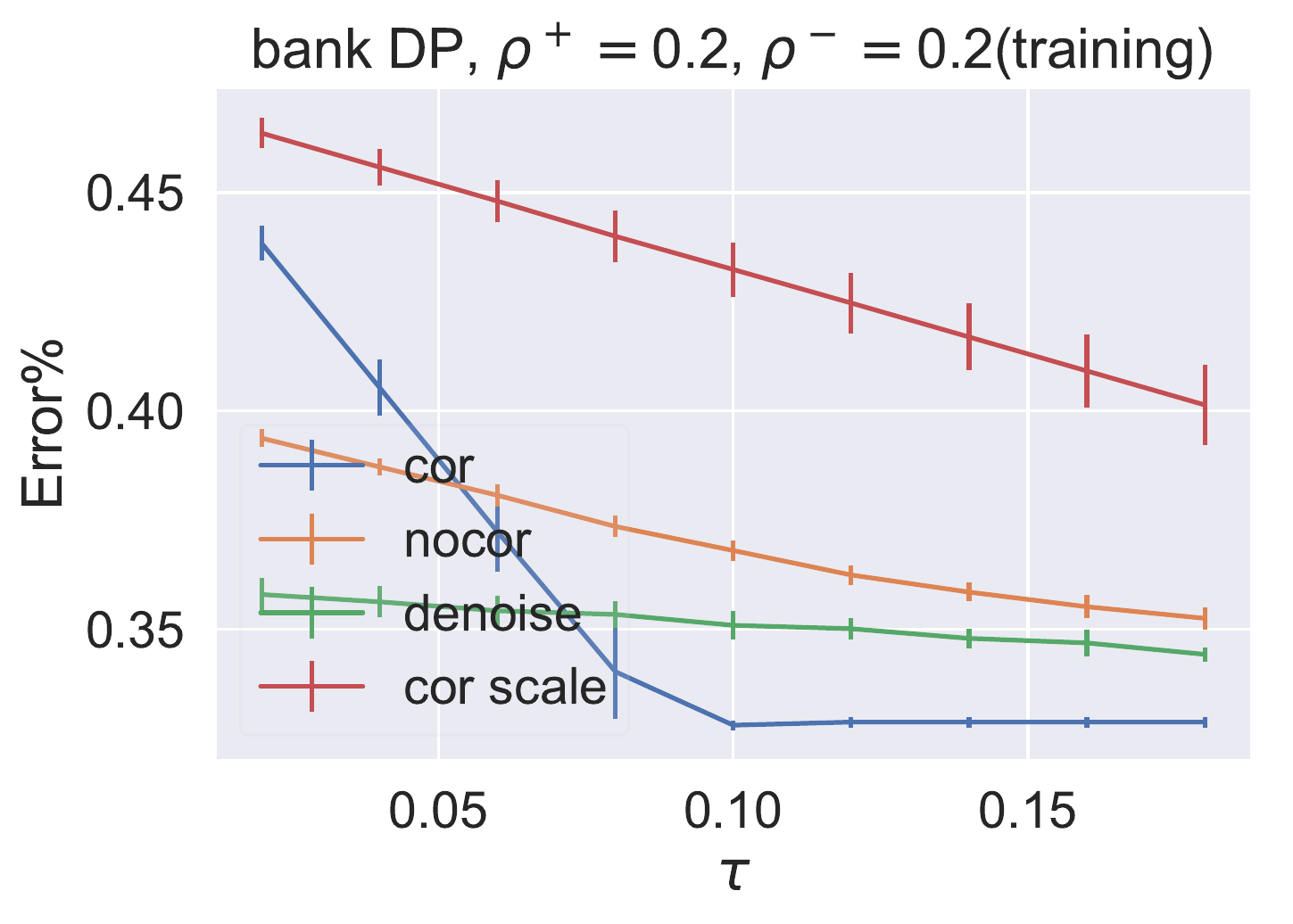}
    \includegraphics[width=0.24\textwidth]{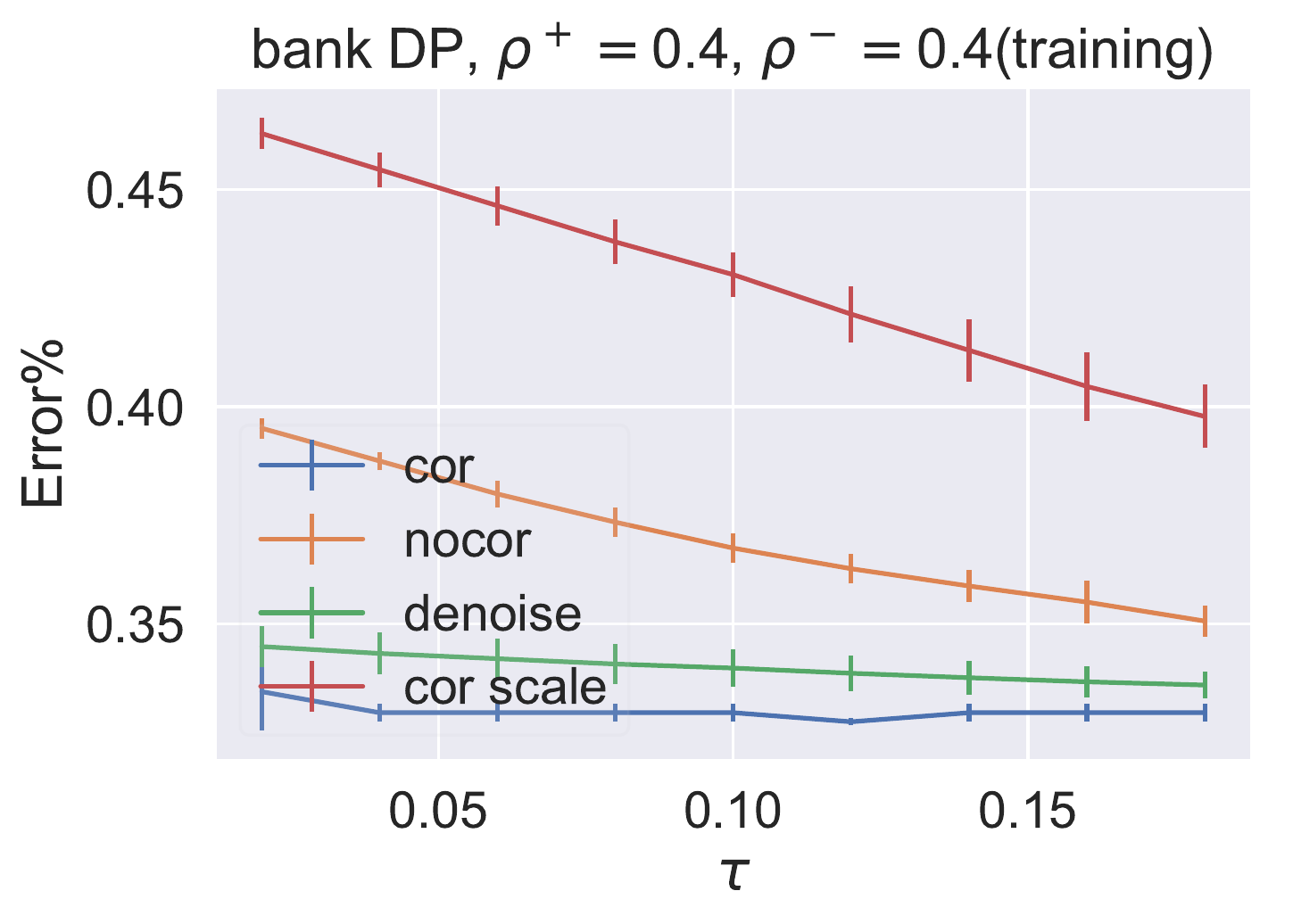}
    
    \caption{(DP)(testing and training) Relationship between input $\tau$ and fairness violation/error on the {\tt bank} dataset.}
    \label{fig:bank_training}
\end{figure*}

\begin{figure*}[h]
    \centering
    \includegraphics[width=0.24\textwidth]{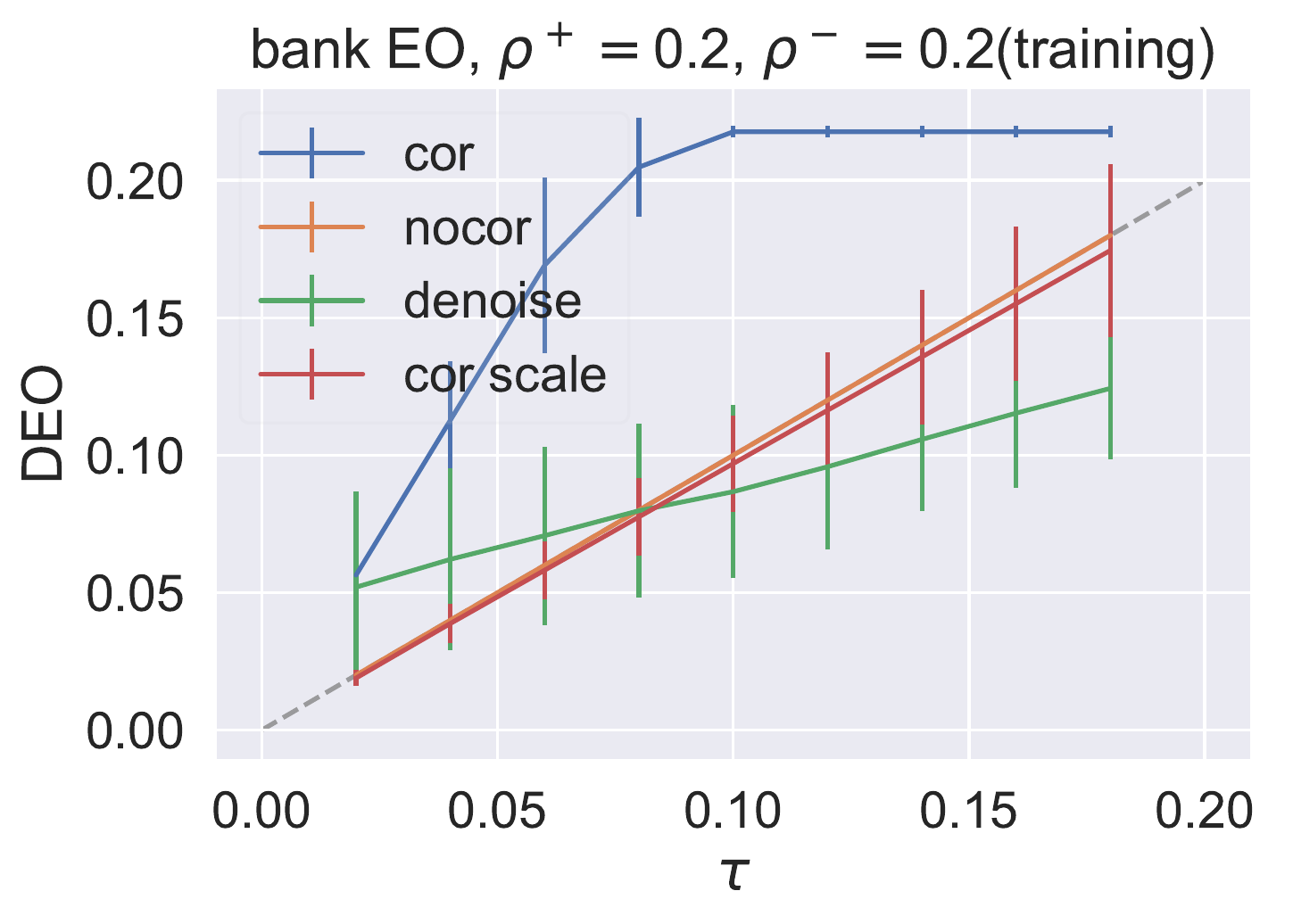}
    \includegraphics[width=0.24\textwidth]{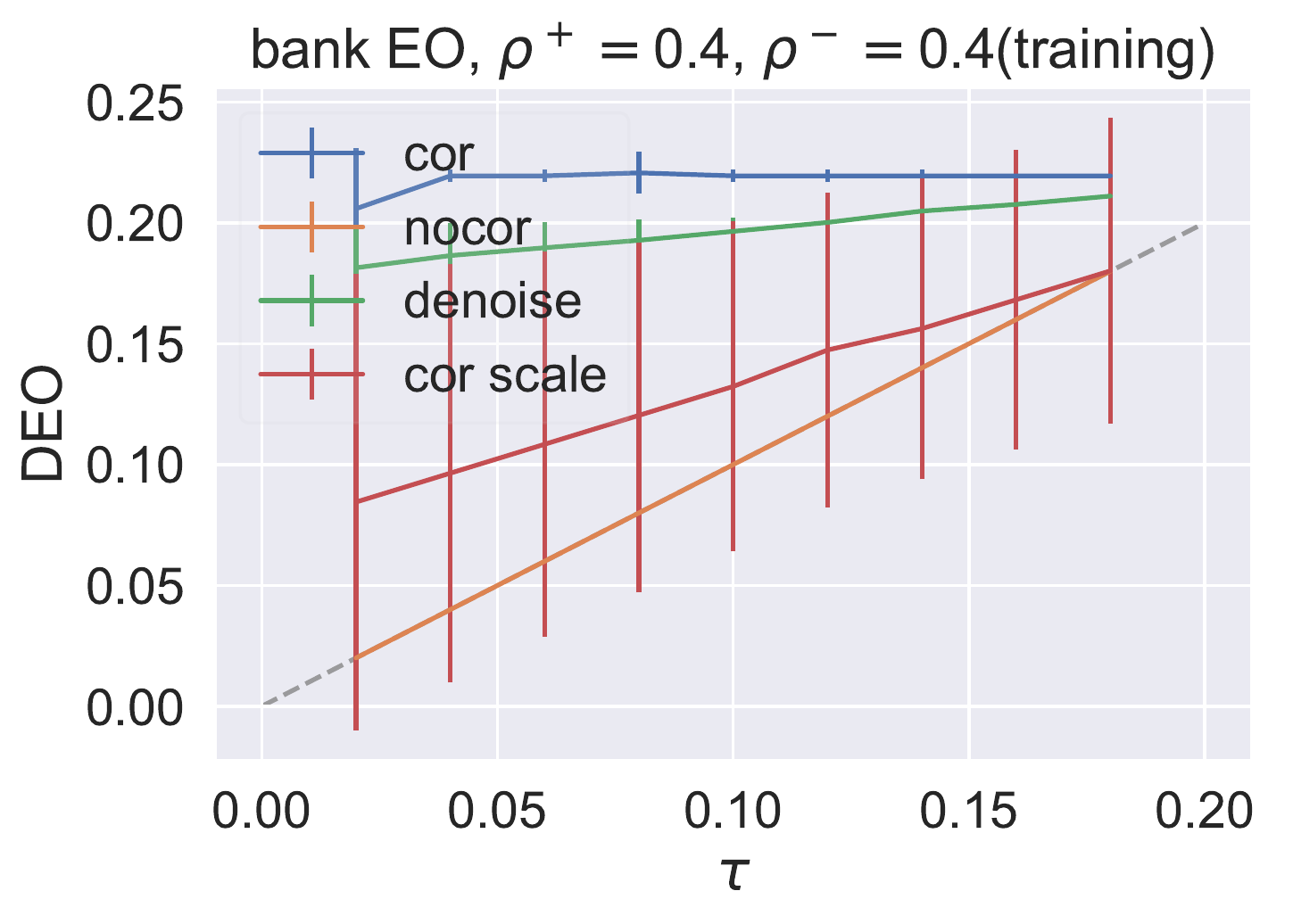}
    \includegraphics[width=0.24\textwidth]{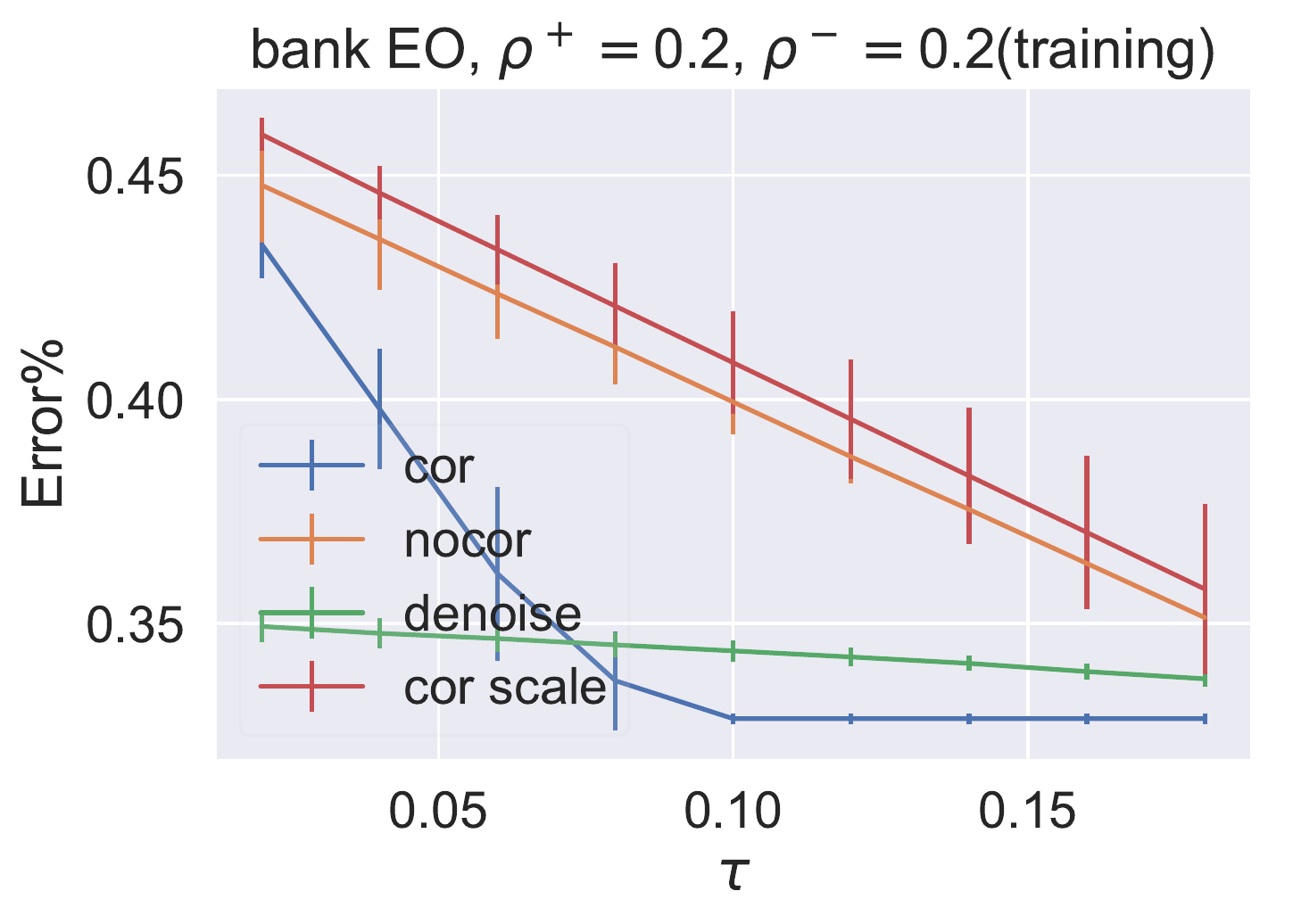}
    \includegraphics[width=0.24\textwidth]{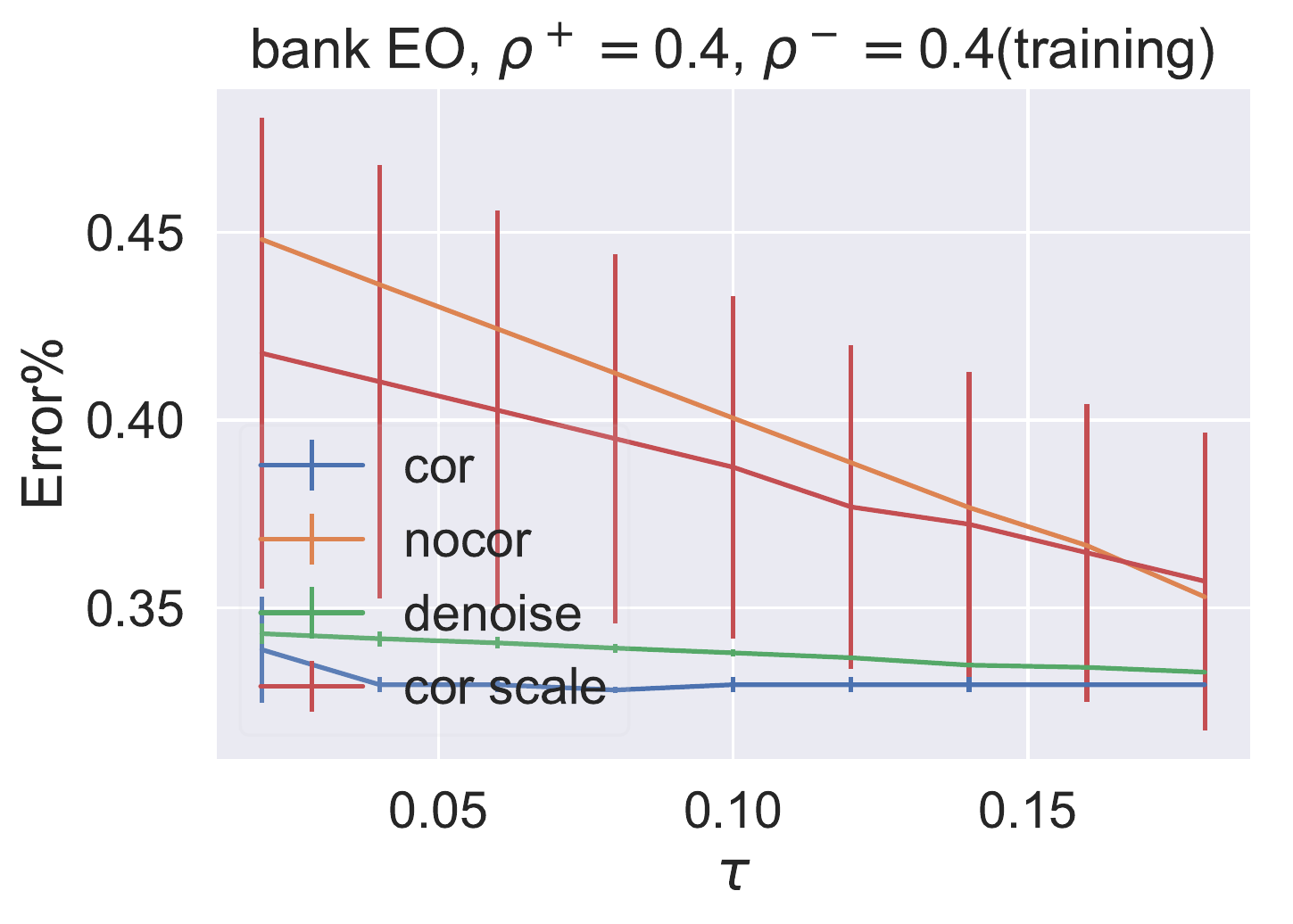}
    
    \includegraphics[width=0.24\textwidth]{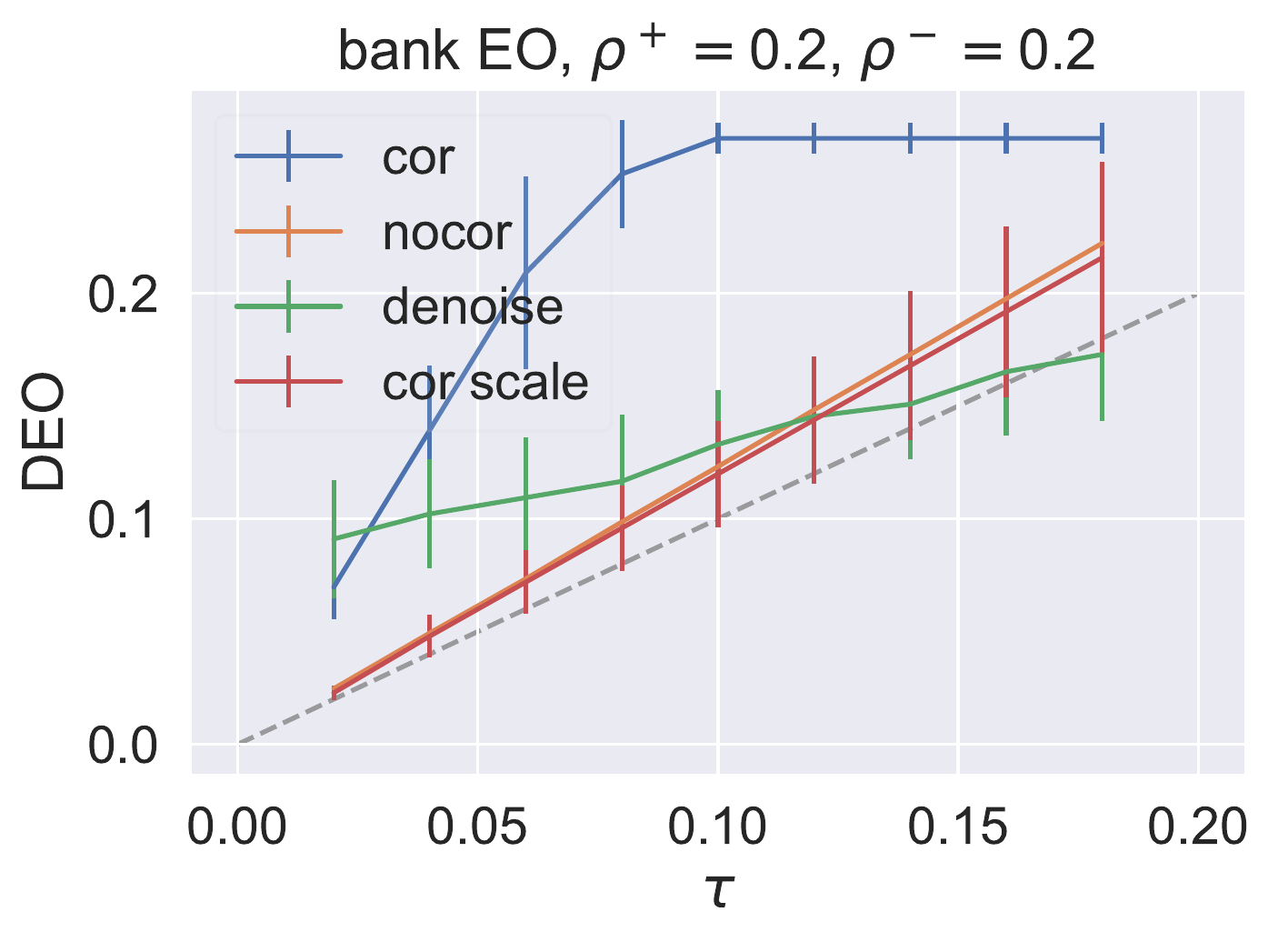}
    \includegraphics[width=0.24\textwidth]{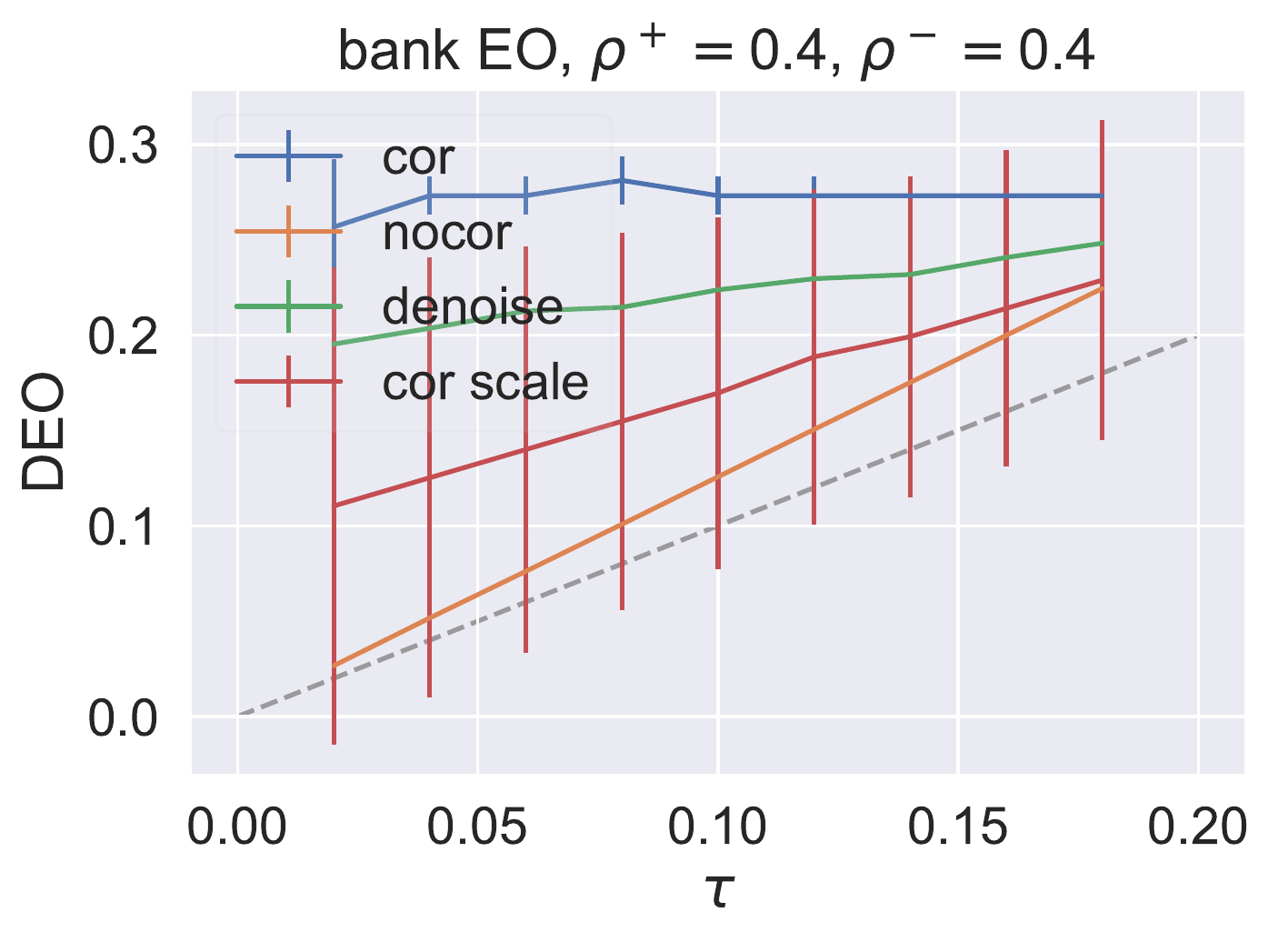}
    \includegraphics[width=0.24\textwidth]{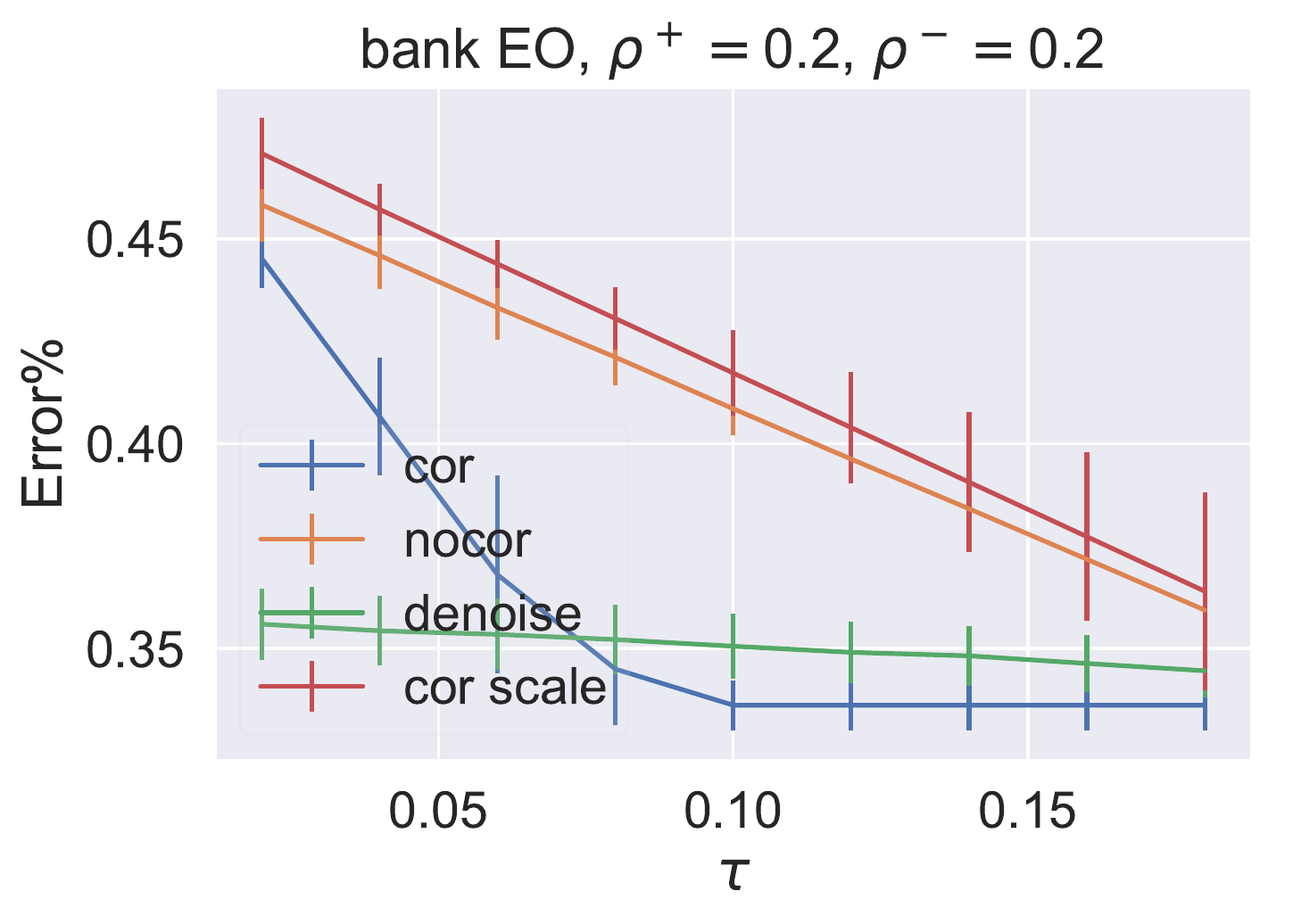}
    \includegraphics[width=0.24\textwidth]{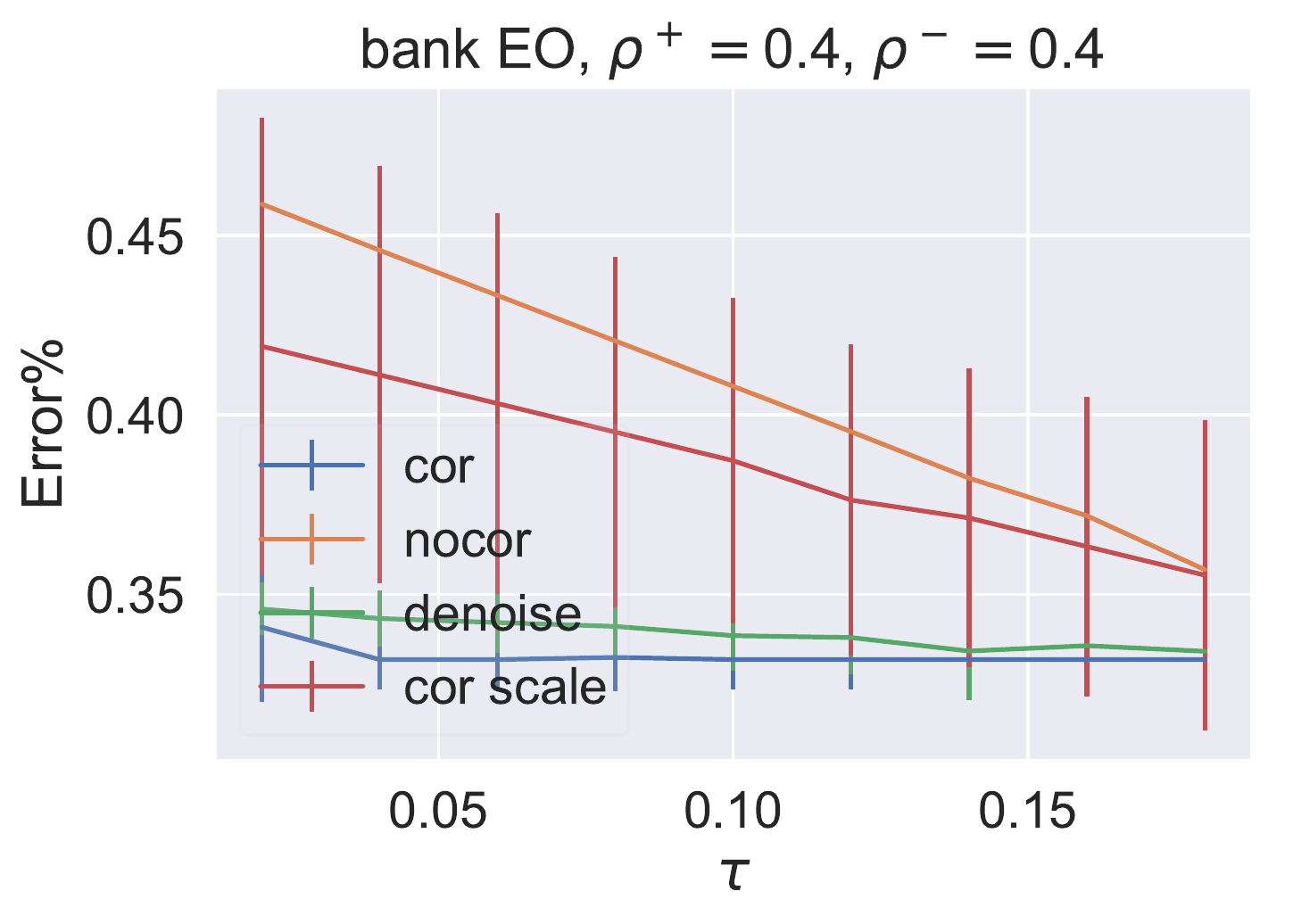}
    
    \caption{(EO)(testing and training) Relationship between input $\tau$ and fairness violation/error on the {\tt bank} dataset.}
    \label{fig:bank_testing}
\end{figure*}

\clearpage

\section{More results for the PU case study}
\label{appendix:pu}
In this section we give some additional results for the PU case study. 

{Figure~\ref{fig:law_full} shows additional results on {\tt law school} for different noise levels $\rho^+ = \rho^- \in \{0.2, 0.4\}$.} {Figure~\ref{fig:law_est_err_diff_tau} shows additional results under noise rate estimation on {\tt law school} for different upper bound of fairness violation: $\tau \in \{0.1, 0.3\}$.}

\begin{figure*}[!h]
    \centering
    \includegraphics[width=0.24\textwidth]{img_pu_law/{disp_test_law,0.0,0.2,1.0,DP,Agarwal,3,False}.pdf}
    \includegraphics[width=0.24\textwidth]{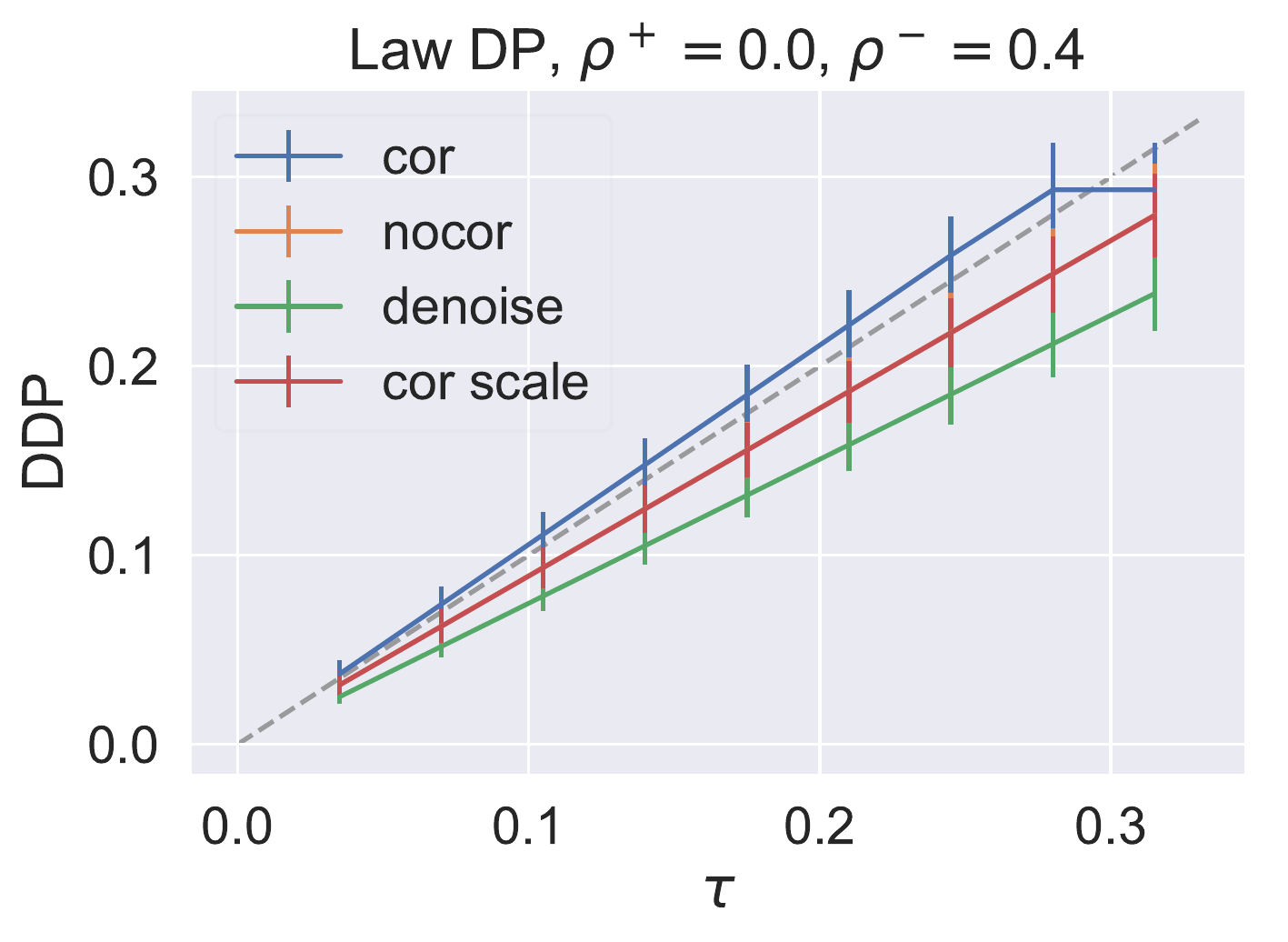}    
    \includegraphics[width=0.24\textwidth]{img_pu_law/{error_test_law,0.0,0.2,1.0,DP,Agarwal,3,False}.pdf}
    \includegraphics[width=0.24\textwidth]{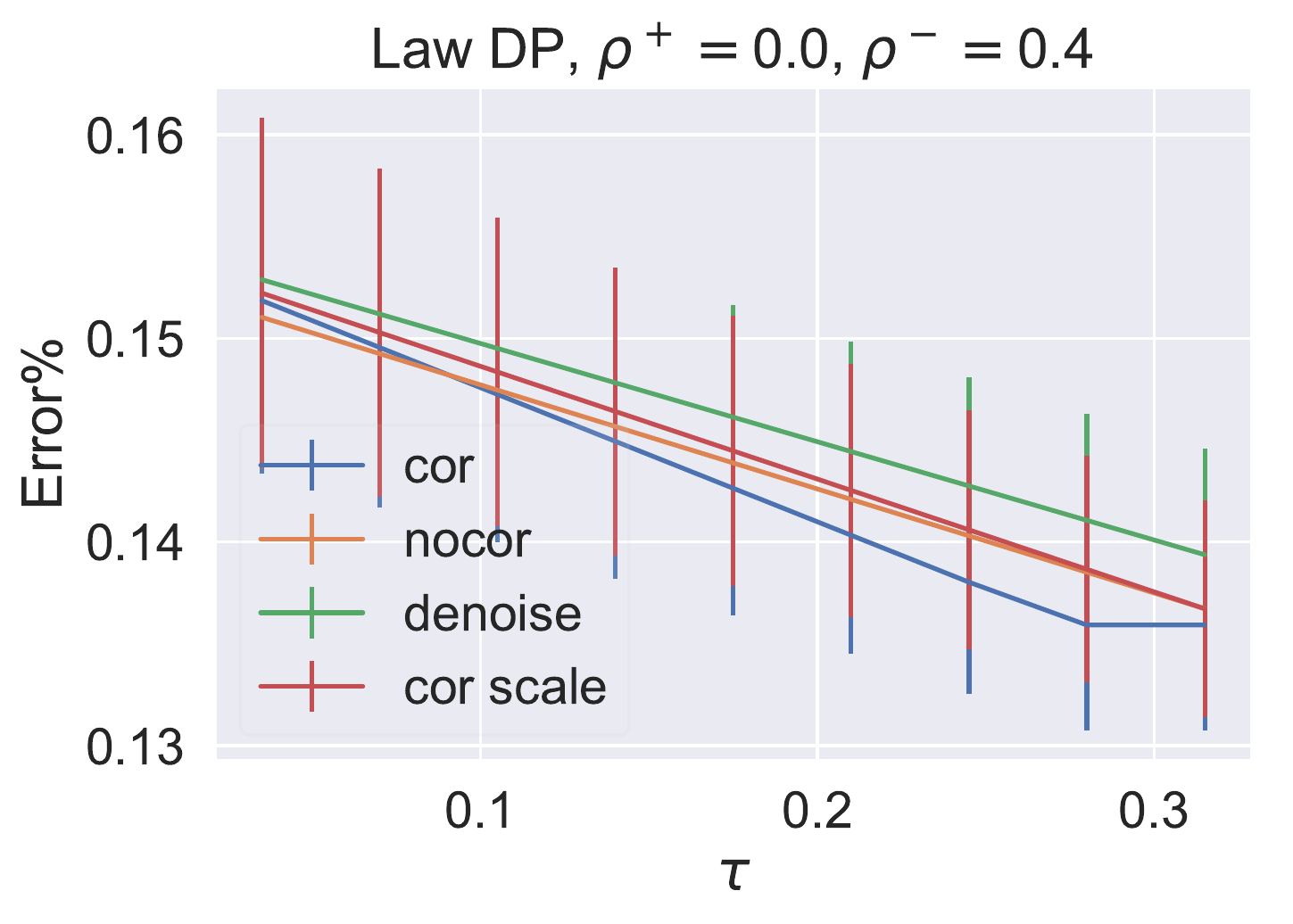}

    \caption{Relationship between input $\tau$ and fairness violation/error on the {\tt law school} dataset using DP constraint (testing curves). The 
    {gray dashed line} represents the ideal fairness violation. Note that in some of the graphs, the red line and the orange line perfectly overlap with each other.}
    \label{fig:law_full}
\end{figure*}

\begin{figure*}[!t]
    \centering
    \includegraphics[width=0.24\textwidth]{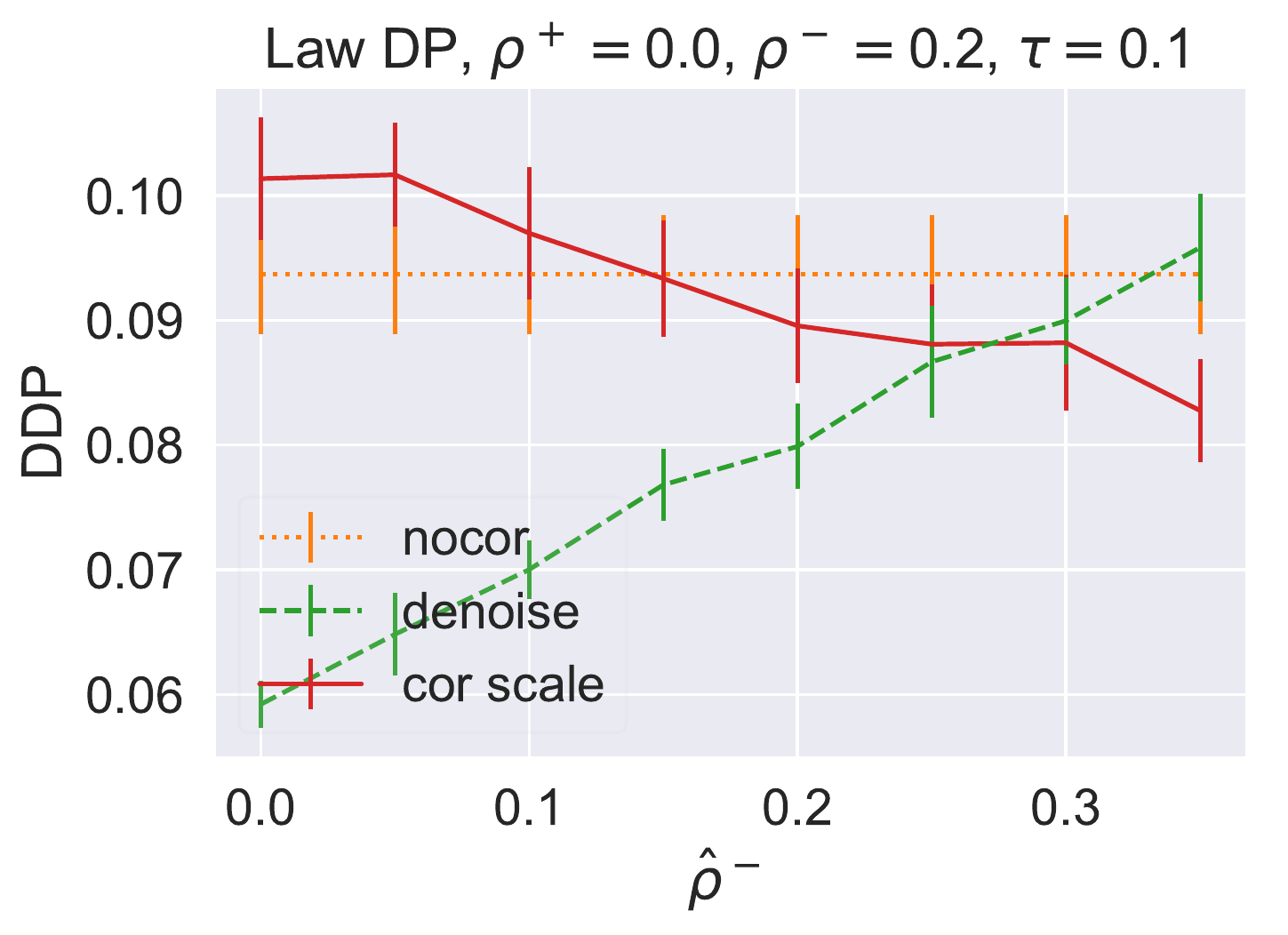}
    \includegraphics[width=0.24\textwidth]{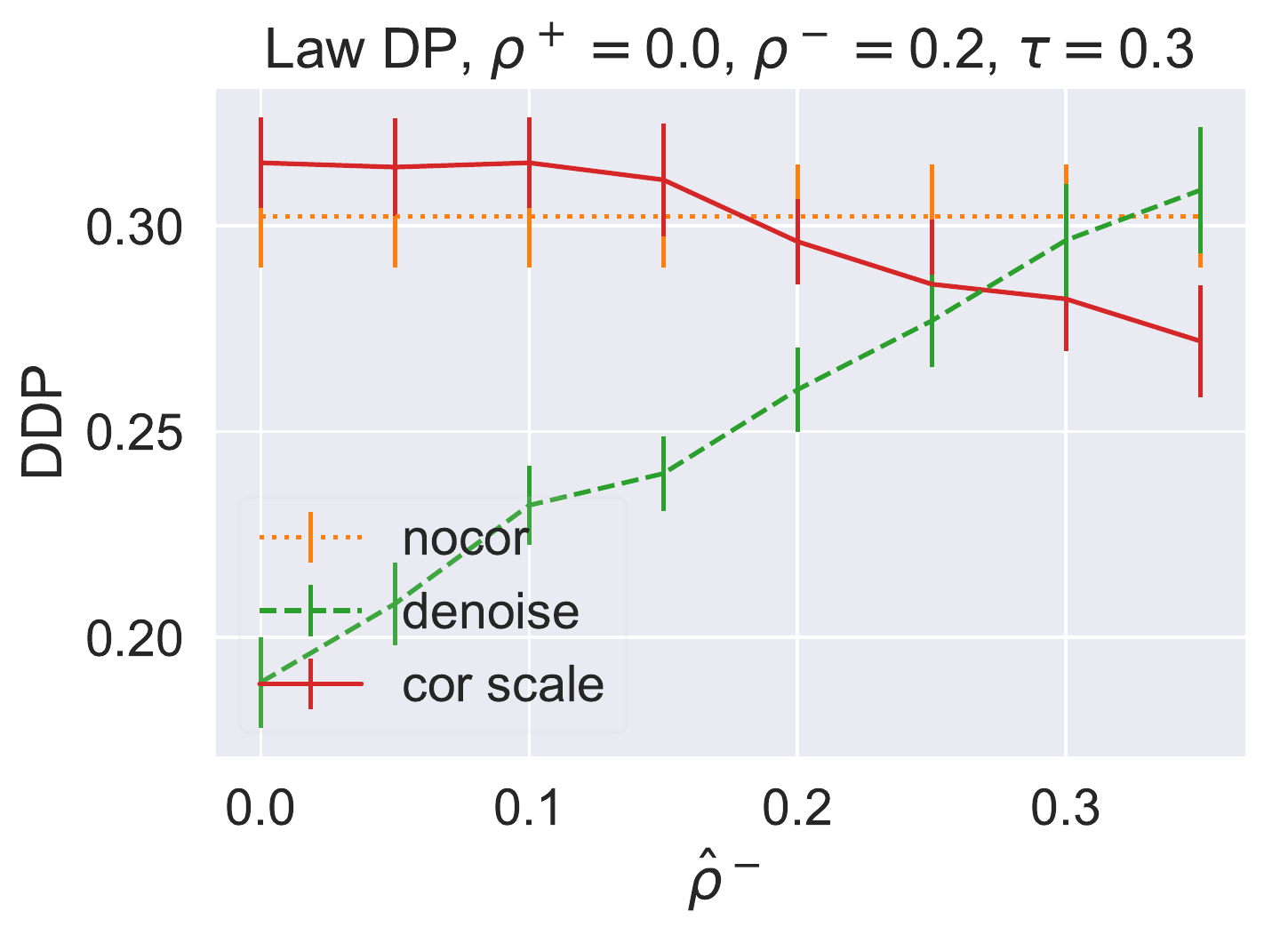}    
    \includegraphics[width=0.24\textwidth]{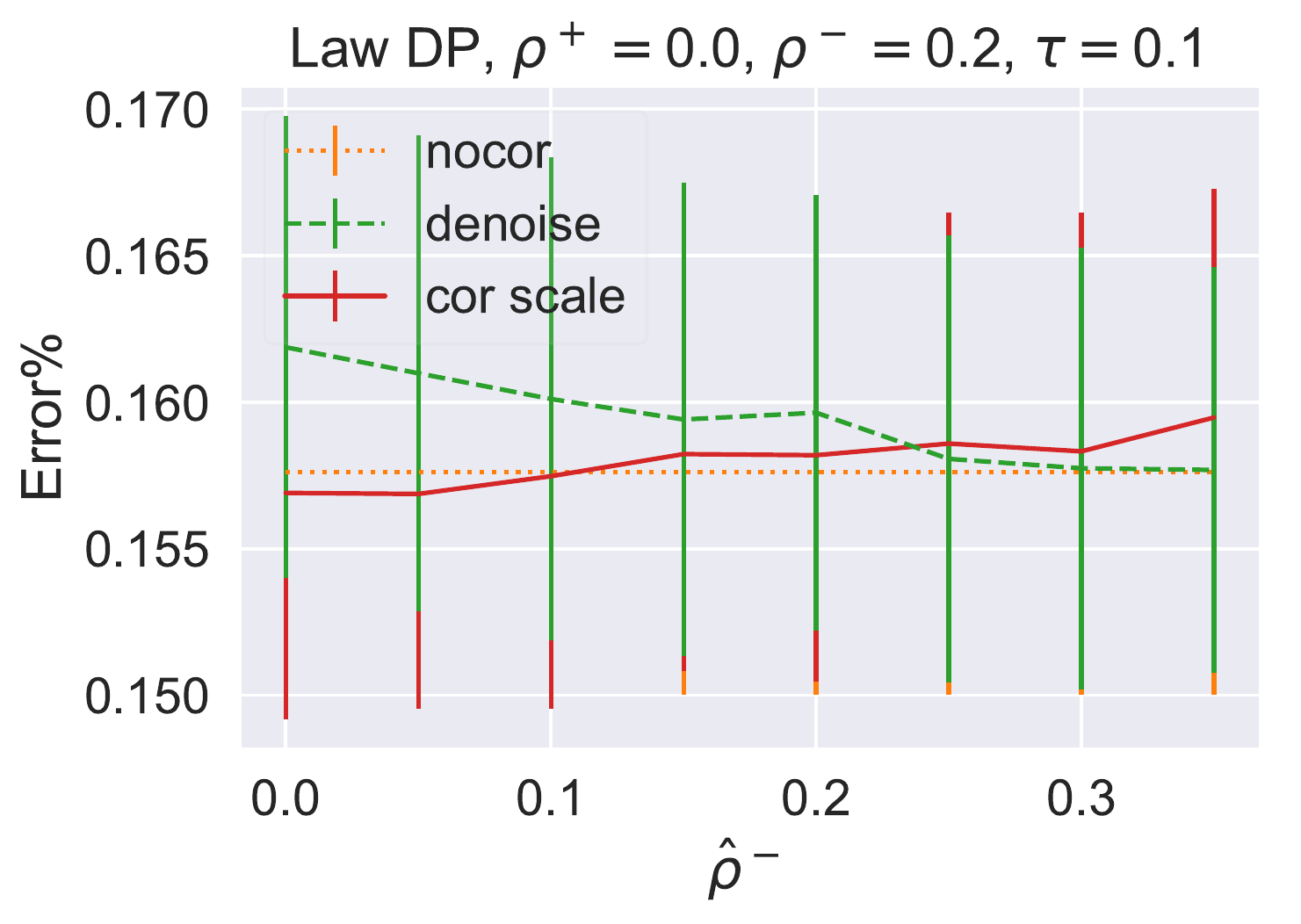}
    \includegraphics[width=0.24\textwidth]{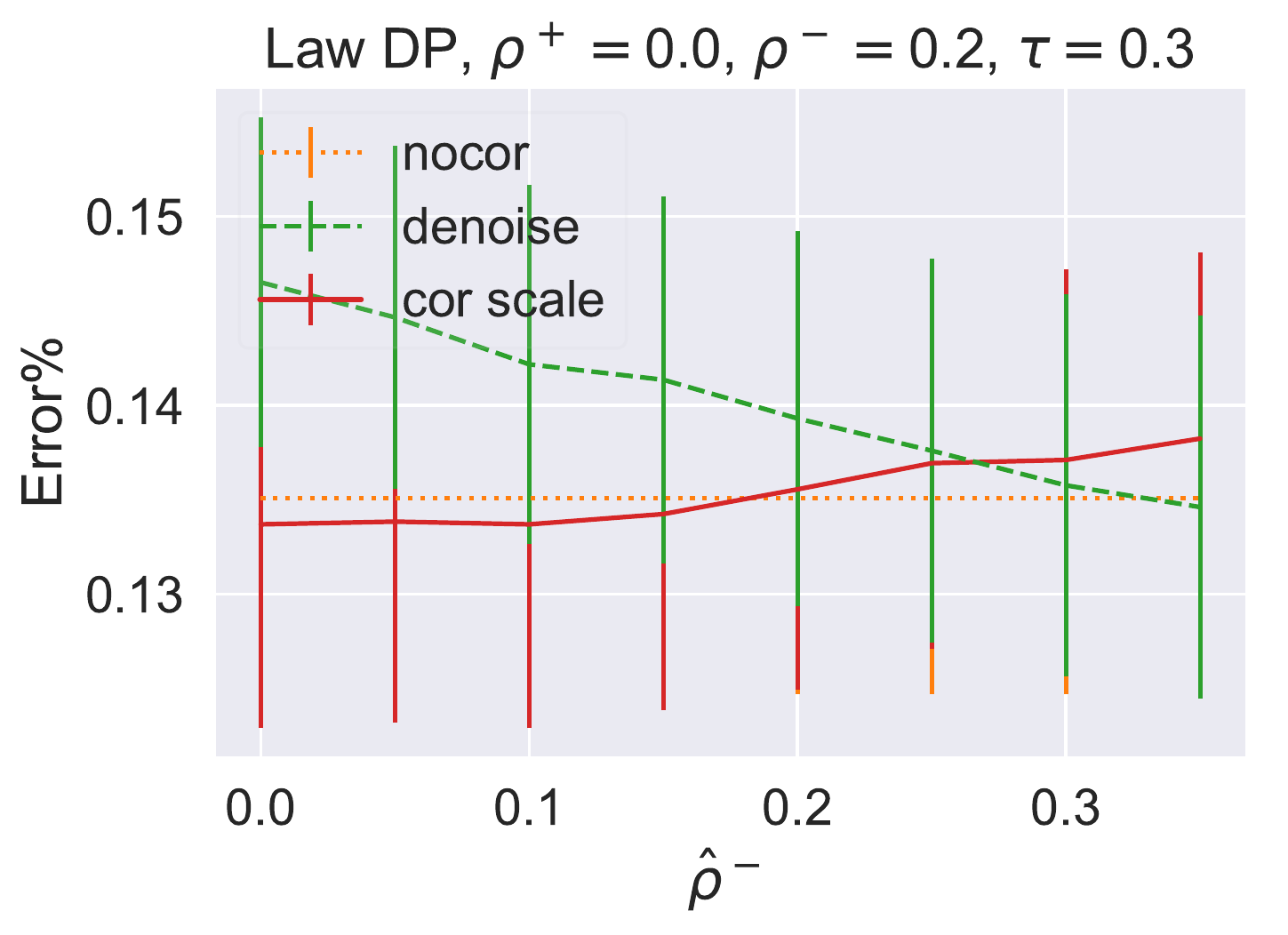}
    \caption{Relationship between the estimated noise level $\hat{\rho}^-$ and fairness violation/error on the {\tt law school} dataset using DP constraint (testing curves) at $\tau \in \{0.1, 0.3\}$, with $\hat{\rho}^+ = 0$.}
    \label{fig:law_est_err_diff_tau}
\end{figure*}

Figure \ref{fig:DP_german} and Figure \ref{fig:german_est_err} show the results under PU noise on the {\tt german} dataset, which is another dataset from the UCI repository ~\citep{UCI}. The task is to predict if one has good credit and the sensitive attribute is whether a person is foreign. The data comprises 1000 examples and 20 features. The trends are similar to those for the {\tt law school} dataset.

\begin{figure*}[h]
    \centering
    \includegraphics[width=0.24\textwidth]{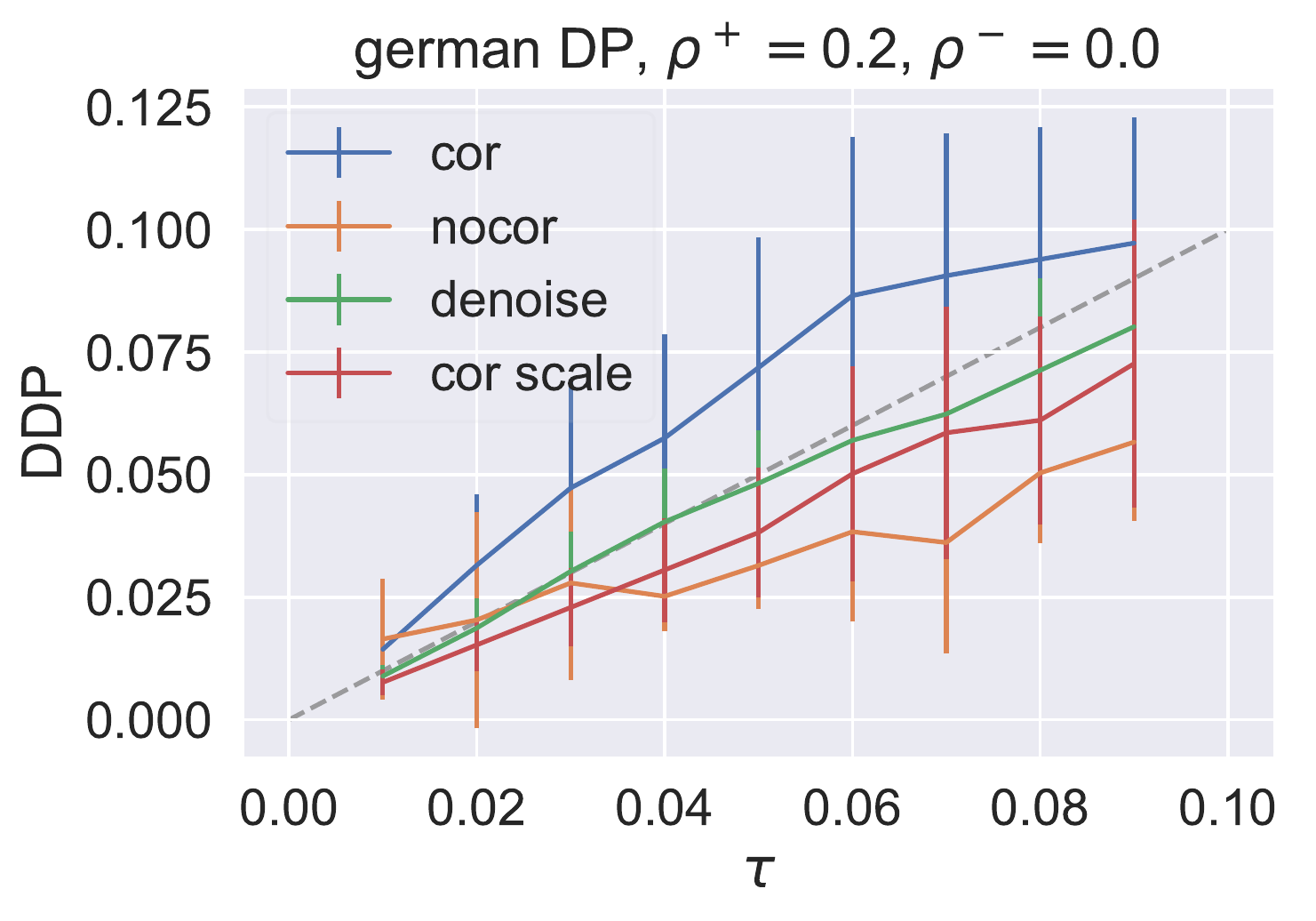}
    \includegraphics[width=0.24\textwidth]{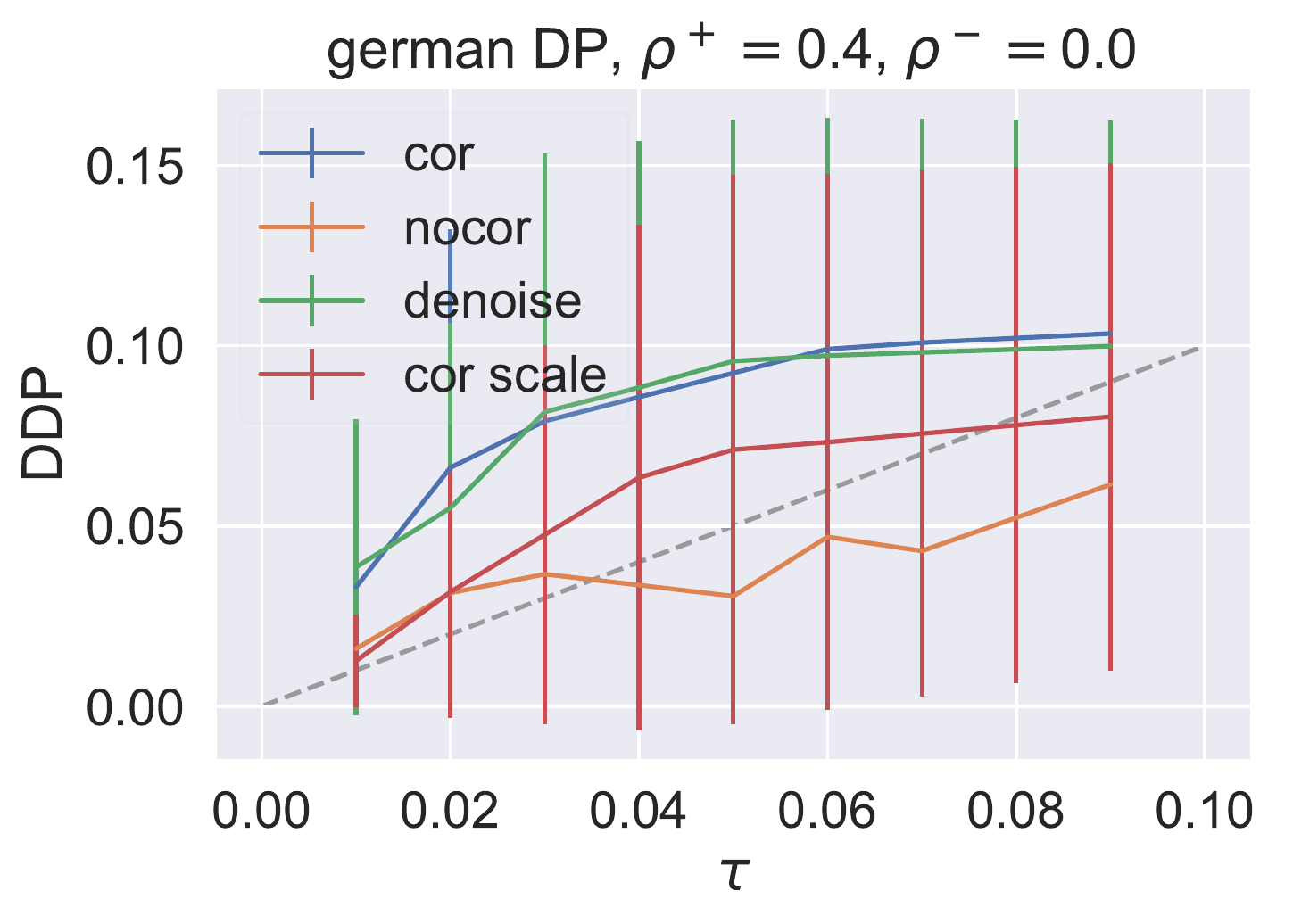}
    \includegraphics[width=0.24\textwidth]{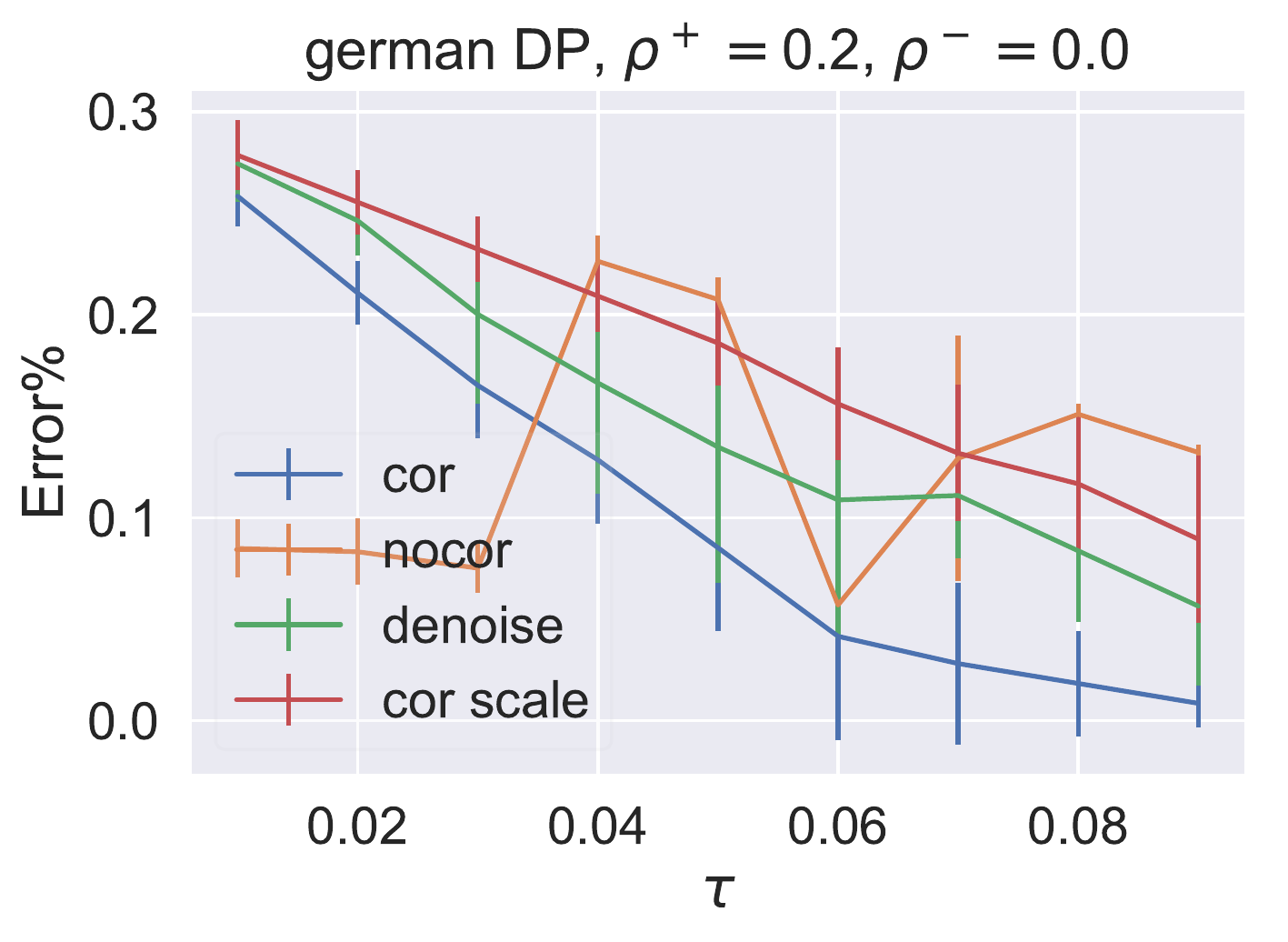}
    \includegraphics[width=0.24\textwidth]{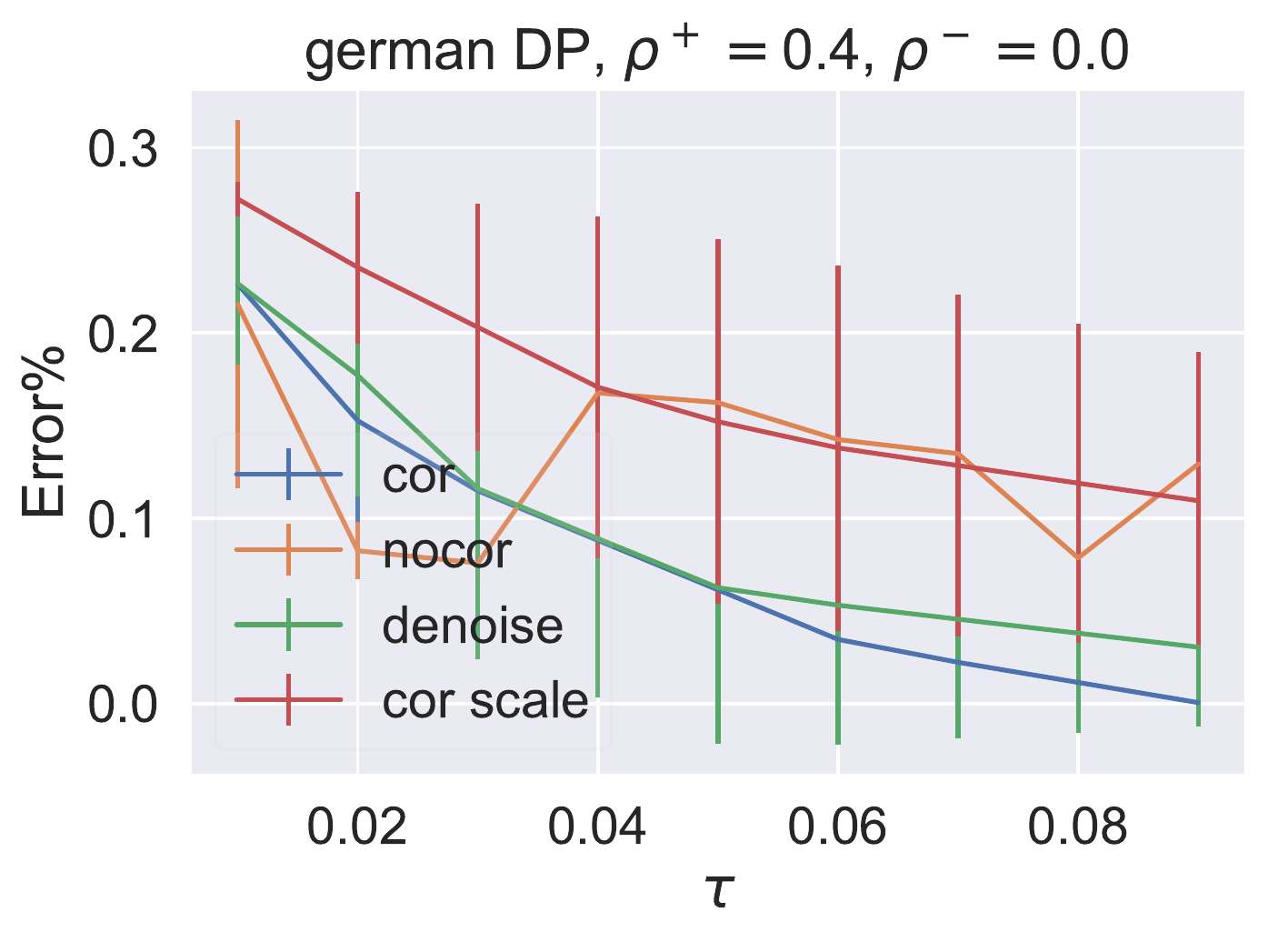}
    
    \includegraphics[width=0.24\textwidth]{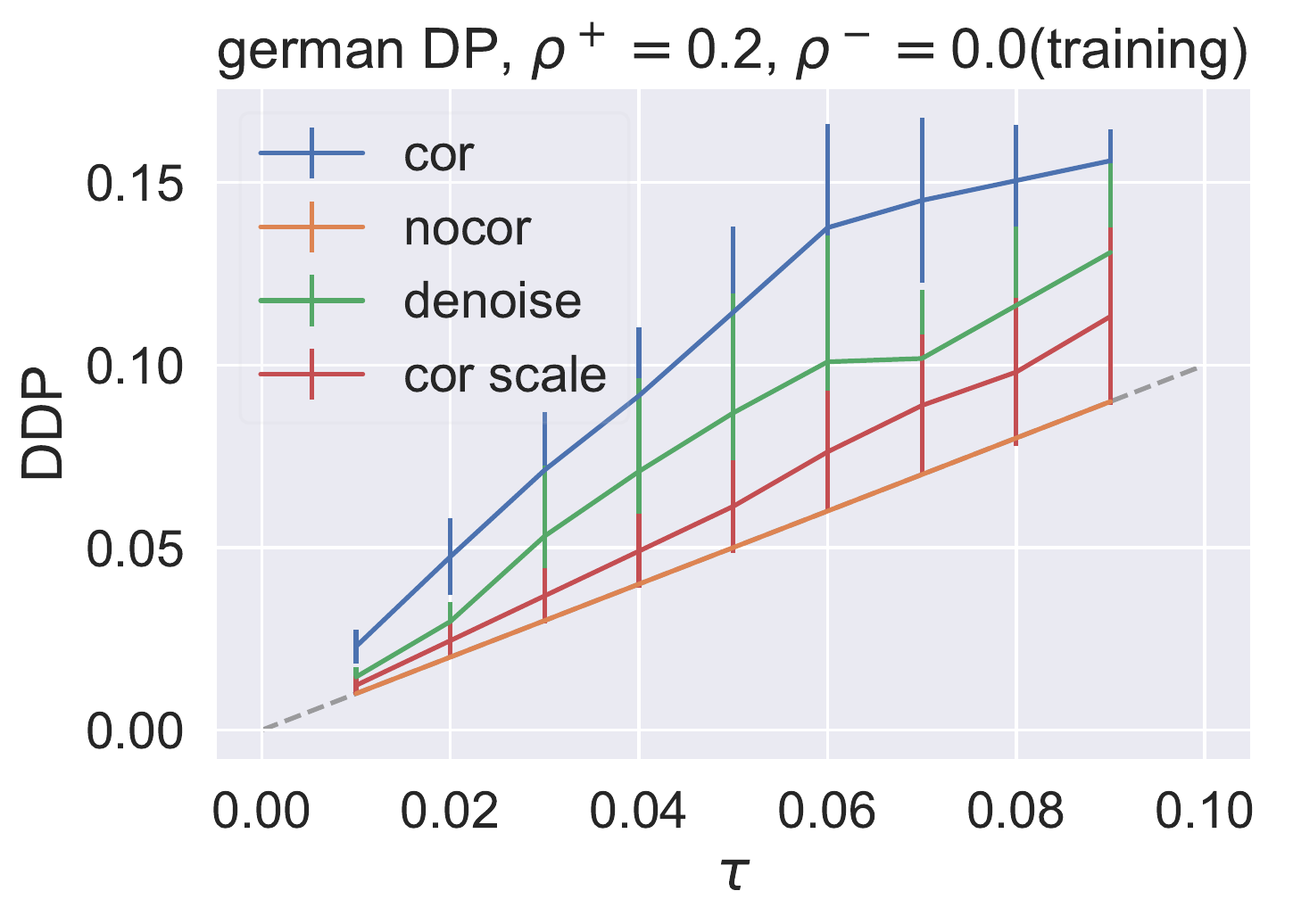}
    \includegraphics[width=0.24\textwidth]{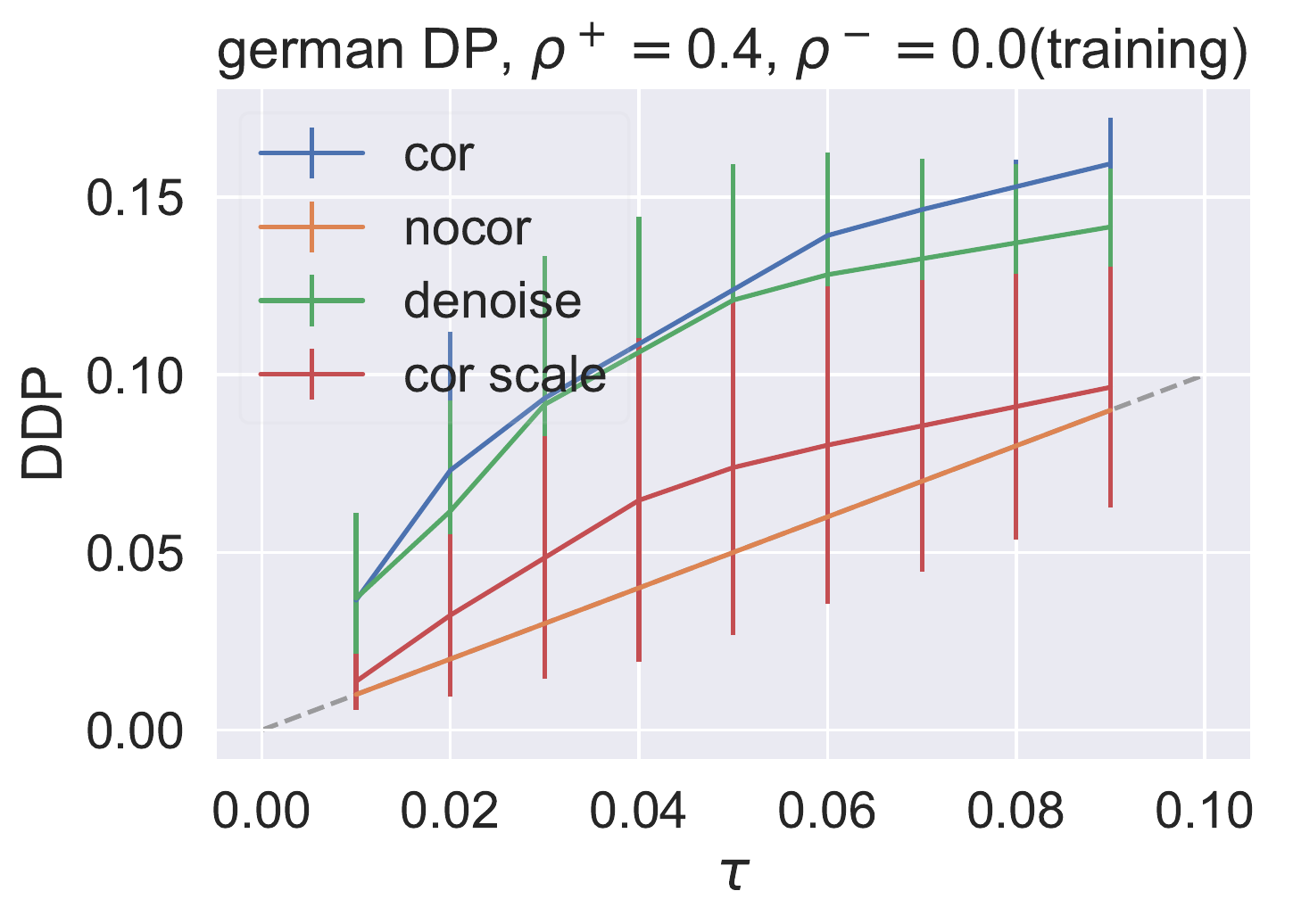}
    \includegraphics[width=0.24\textwidth]{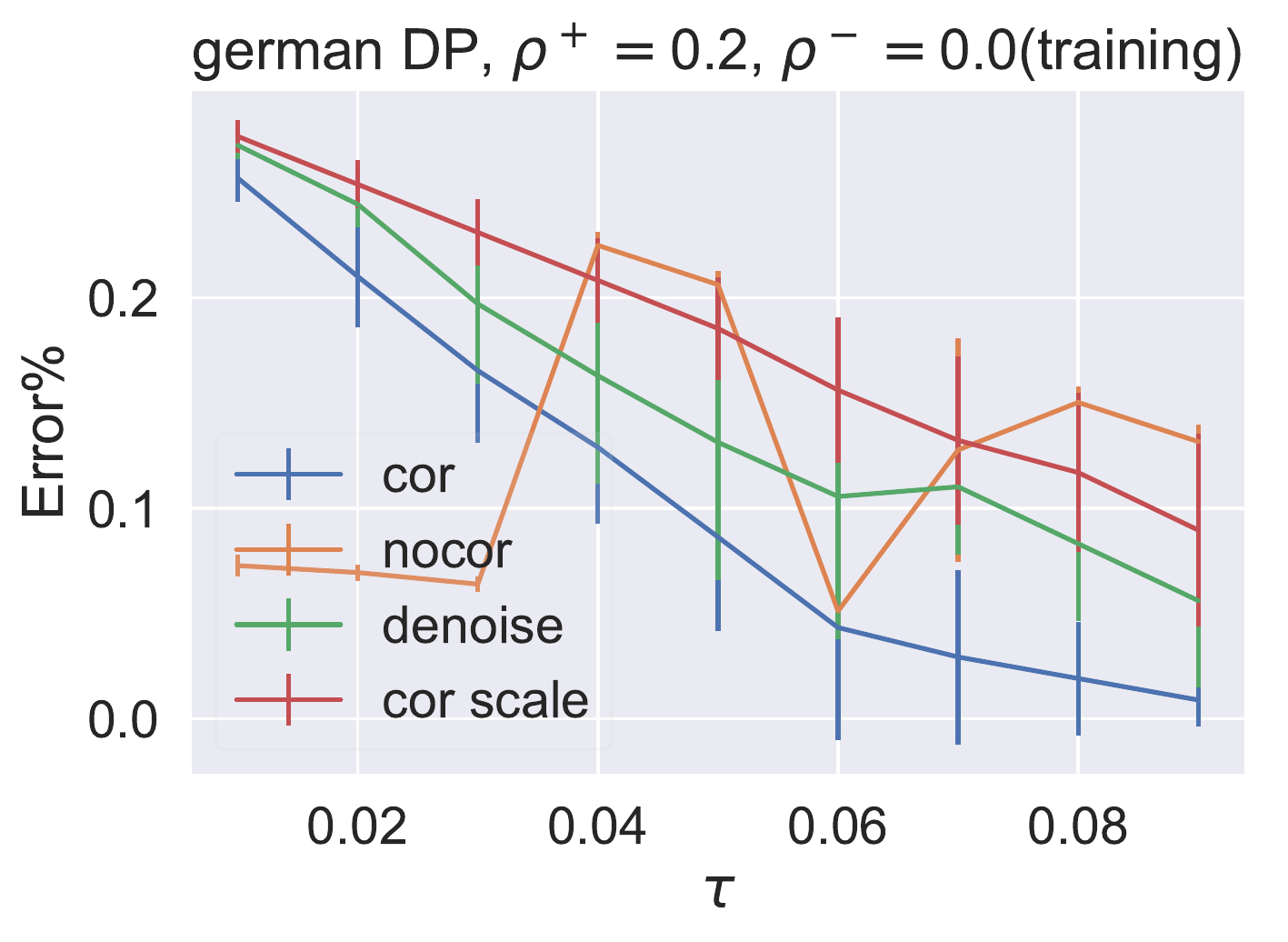}
    \includegraphics[width=0.24\textwidth]{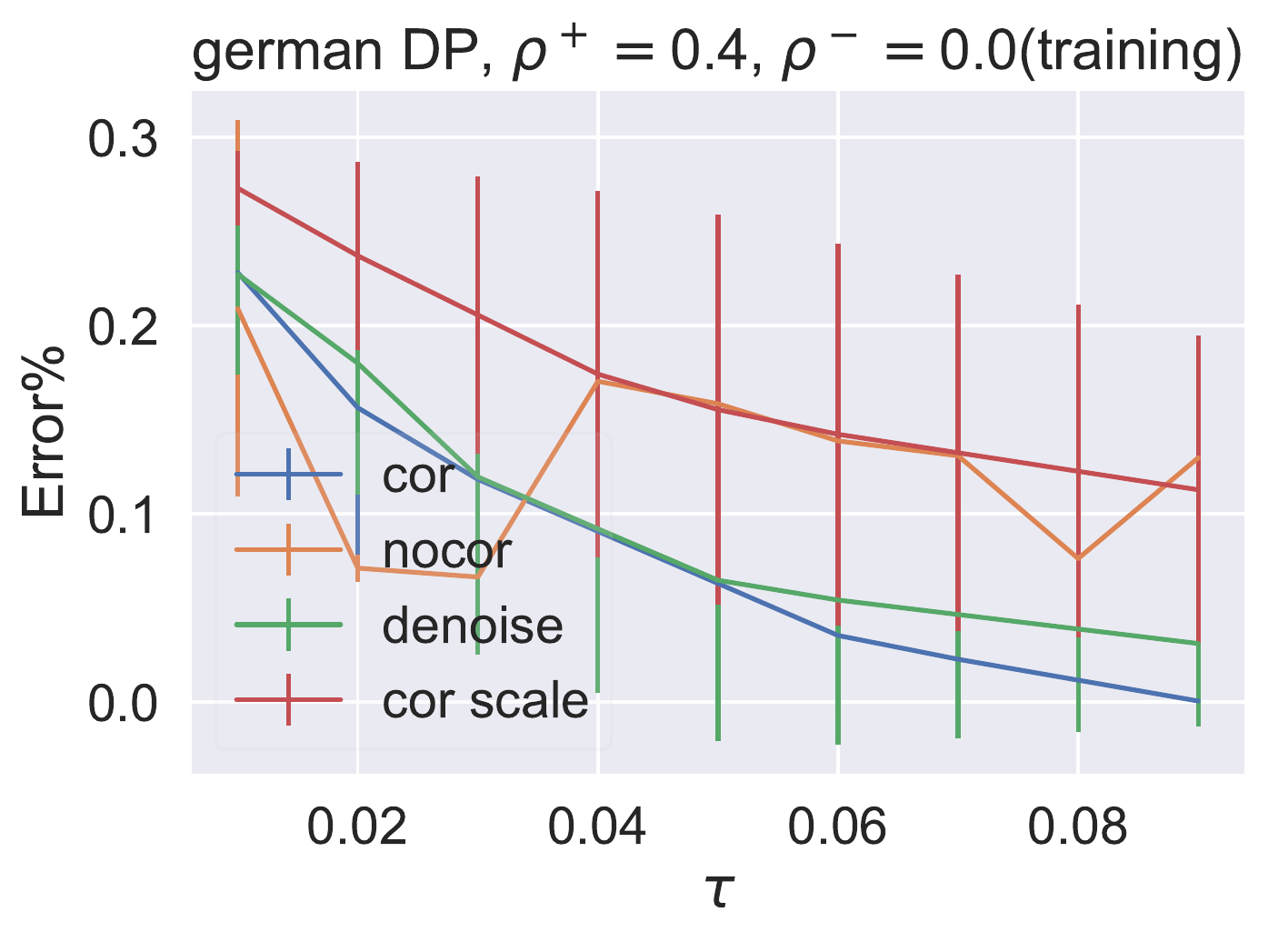}
    
    \caption{(DP)(training and testing) Relationship between input $\tau$ and fairness violation/error on the {\tt german} dataset.}
    \label{fig:DP_german}
\end{figure*}

\begin{figure*}[h]
    \centering
    \includegraphics[width=0.24\textwidth]{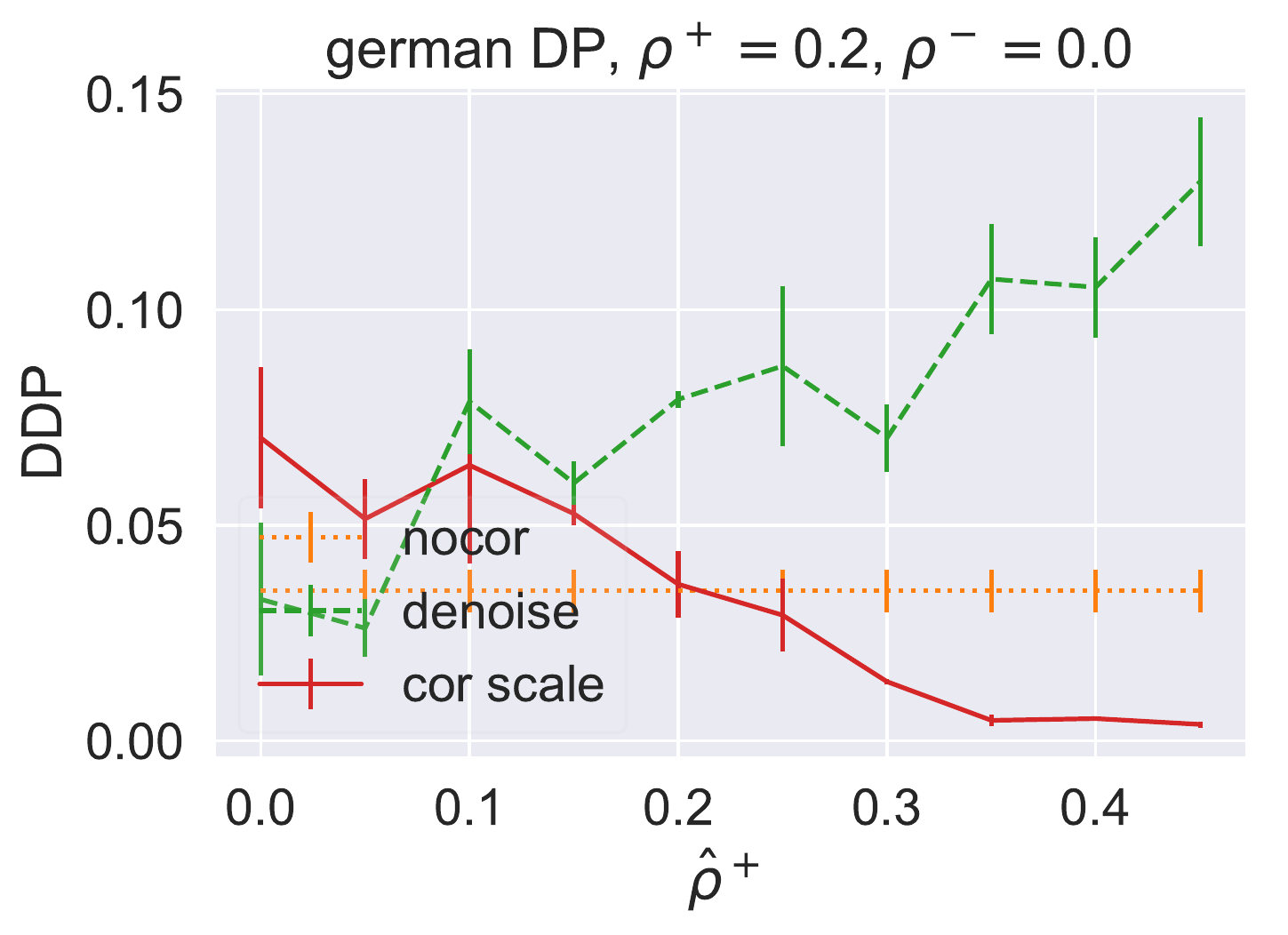}
    \includegraphics[width=0.24\textwidth]{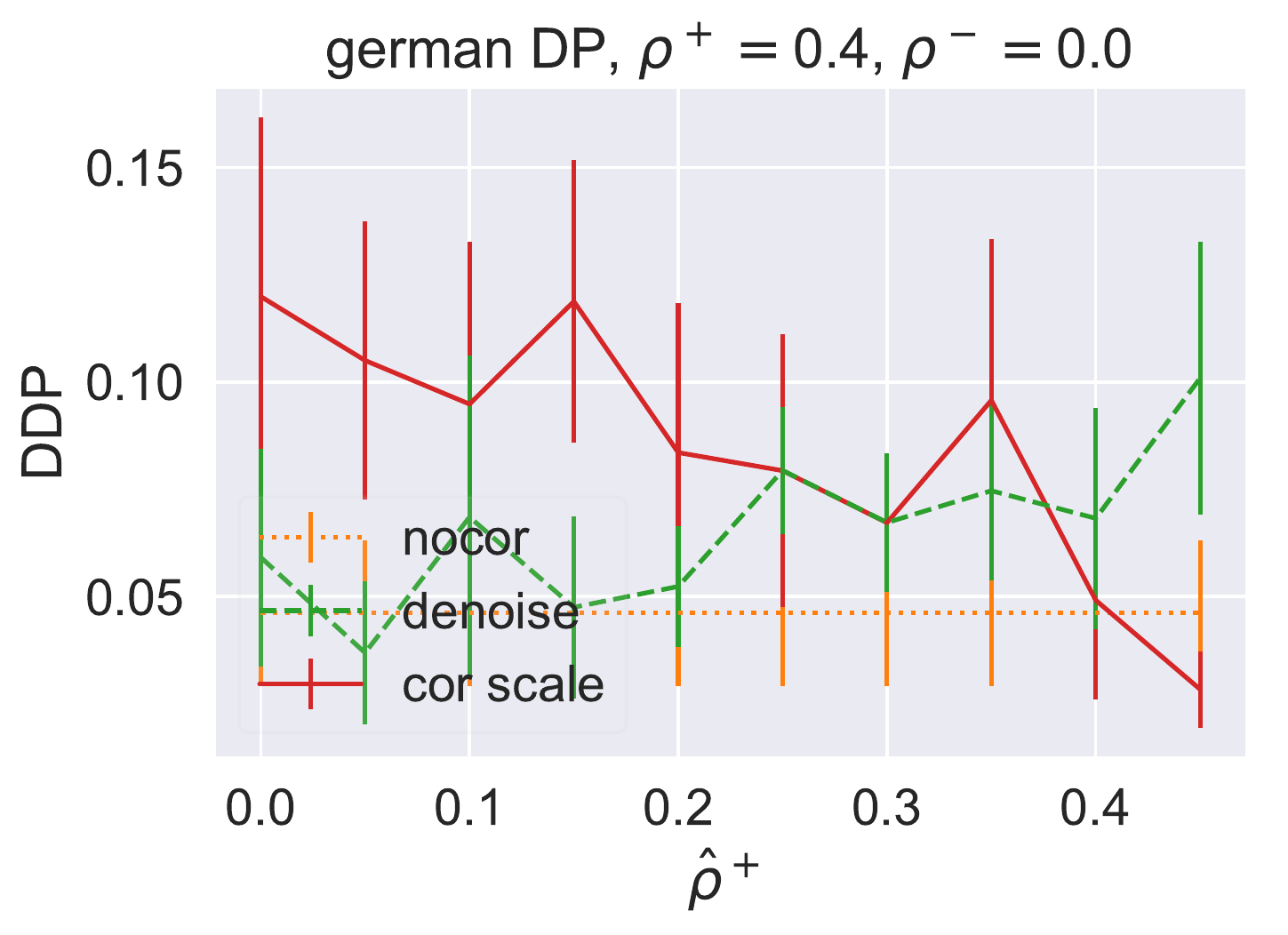}    
    \includegraphics[width=0.24\textwidth]{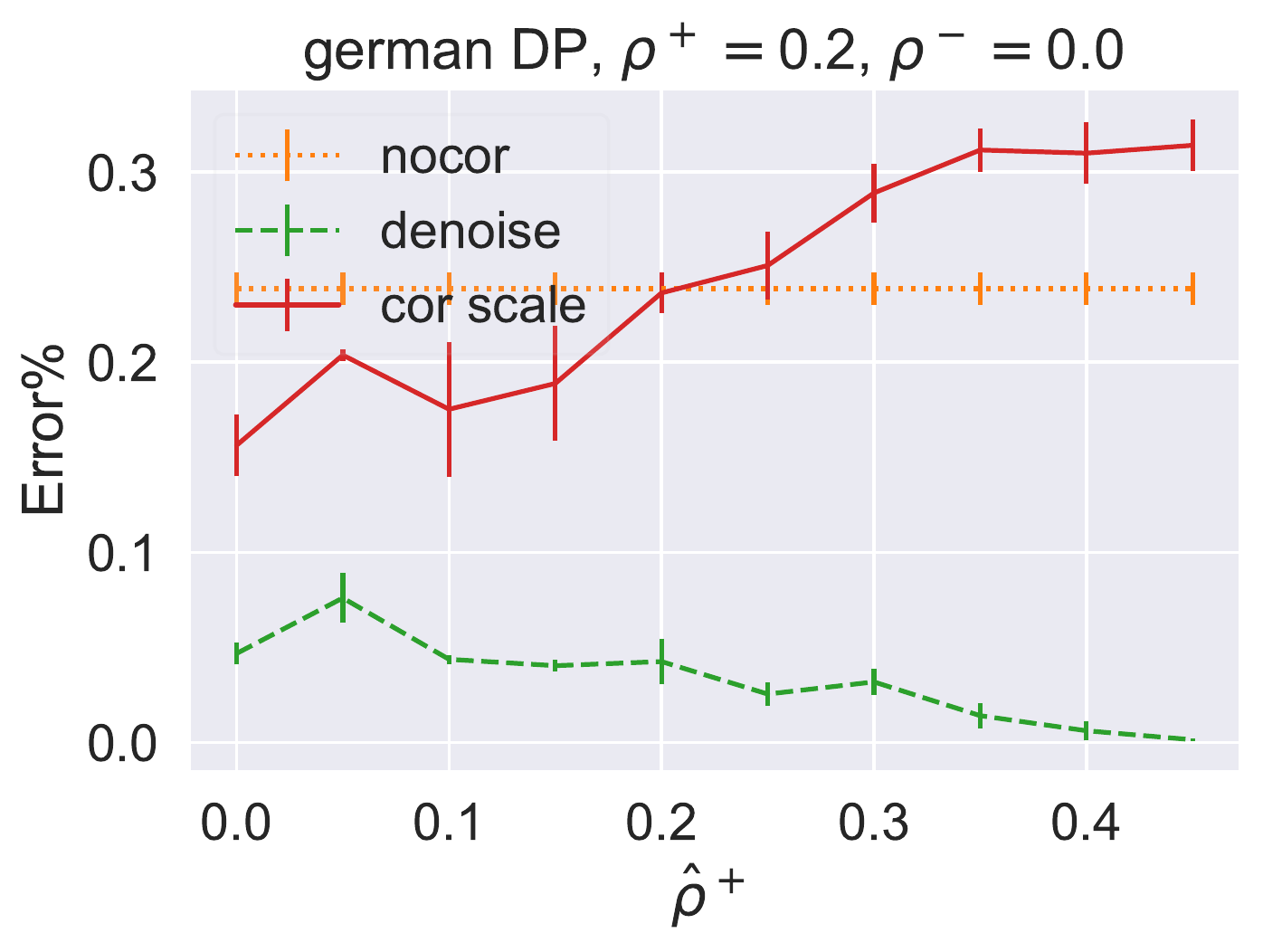}
    \includegraphics[width=0.24\textwidth]{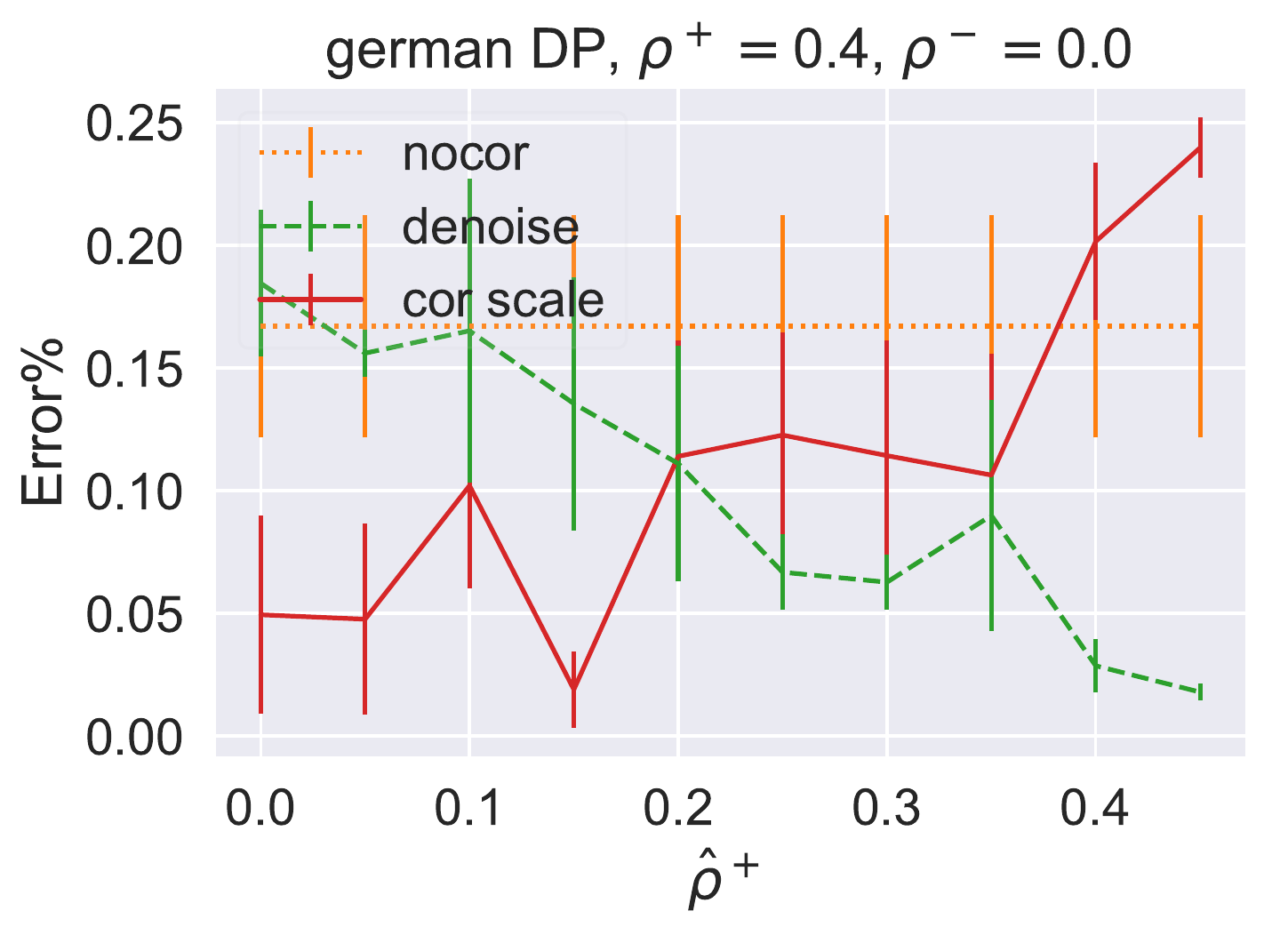}  
    \caption{Relationship between the estimated noise level $\hat{\rho}^+$ and fairness violation/error on the {\tt german} dataset using DP constraint (testing curves). Note that $\hat{\rho}^-$ is fixed to 0 and $\tau=0.04$.}
    \label{fig:german_est_err}
\end{figure*}



\clearpage

\section{The influence of different noise levels}
\label{appendix:noise-level}

Figure \ref{fig:diff_noise} explores the influence of the noise level on the trends and relationships between our method's performance and that of the benchmarks. We run these experiments on the UCI {\tt adult} dataset, which is another dataset from the UCI repository ~\citep{UCI}. The task is to predict if one has income more than 50K and gender is the sensitive attribute. The data comprises 48842 examples and 14 features. We run these experiments with the DP constraint under different CCN noise levels ($\rho^+ = \rho^- \in \set{0.01, 0.1, 0.2, 0.3, 0.4, 0.48}$). We include both training and testing curves for completeness. As we can see, as the noise increases the gap between the corrupted data curves and the uncorrupted data curve increases. It becomes very hard to get close to the non-corrupted case when noise becomes too high.

\begin{figure*}[h!]
    \centering
    \includegraphics[width=0.24\textwidth]{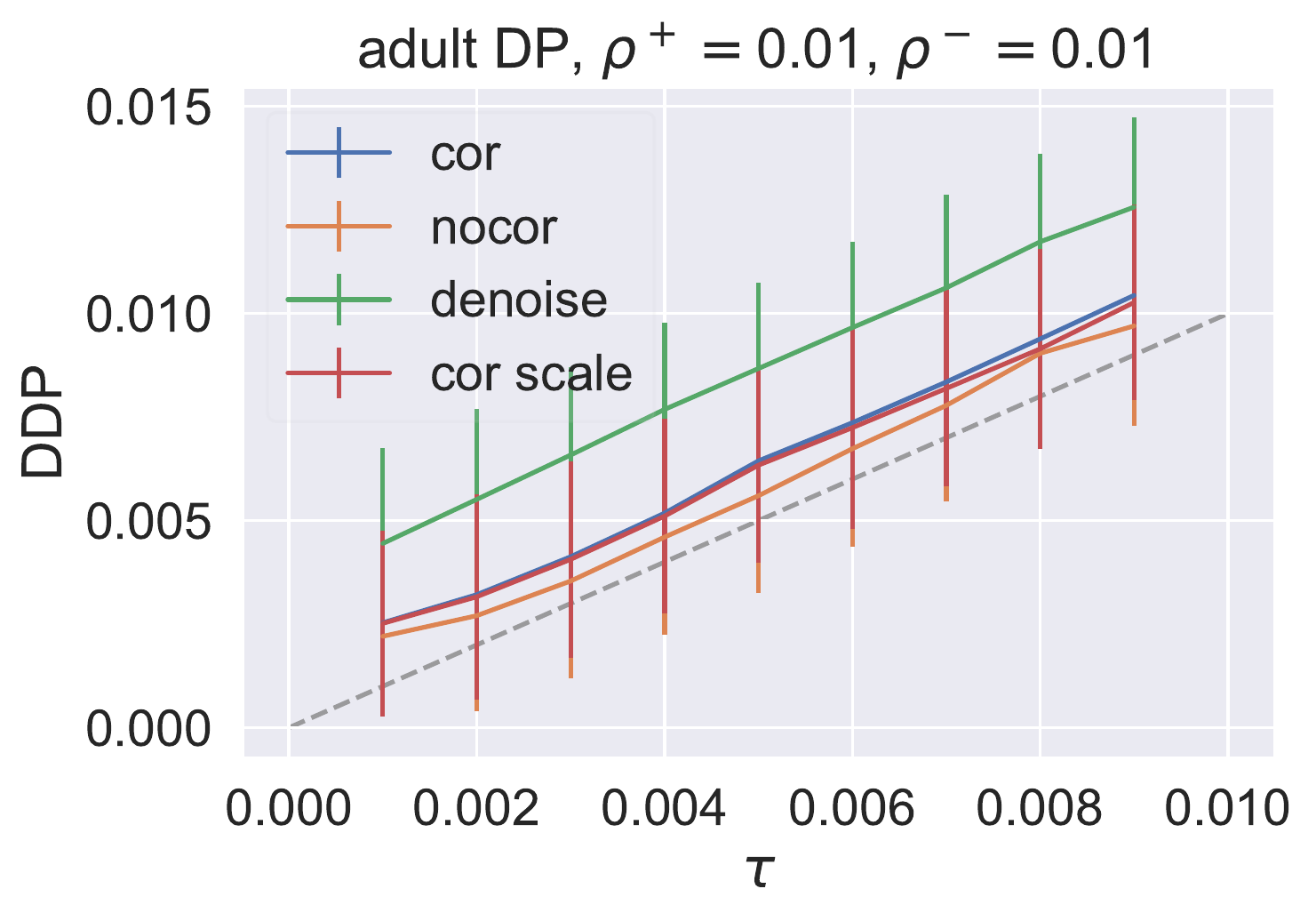}
    \includegraphics[width=0.24\textwidth]{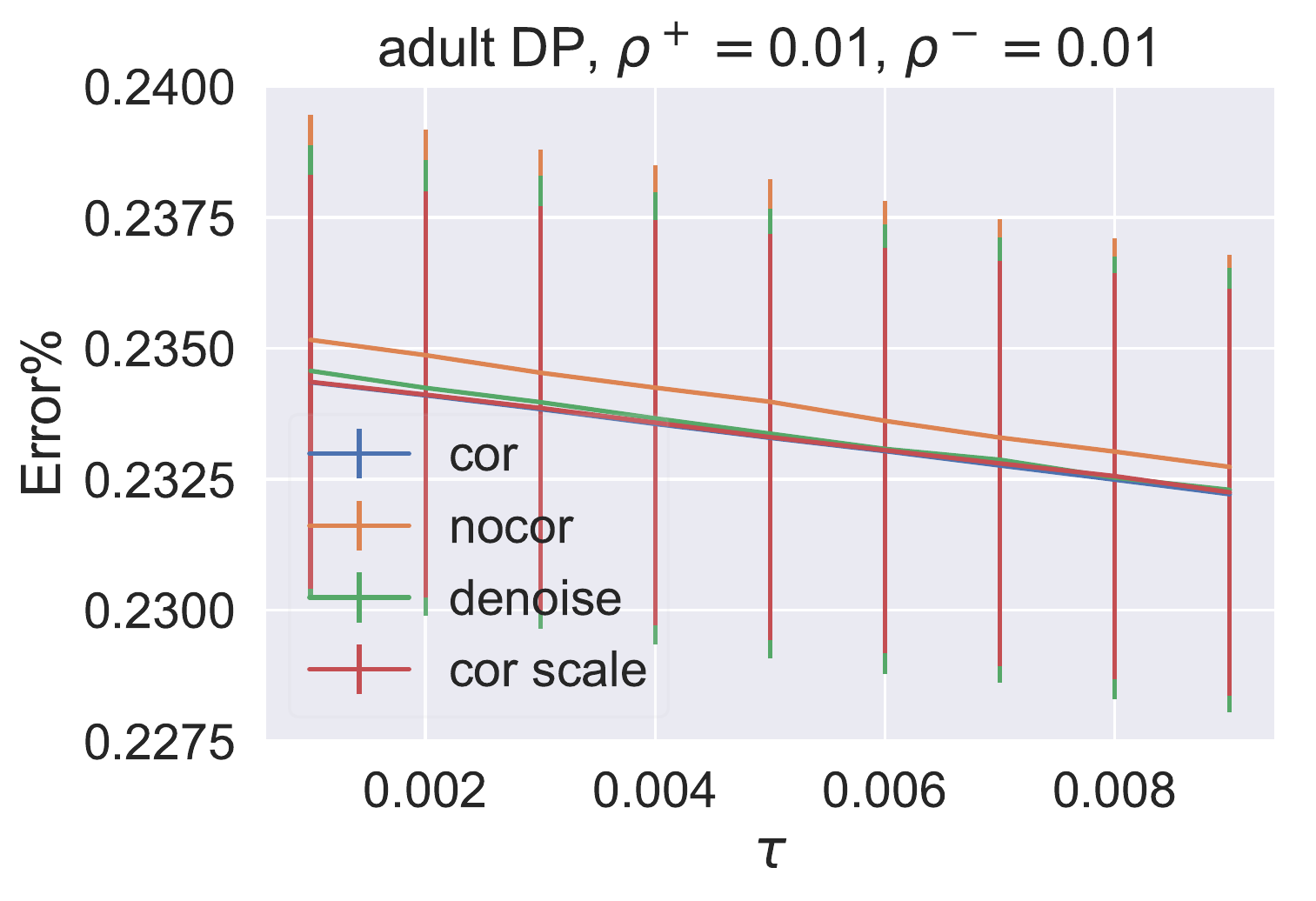}
    \includegraphics[width=0.24\textwidth]{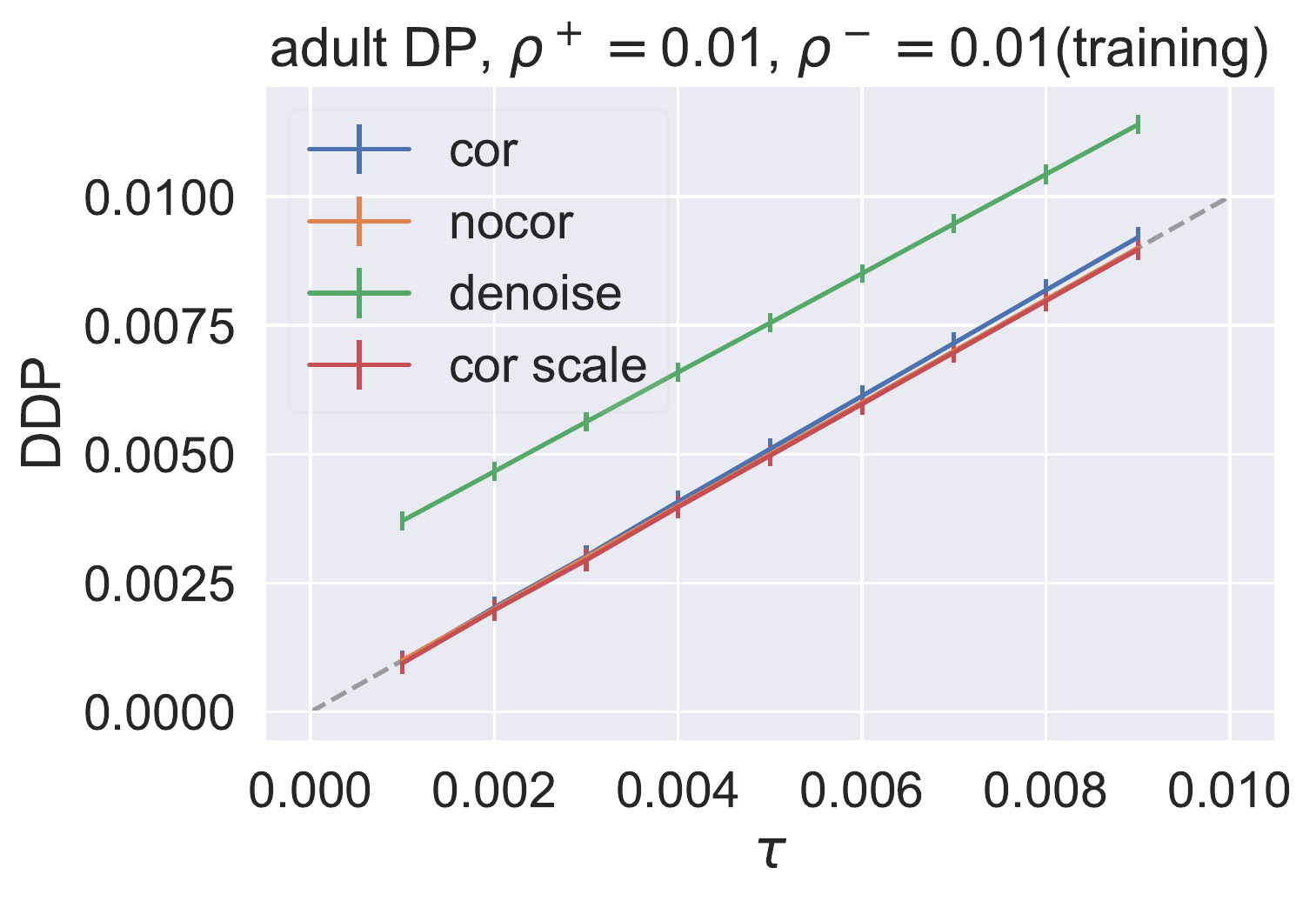}
    \includegraphics[width=0.24\textwidth]{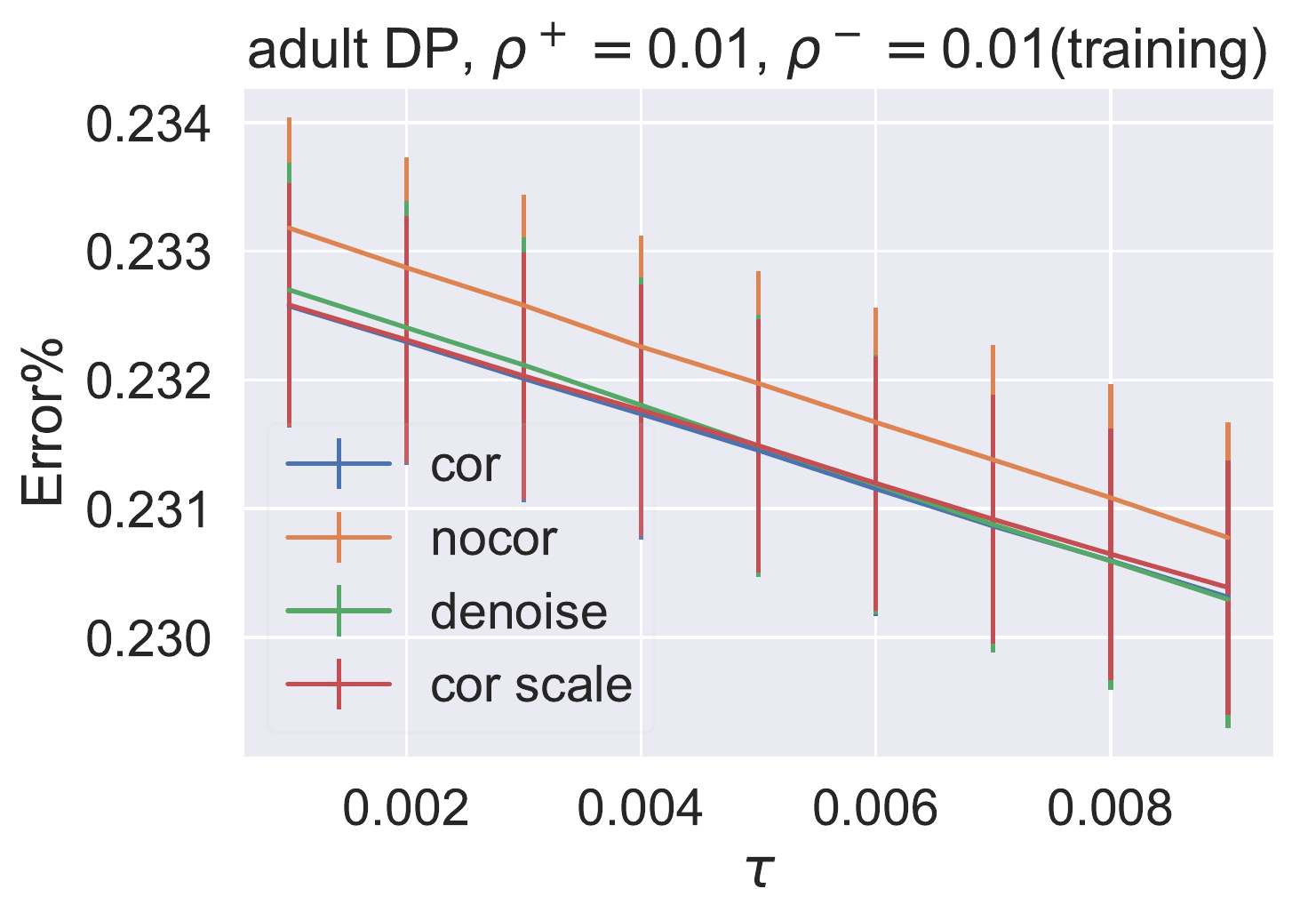}
    
    \includegraphics[width=0.24\textwidth]{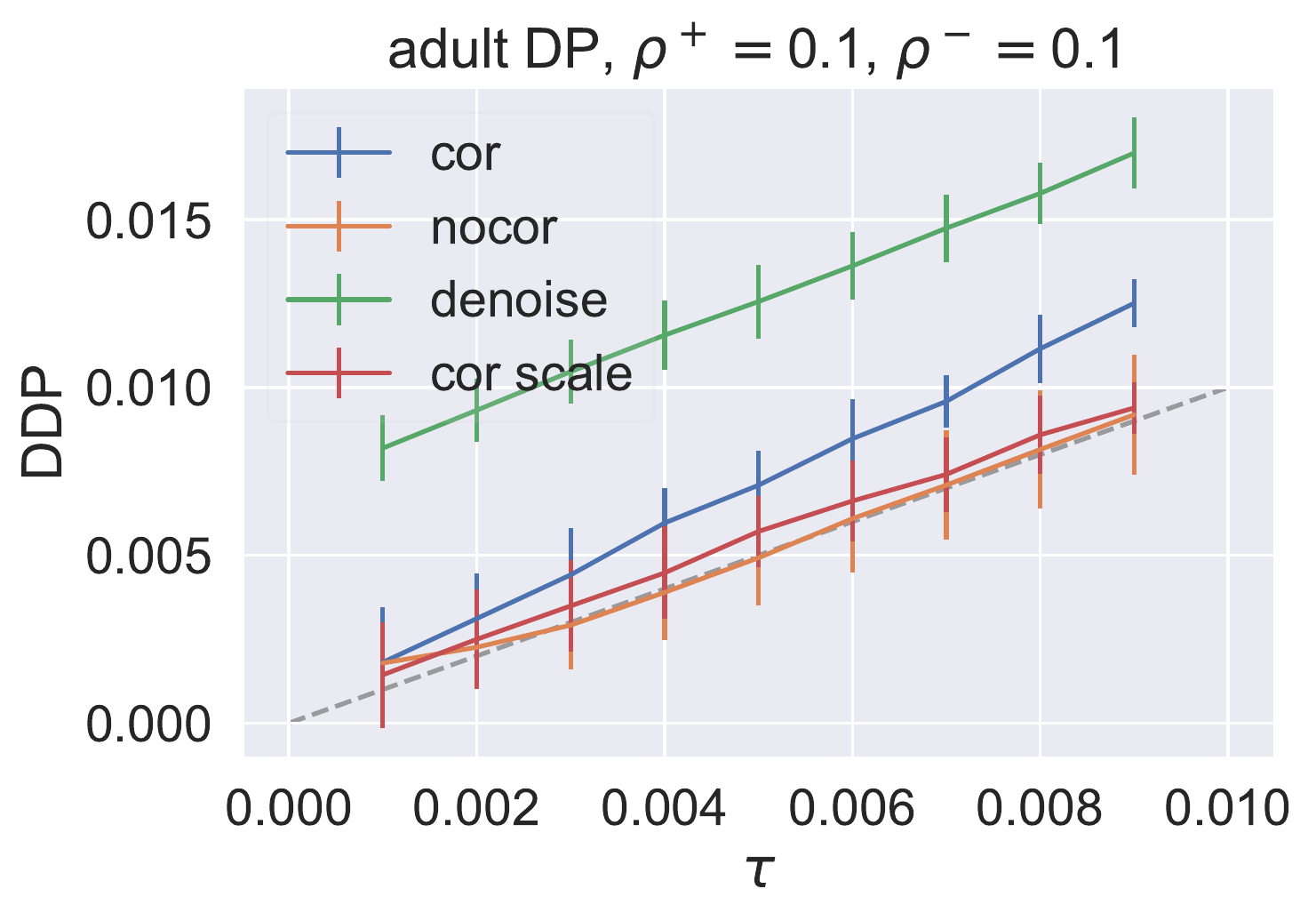}
    \includegraphics[width=0.24\textwidth]{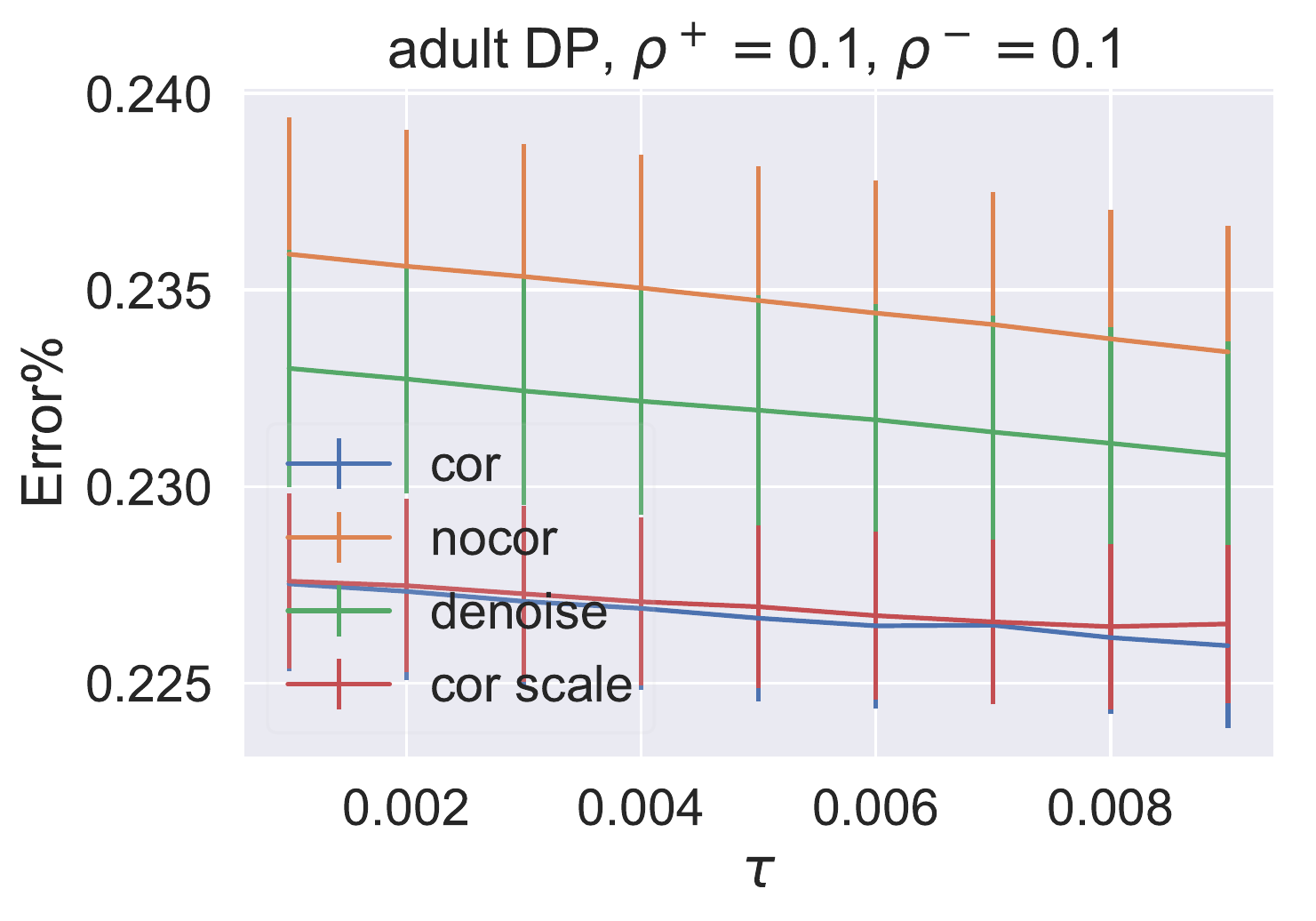}
    \includegraphics[width=0.24\textwidth]{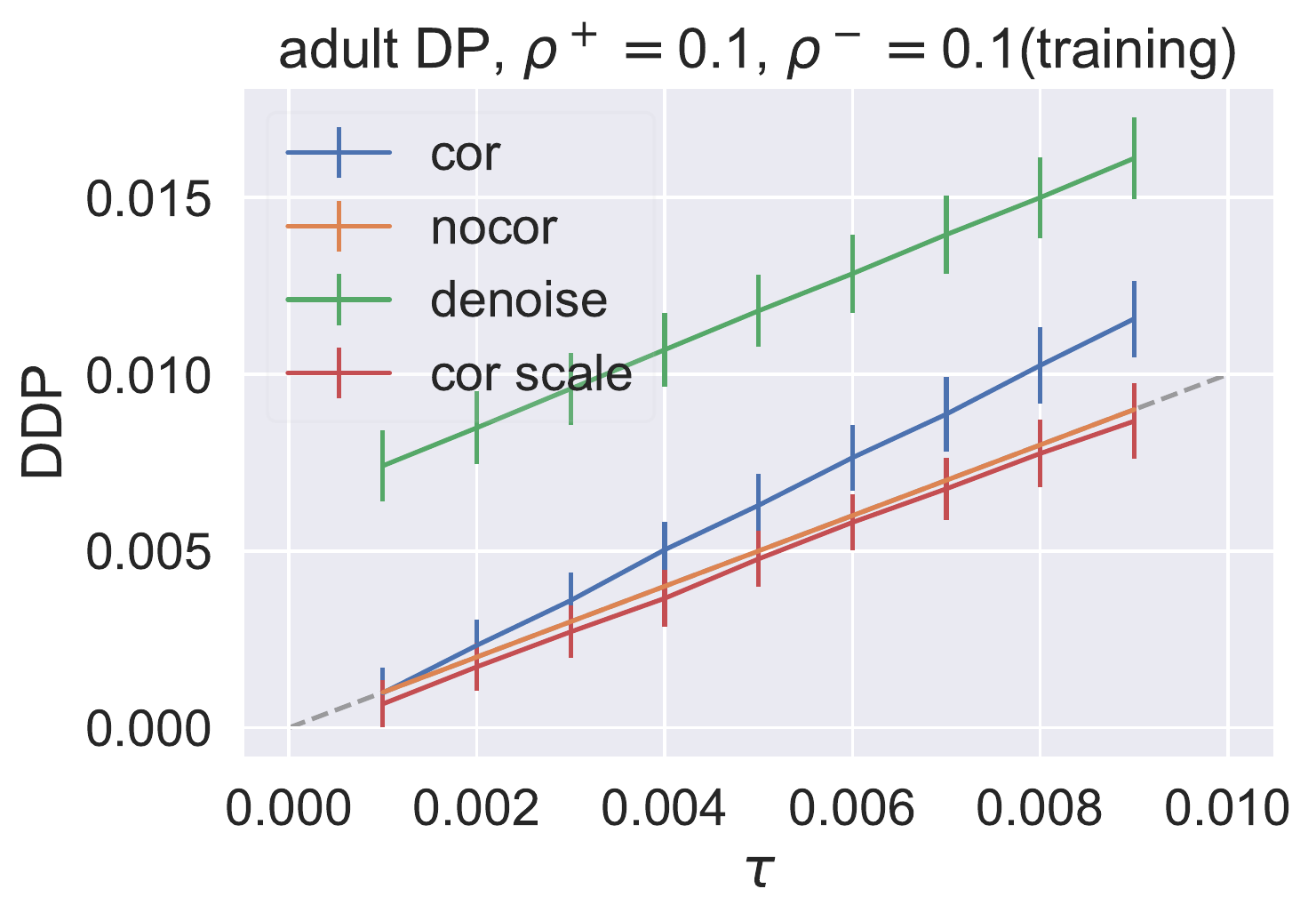}
    \includegraphics[width=0.24\textwidth]{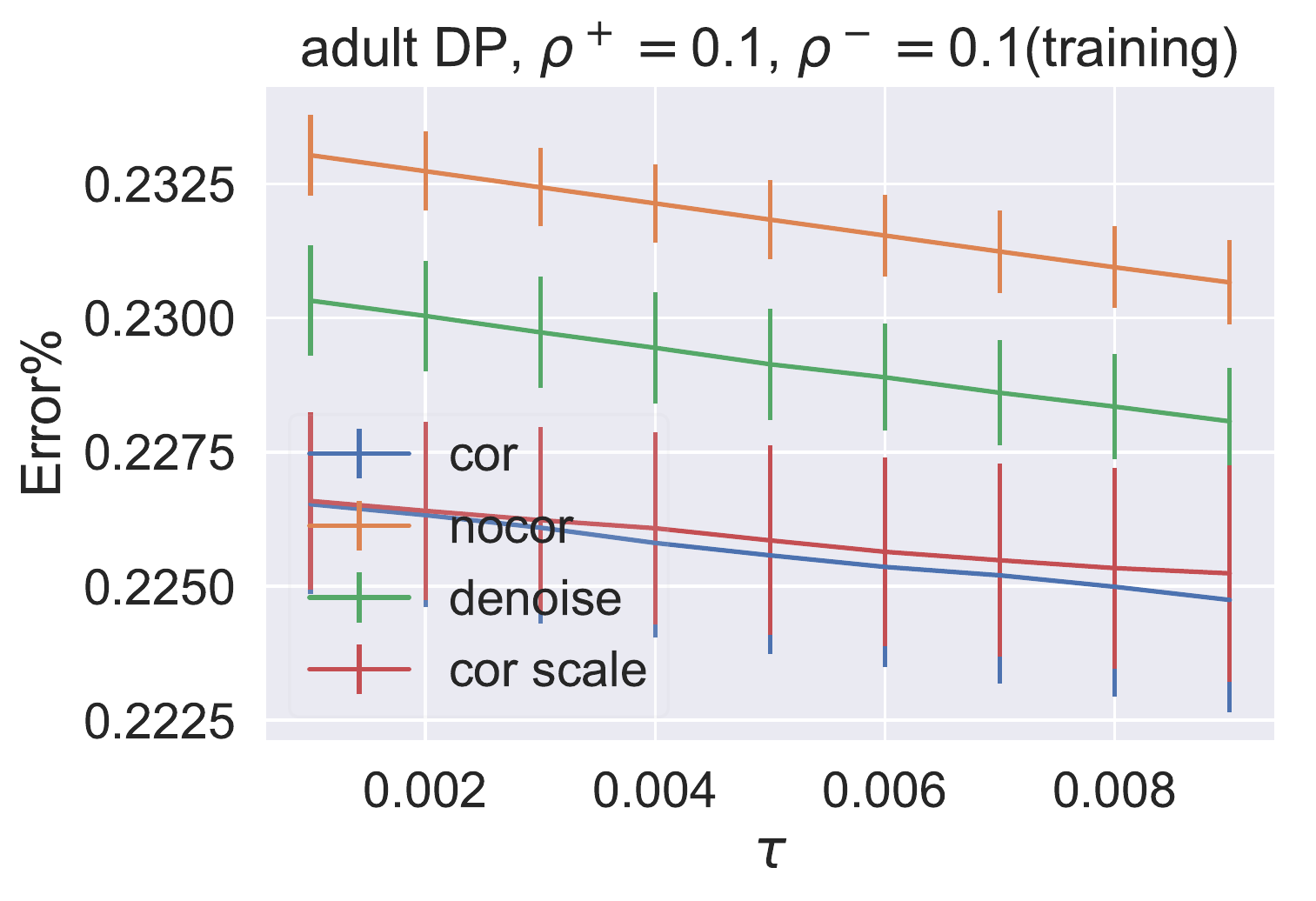}
    
    \includegraphics[width=0.24\textwidth]{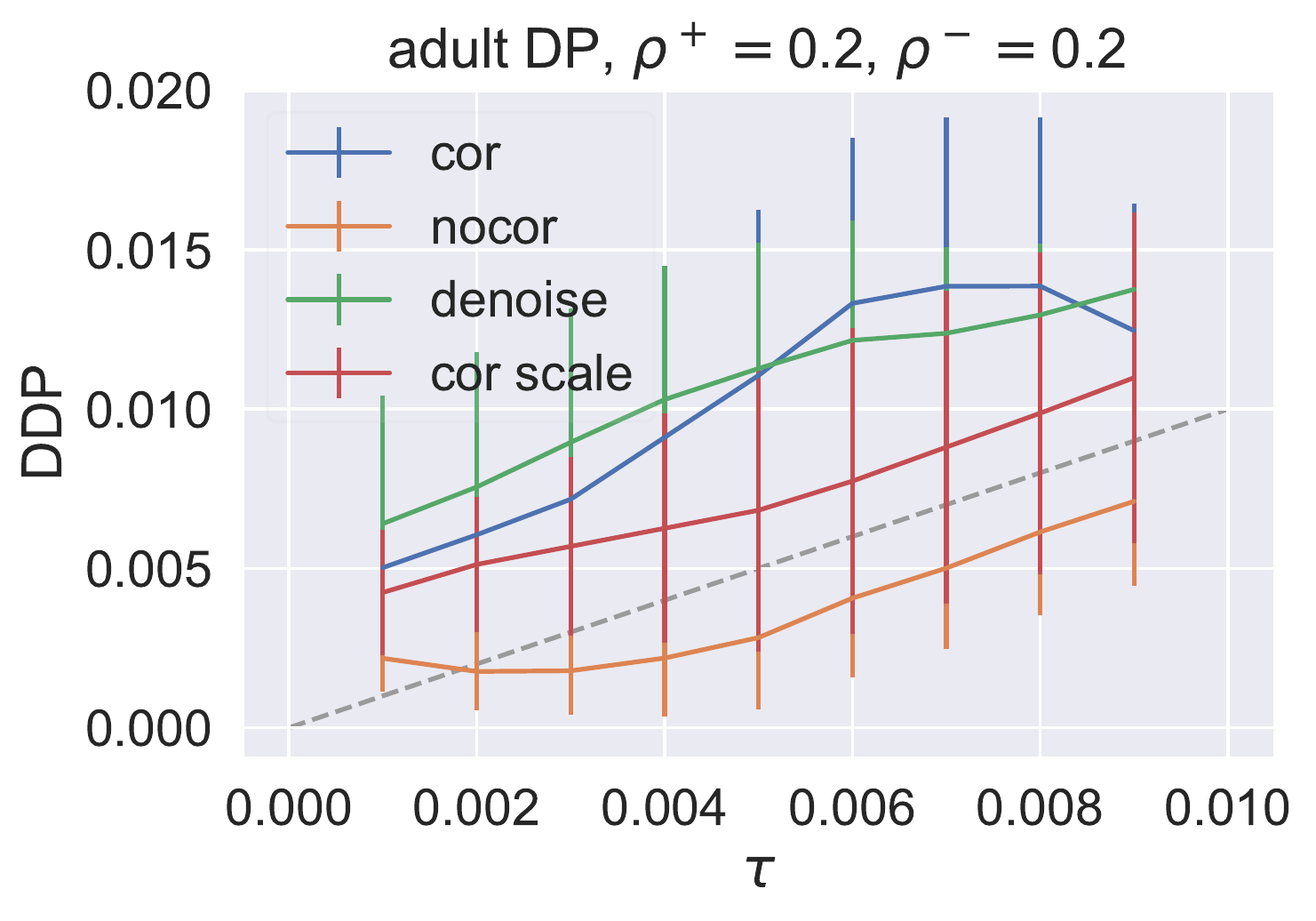}
    \includegraphics[width=0.24\textwidth]{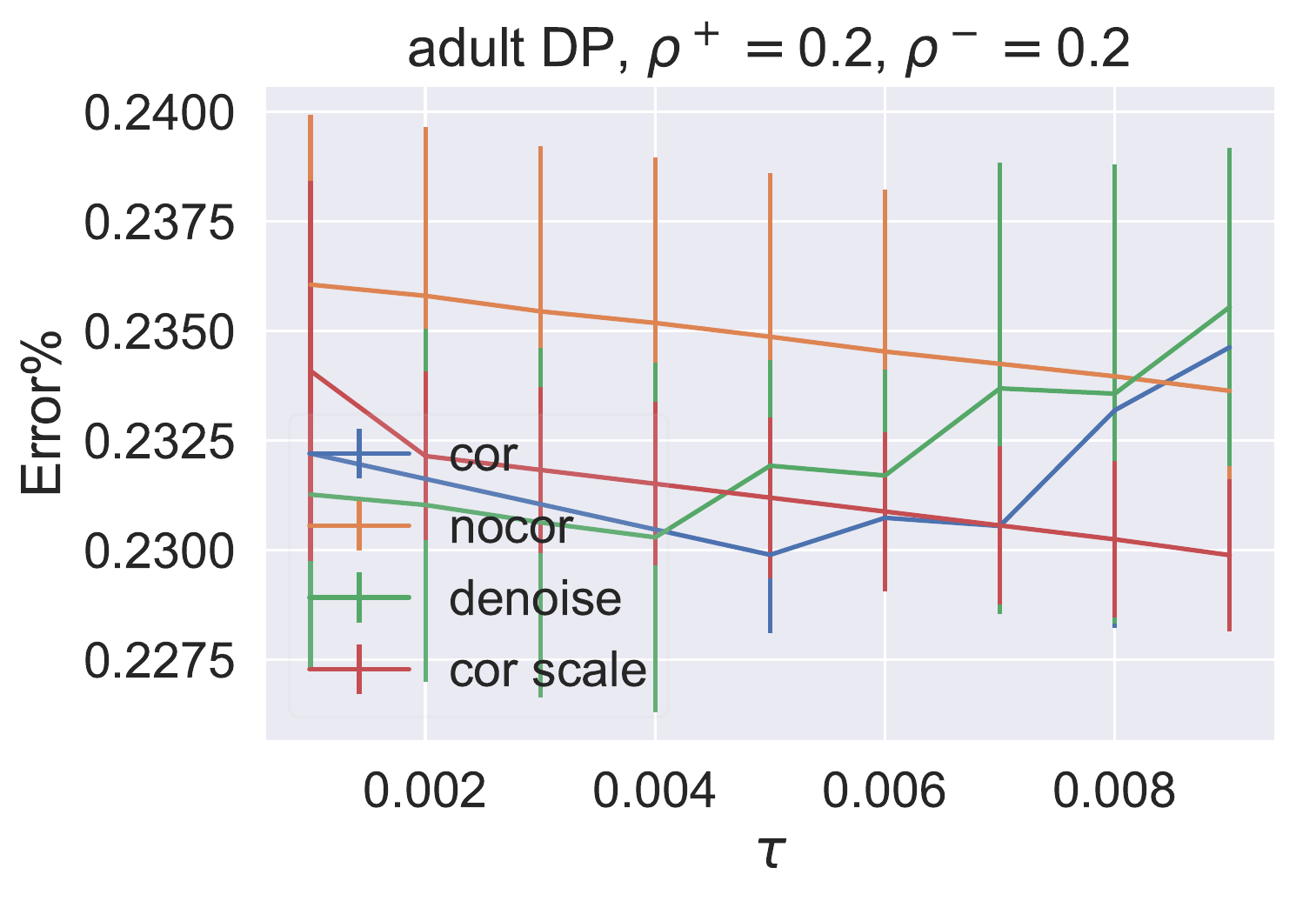}
    \includegraphics[width=0.24\textwidth]{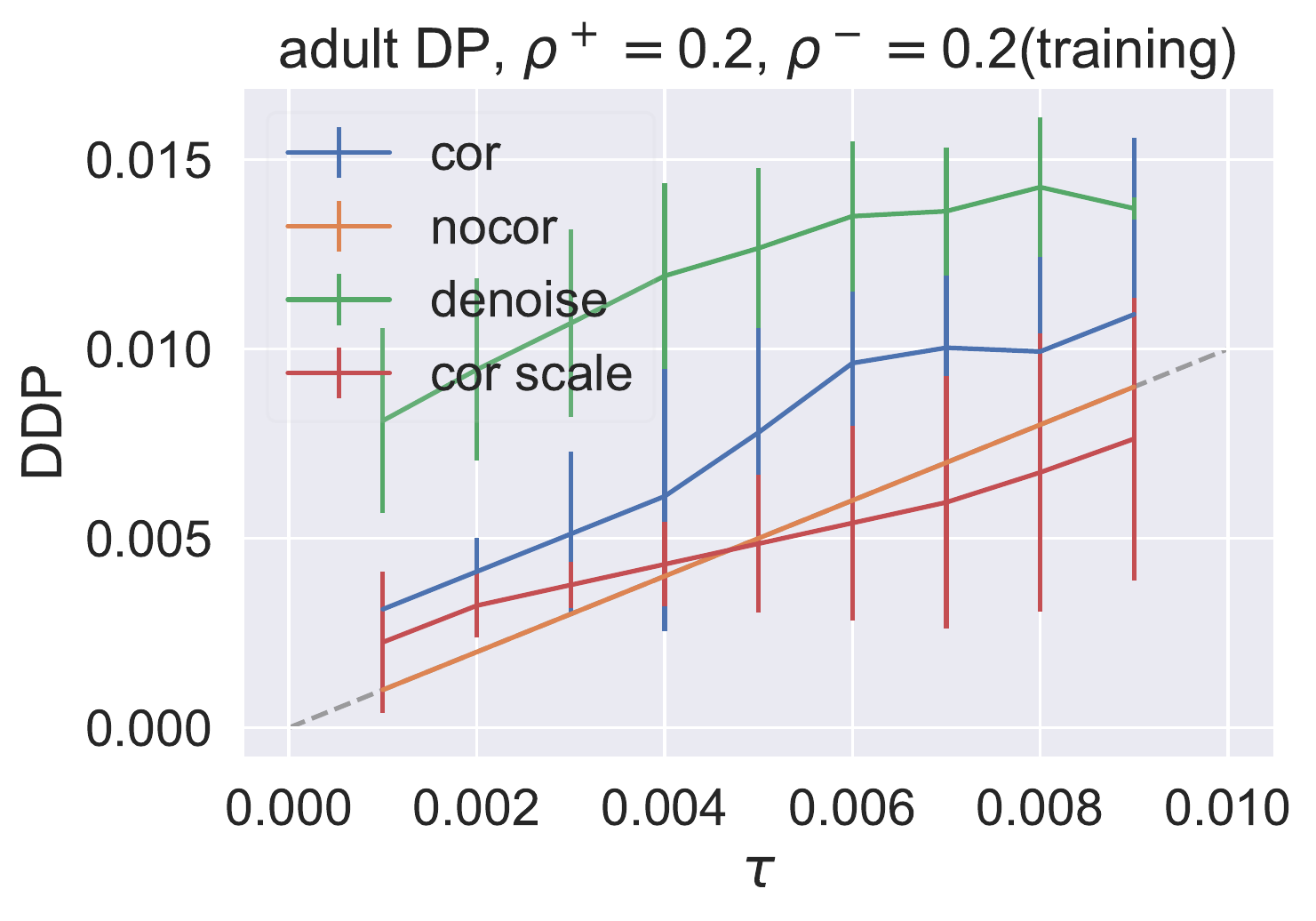}
    \includegraphics[width=0.24\textwidth]{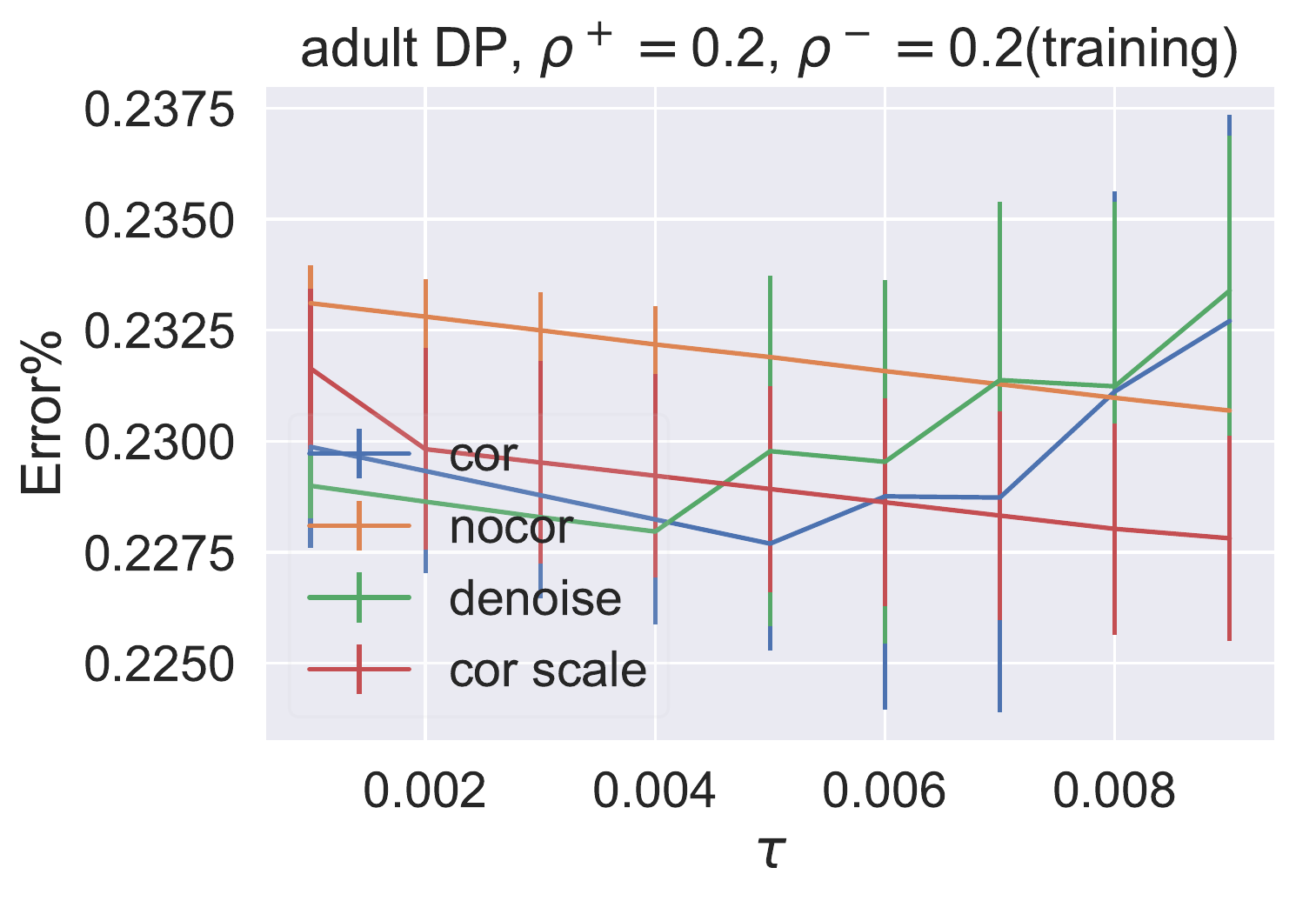}
    
    \includegraphics[width=0.24\textwidth]{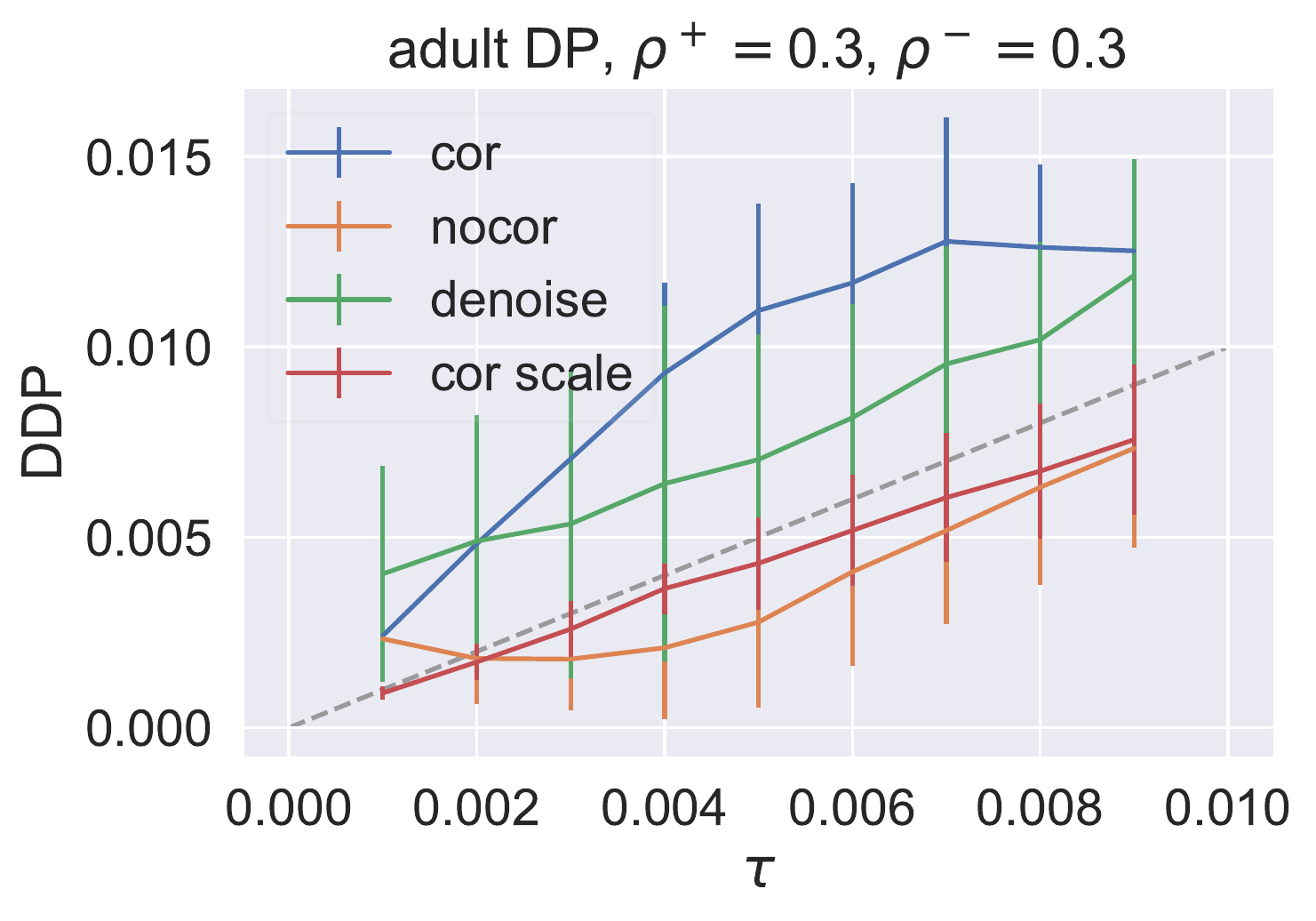}
    \includegraphics[width=0.24\textwidth]{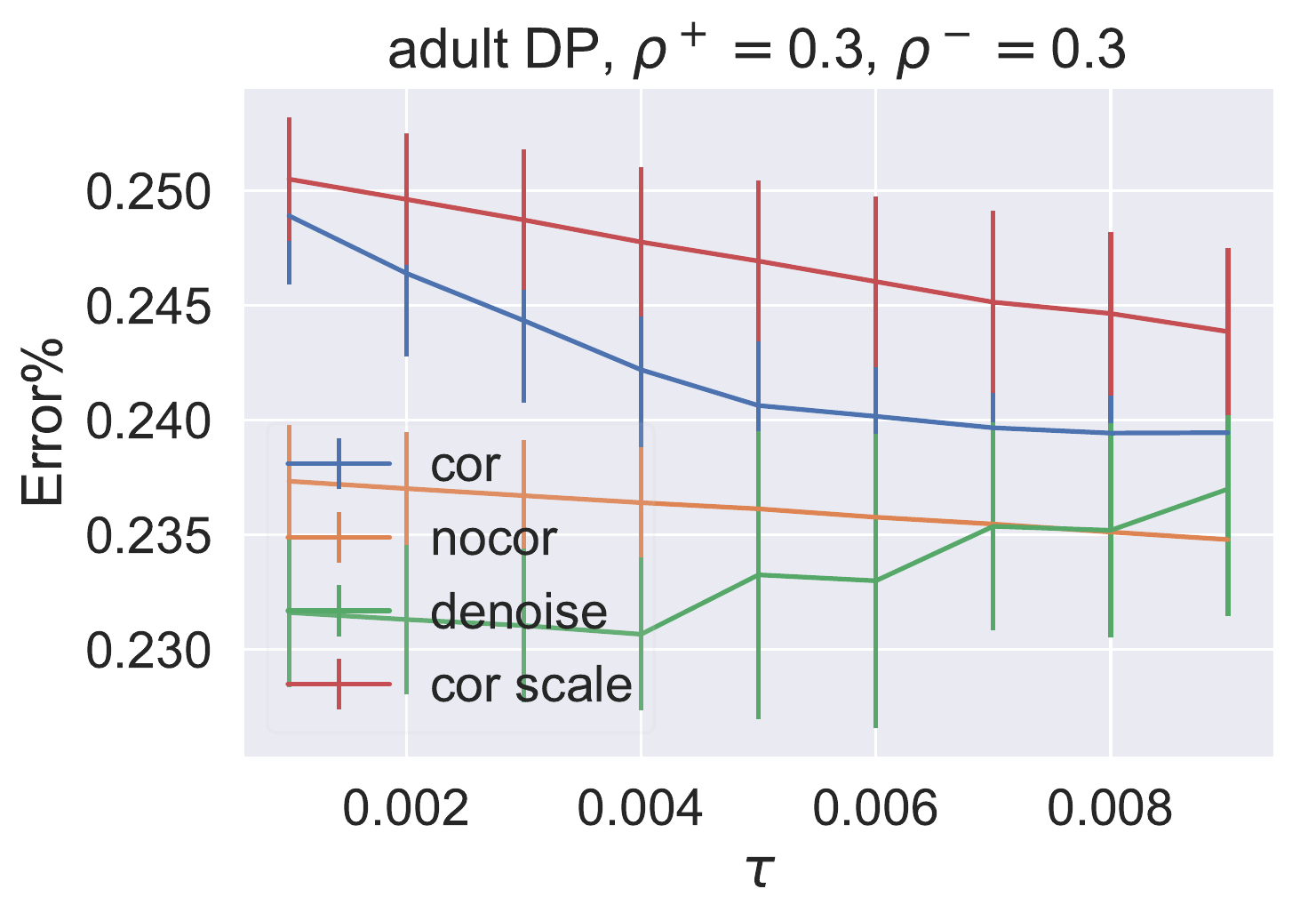}
    \includegraphics[width=0.24\textwidth]{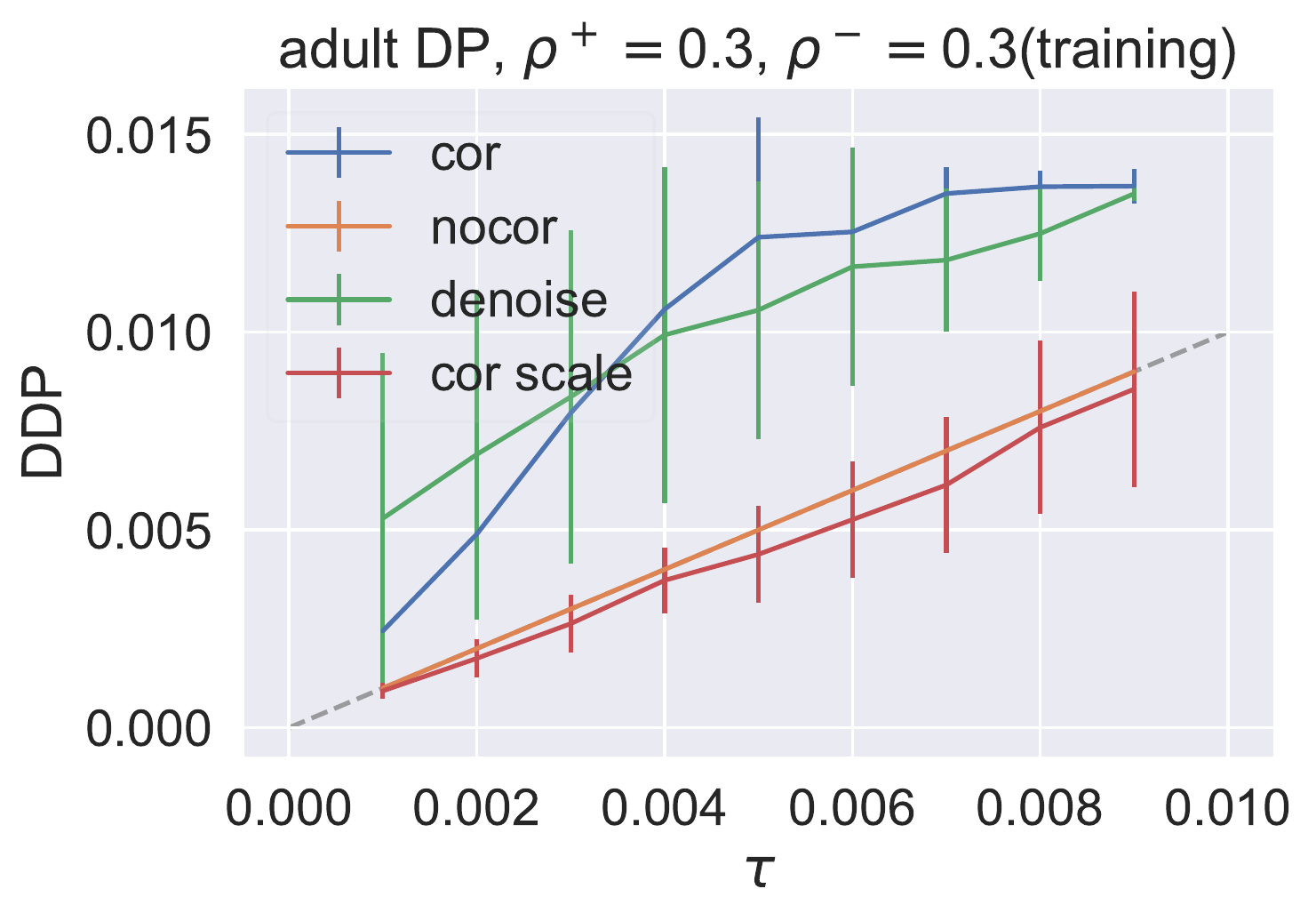}
    \includegraphics[width=0.24\textwidth]{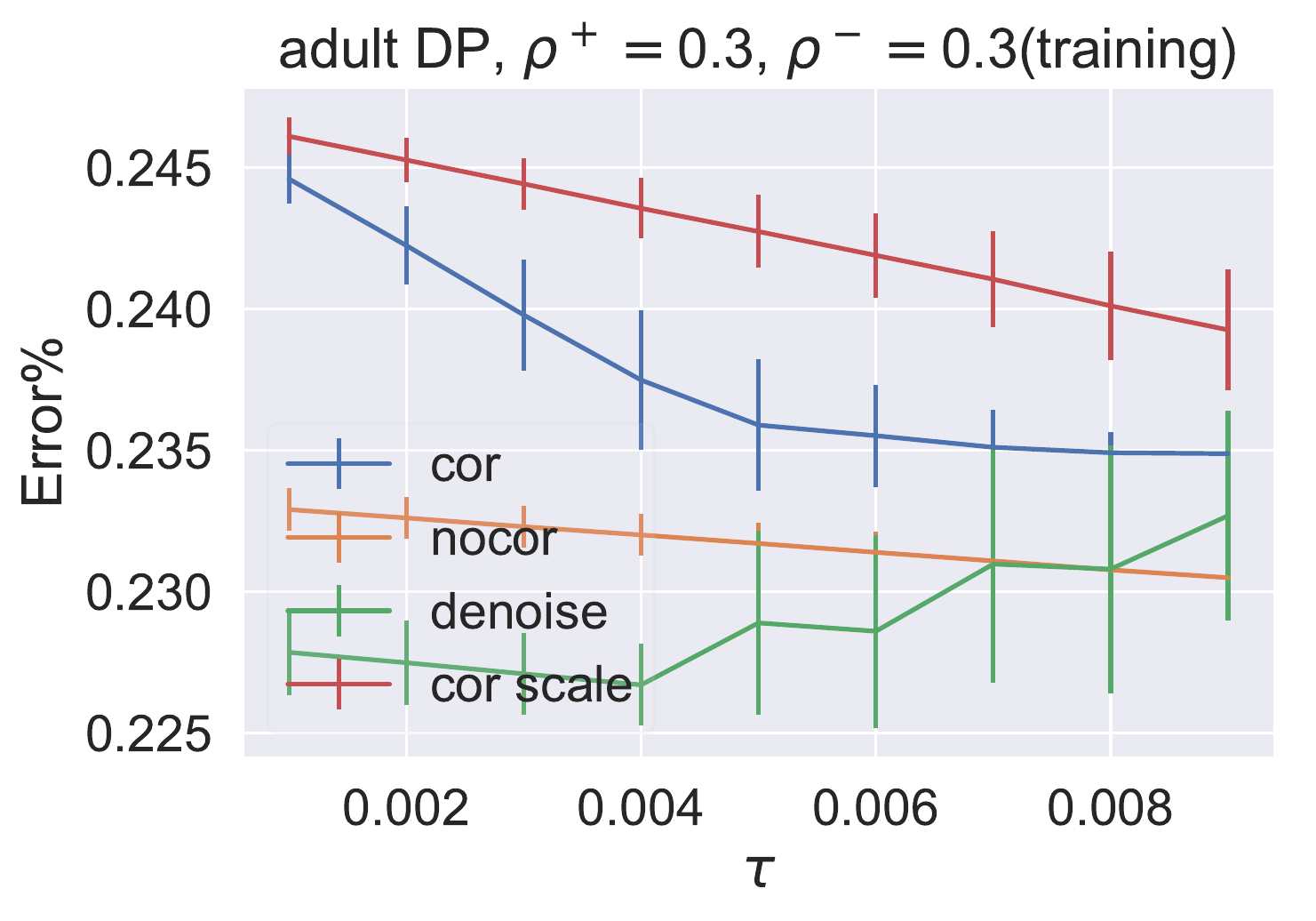}
    
    \includegraphics[width=0.24\textwidth]{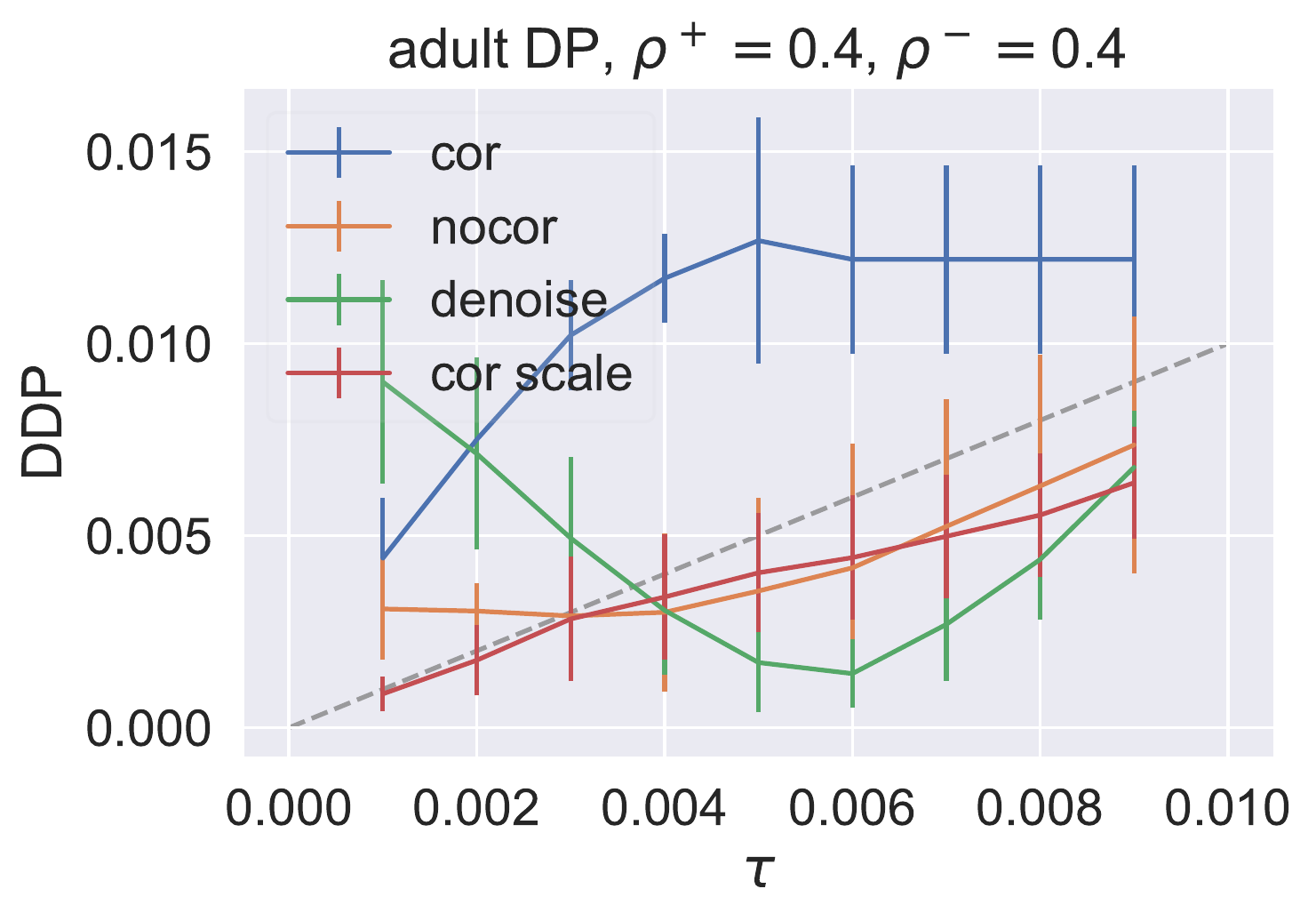}
    \includegraphics[width=0.24\textwidth]{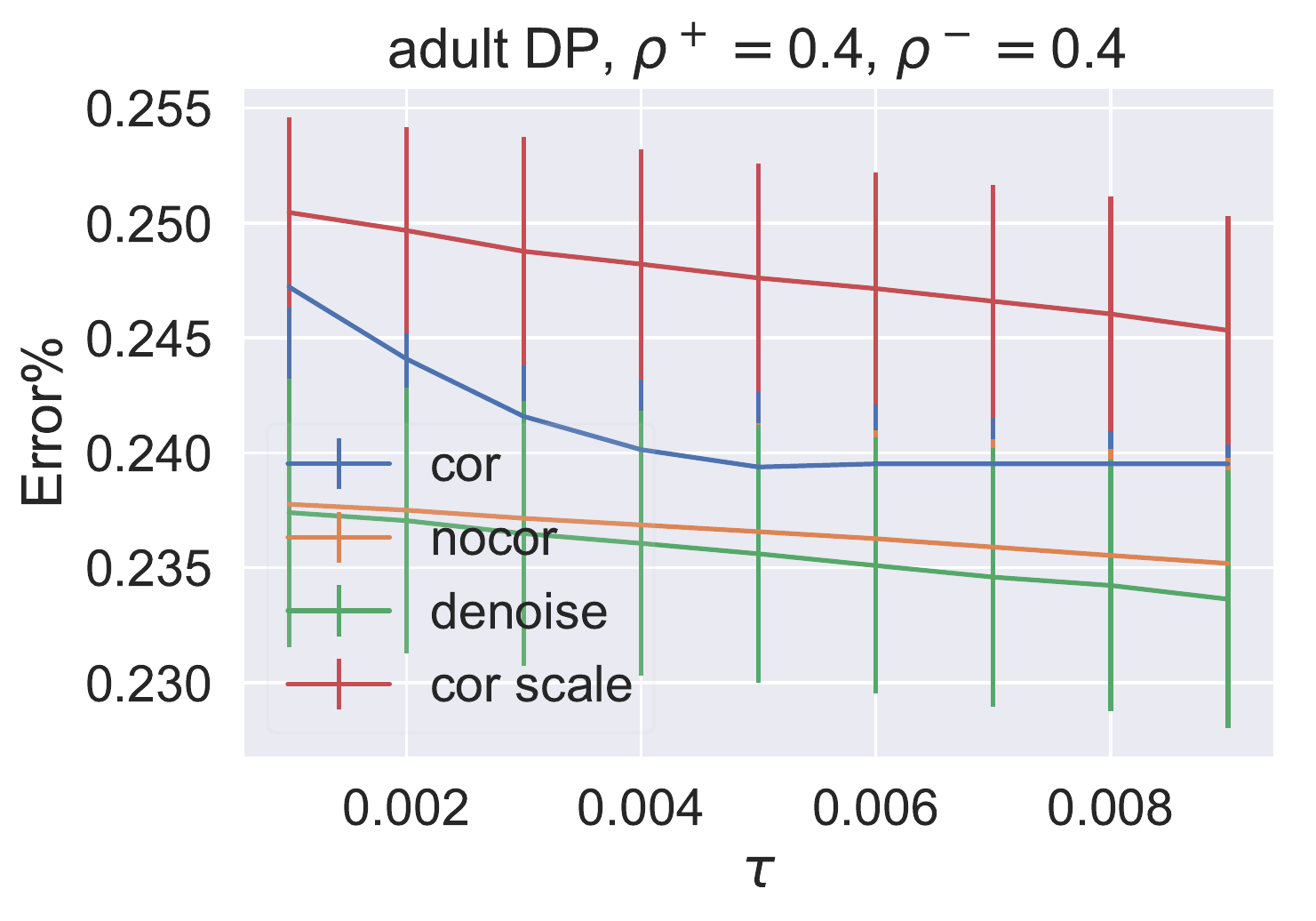}
    \includegraphics[width=0.24\textwidth]{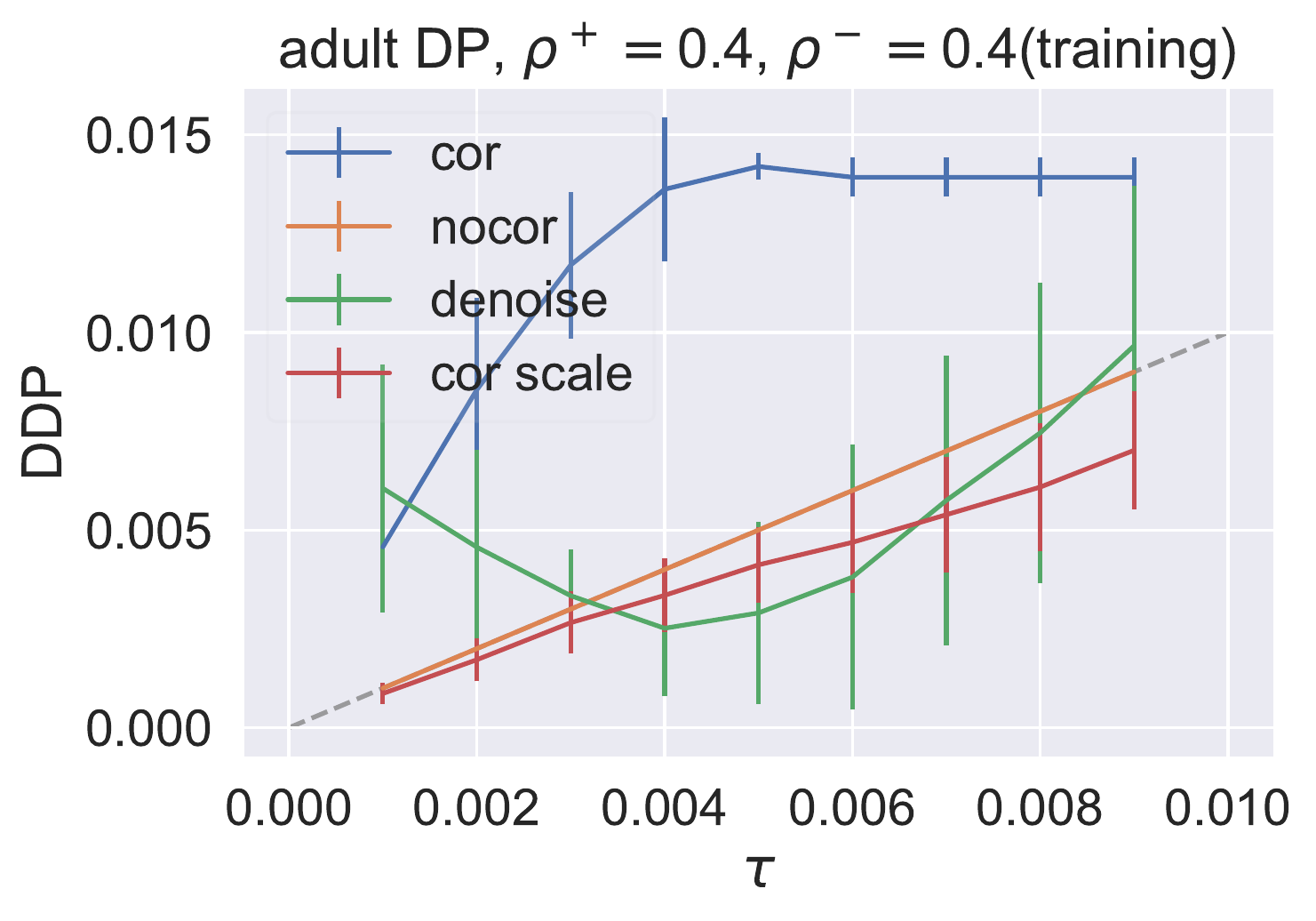}
    \includegraphics[width=0.24\textwidth]{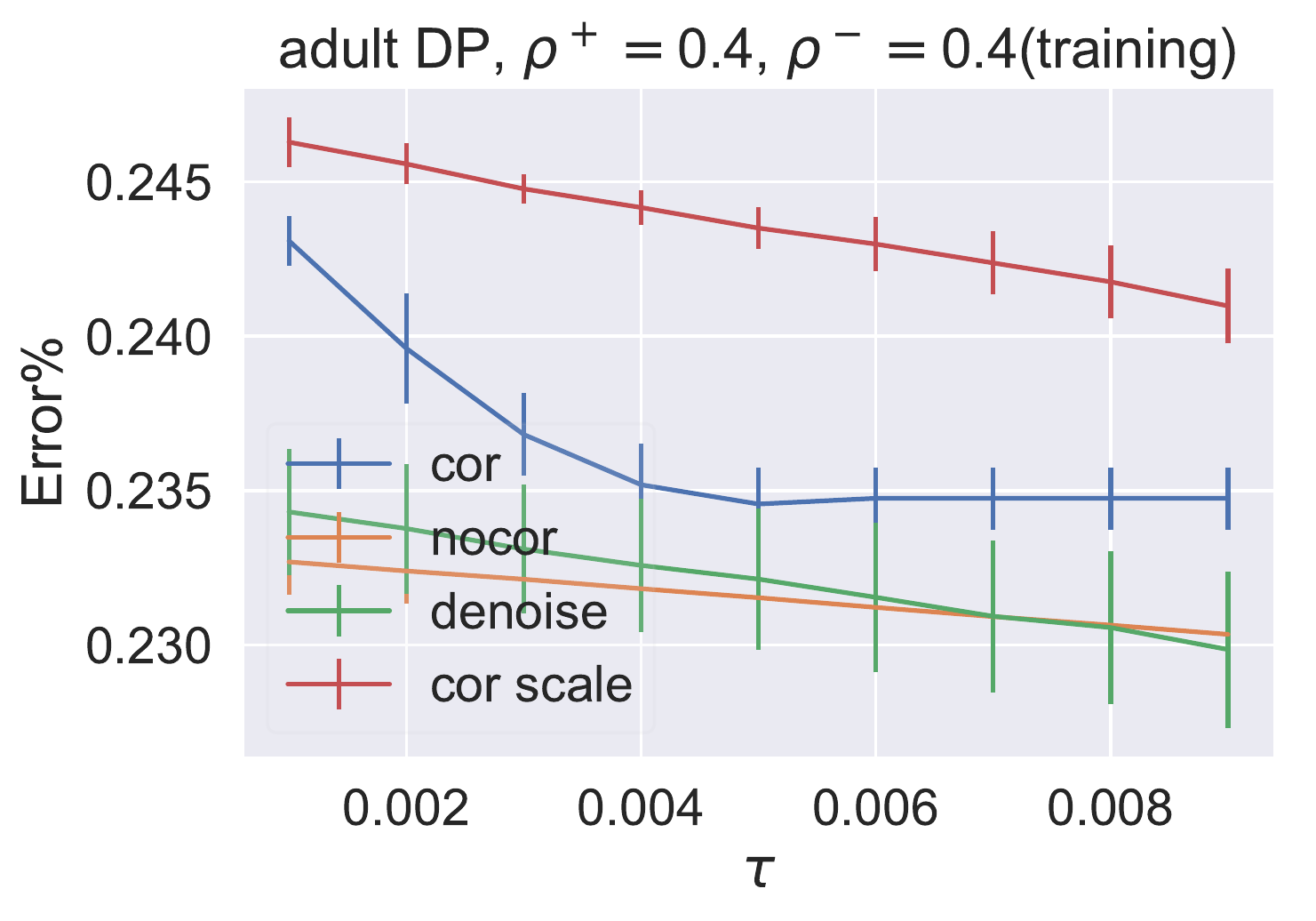}
    
    \includegraphics[width=0.24\textwidth]{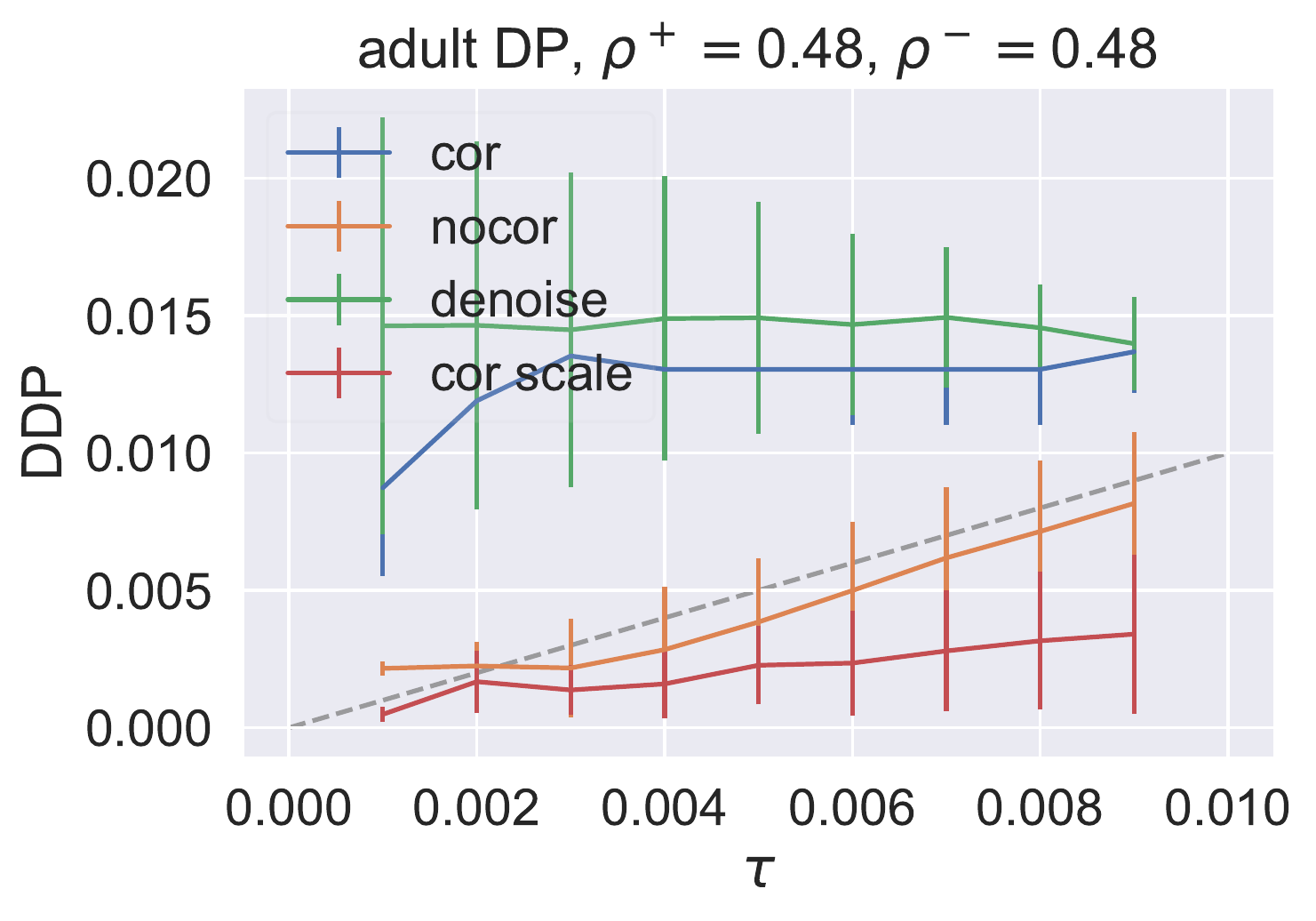}
    \includegraphics[width=0.24\textwidth]{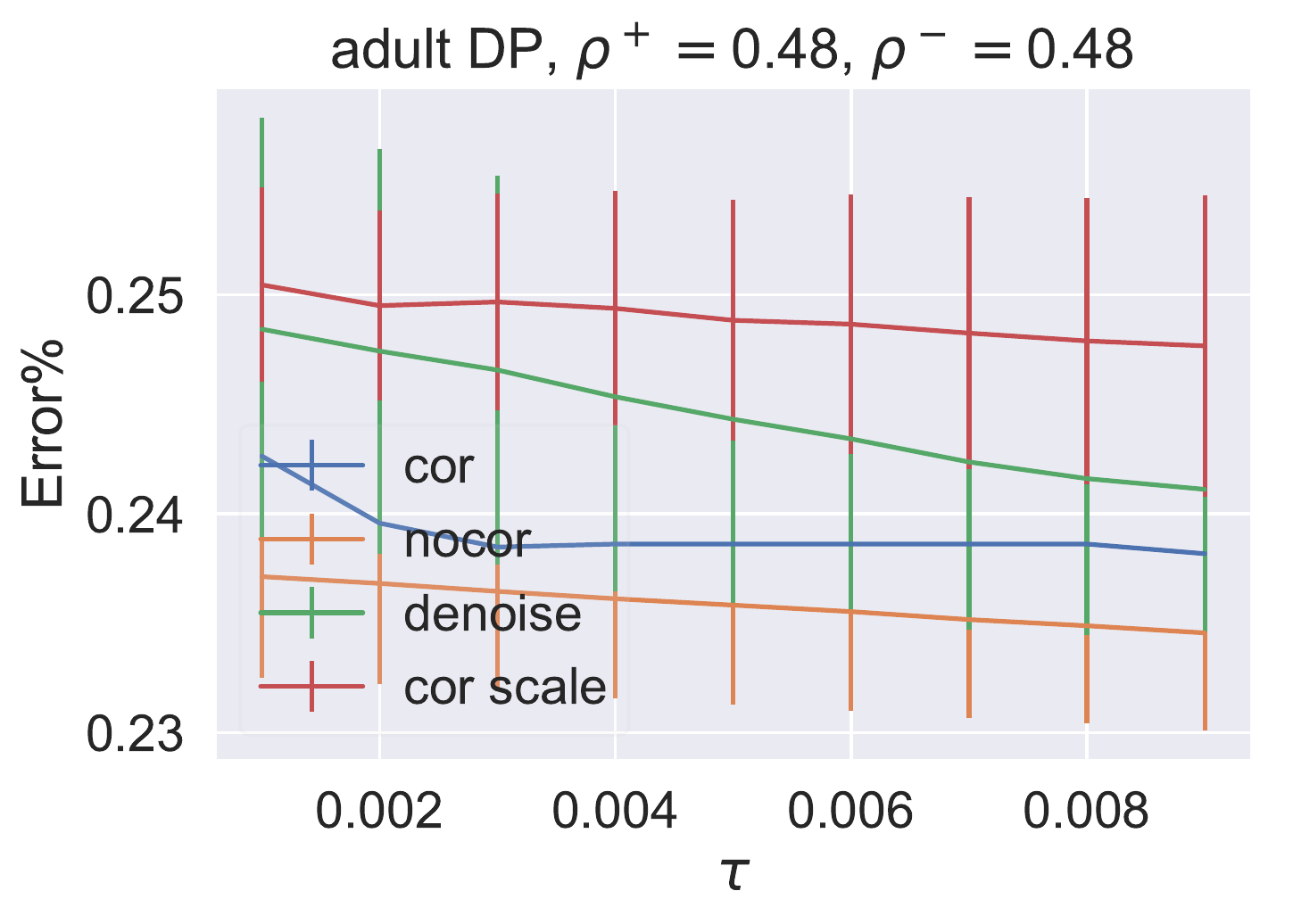}
    \includegraphics[width=0.24\textwidth]{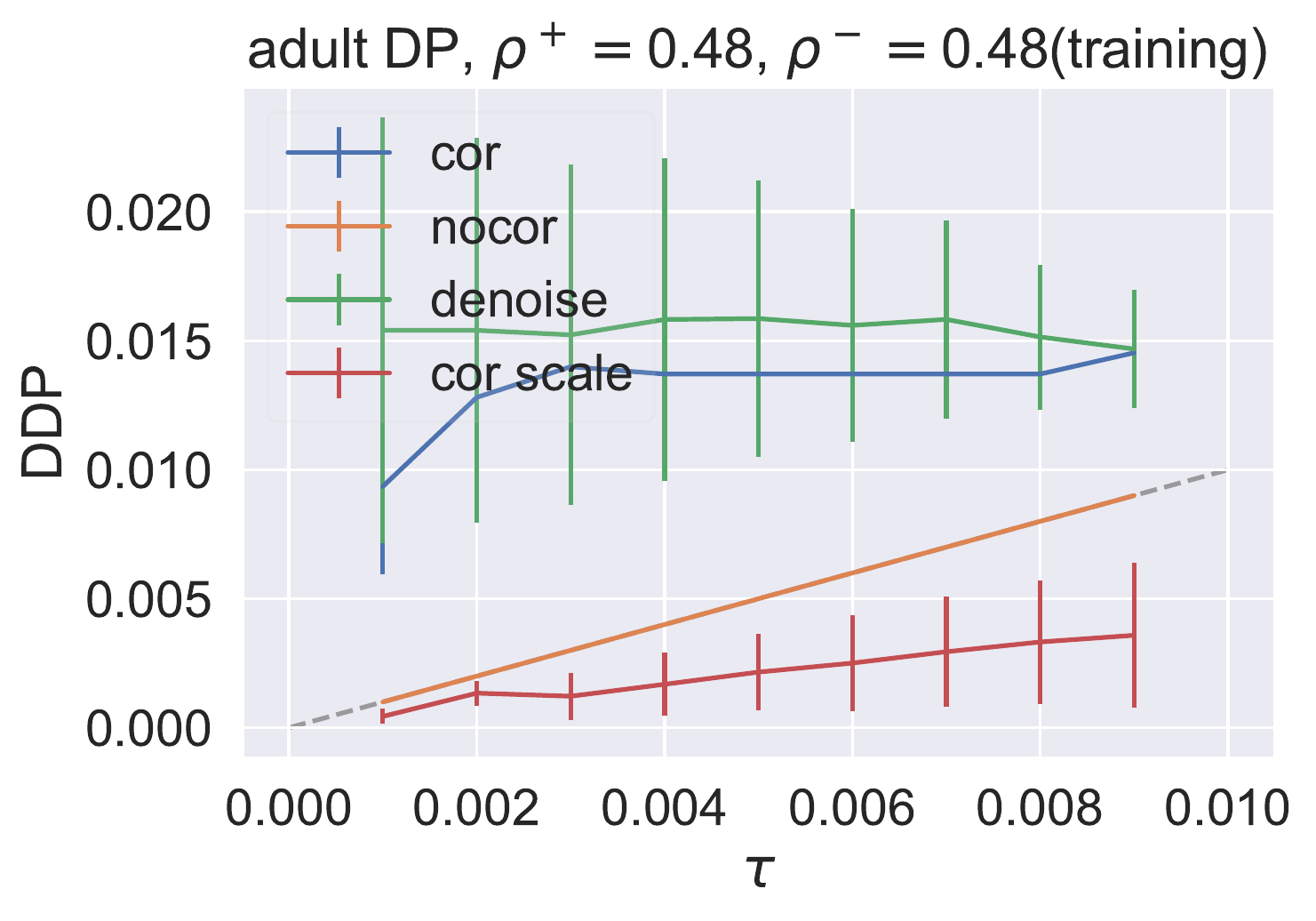}
    \includegraphics[width=0.24\textwidth]{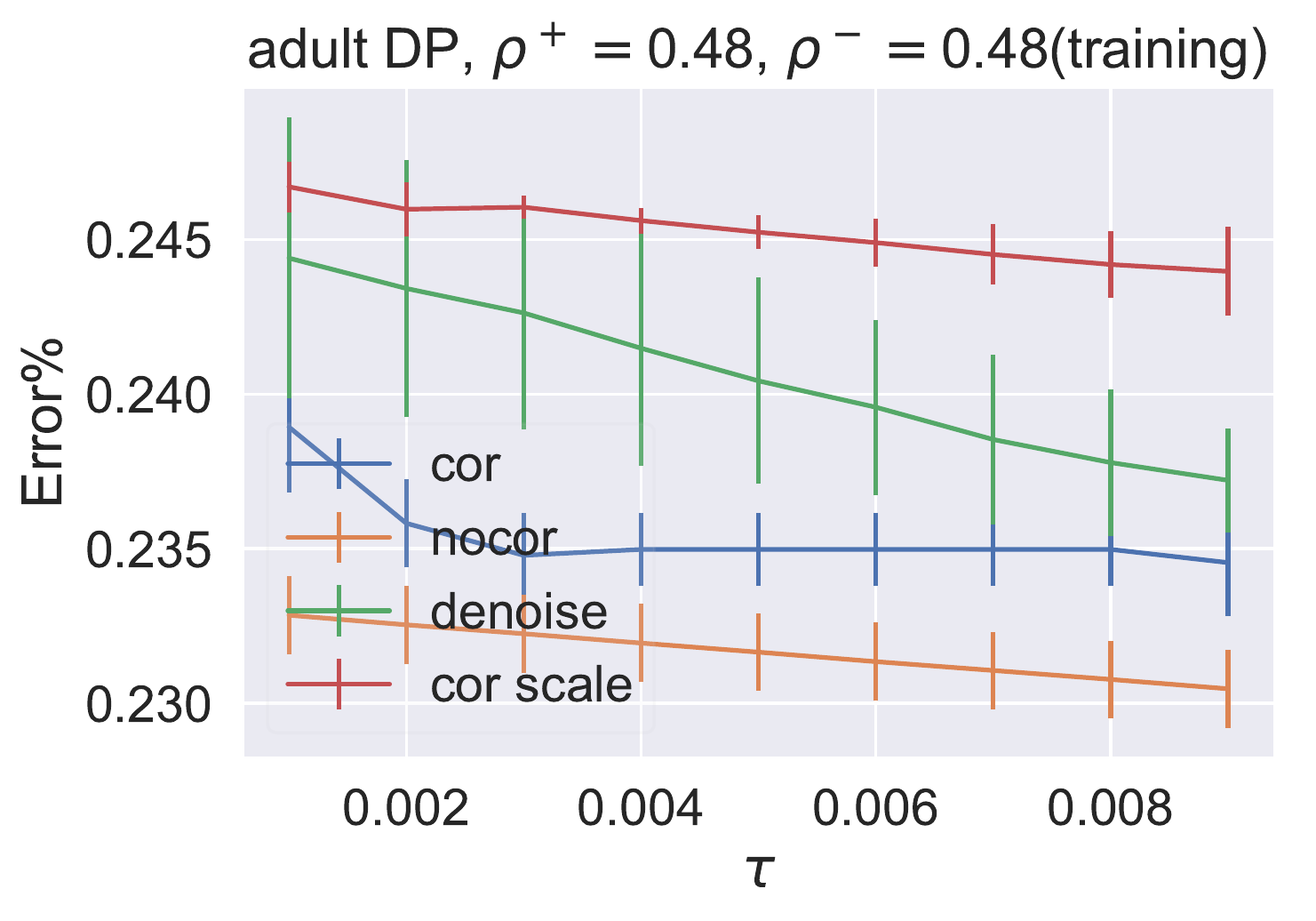}

    \caption{Relationship between input $\tau$ and fairness violation/error on the {\tt adult} dataset for various noise levels. From left to right: testing fairness violation, testing error, training fairness violation, and training error. Different noise levels from top to bottom.}
    \label{fig:diff_noise}
\end{figure*}

\end{appendices}